\documentclass{article}

\PassOptionsToPackage{round}{natbib}  %
\usepackage[main,final]{neurips_2026}

\usepackage[utf8]{inputenc}
\usepackage[T1]{fontenc}
\usepackage{amsmath,amssymb,amsthm}
\usepackage{mathtools}
\usepackage{algorithm}
\usepackage{algorithmic}
\usepackage{etoc}
\usepackage{graphicx}
\usepackage{booktabs}
\usepackage{multirow}
\usepackage{makecell}
\usepackage[bottom]{footmisc}  %
\usepackage{hyperref}
\hypersetup{colorlinks=true, linkcolor=red, citecolor=blue,
  pdftitle={TOPPO: Rethinking PPO for Multi-Task Reinforcement Learning with Critic Balancing},
  pdfauthor={Yuanpeng Li, Gefei Lin, Annie Qu, Rui Miao}}
\usepackage{subcaption}
\usepackage{wrapfig}
\usepackage{float}
\usepackage{needspace}
\usepackage{placeins}  %
\usepackage{xcolor}
\usepackage{tikz}
\usepackage{microtype}
\usepackage{enumitem}
\setlist[itemize]{topsep=2pt,partopsep=0pt,parsep=0pt,itemsep=1pt,leftmargin=*}
\setlist[enumerate]{topsep=2pt,partopsep=0pt,parsep=0pt,itemsep=1pt,leftmargin=*}
\usepackage[capitalize,nameinlink]{cleveref}

\newtheorem{theorem}{Theorem}
\newtheorem{proposition}{Proposition}
\newtheorem{corollary}{Corollary}
\theoremstyle{definition}

\theoremstyle{remark}
\newtheorem{remark}{Remark}

\newcommand{\R}{\mathbb{R}}
\newcommand{\E}{\mathbb{E}}

\newcommand{\bw}{\mathbf{w}}
\newcommand{\bg}{\mathbf{g}}
\newcommand{\bG}{\mathbf{G}}
\newcommand{\bc}{\mathbf{c}}
\newcommand{\bd}{\mathbf{d}}

\newcommand{\btheta}{\boldsymbol{\theta}}

\newcommand{\cone}{\mathrm{cone}^{+}}
\newcommand{\diag}{\mathrm{diag}}

\crefname{lemma}{Lemma}{Lemmas}
\crefname{theorem}{Theorem}{Theorems}
\crefname{proposition}{Proposition}{Propositions}
\crefname{corollary}{Corollary}{Corollaries}
\crefname{remark}{Remark}{Remarks}
\crefname{section}{Section}{Sections}
\crefname{appendix}{Appendix}{Appendices}
\crefname{table}{Table}{Tables}
\crefname{figure}{Figure}{Figures}
\crefname{equation}{Eq.}{Eqs.}
\crefname{algorithm}{Algorithm}{Algorithms}

\newcounter{diagaxis}

\crefformat{diagaxis}{#2#1#3}
\Crefformat{diagaxis}{#2#1#3}
\crefmultiformat{diagaxis}{#2#1#3}{ and~#2#1#3}{, #2#1#3}{, and~#2#1#3}
\Crefmultiformat{diagaxis}{#2#1#3}{ and~#2#1#3}{, #2#1#3}{, and~#2#1#3}
\crefrangeformat{diagaxis}{#3#1#4--#5#2#6}
\Crefrangeformat{diagaxis}{#3#1#4--#5#2#6}

\makeatletter
\def\@fnsymbol#1{\ensuremath{\ifcase#1\or \dagger \or \ddagger \or \mathsection \or \mathparagraph \or \| \or ** \or \dagger\dagger \or \ddagger\ddagger \else\@ctrerr\fi}}
\makeatother

\newcommand{\repo}{\url{https://github.com/NAIVlab/TOPPO}}
\newcommand{\wandblogs}{\url{https://forge.coreweave.com/wandb/liyp1220-uci/TOPPO}}
\title{TOPPO: Rethinking PPO for Multi-Task Reinforcement Learning with Critic Balancing}

\author{%
  Yuanpeng Li \\
  UC Irvine \\
  \And
  Gefei Lin \\
  George Washington University \\
  \And
  Annie Qu \\
  UC Santa Barbara \\
  \And
  Rui Miao\thanks{Correspondence to Rui Miao <\href{mailto:rui.miao@utdallas.edu}{\nolinkurl{rui.miao@utdallas.edu}}>} \\
  UT Dallas
}

\begin{document}
\makeatletter
\def\toppo@notice{\enlargethispage{2\baselineskip}%
  \@float{noticebox}[!b]\footnotesize\@noticestring\end@float}
\let\@notice\relax
\makeatother
\maketitle
\vspace{-0.8cm}
\begin{center}
{\small Code: \repo \\ 
Training logs: \wandblogs}
\end{center}
\begin{abstract}
Soft Actor-Critic (SAC) and its variants dominate Multi-Task Reinforcement Learning (MTRL) due to their off-policy sample efficiency, while on-policy methods such as Proximal Policy Optimization (PPO) remain underexplored.  We diagnose that PPO in MTRL suffers from a previously overlooked issue: critic-side gradient ill-conditioning, which may cause tail tasks to stall while easy tasks dominate the value function's updates. To address this, we propose TOPPO (Tail-Optimized PPO), a reformulation of PPO via Critic Balancing---a set of modules that improve gradient conditioning and balance learning dynamics across tasks. Unlike prior approaches that rely on modular architectures or large models, TOPPO targets the optimization bottleneck within PPO itself. Empirically, on the Meta-World+ MT50 benchmark, TOPPO achieves stronger mean and tail-task performance than prior SAC-family and ARS-family methods while using substantially fewer parameters and reaching their final performance in fewer environment steps. Notably, TOPPO matches or surpasses strong SAC baselines early in training and maintains superior performance at full budget.  Ablations confirm the effectiveness of each module in TOPPO and provide insights into their interactions. Our results demonstrate that, with proper optimization, on-policy methods can rival or exceed off-policy approaches in MTRL, challenging the prevailing reliance on SAC and highlighting critic-side gradient conditioning as the central bottleneck.

\end{abstract}

\begin{figure}[!b]
\centering
\includegraphics[width=0.86\linewidth]{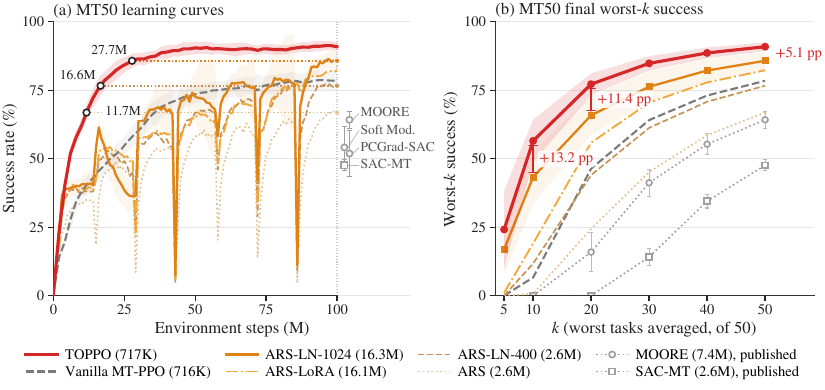}\par
\vspace{-0.2cm}
\caption{TOPPO (717K params) on Meta-World V2 MT50 outperforms the ARS family at $3.6$–$22.7\times$ fewer parameters. (a) Training curve (Circles: TOPPO reaching ARS final means). (b) Worst-$k$ tail tasks. Bands are standard deviations.}
\label{fig:train-curve}
\end{figure}
\makeatletter\toppo@notice\makeatother
\etocdepthtag.toc{maintext}
\section{Introduction}\label{sec:intro}
\begin{table}[!b]
  \vspace{-0.7cm}
\centering
\caption{MT-PPO diagnosis, TOPPO module and mechanism. The main failure is critic-side ill-conditioning (\cref{diag:scale,diag:colinear,diag:unfair}). %
The actor row is a residual combiner choice after per-task advantage normalization removes most actor-side scale spread.}

\label{tab:three-layer-map}
\footnotesize
\setlength{\tabcolsep}{3pt}
\renewcommand{\arraystretch}{1.22}
\begin{tabular}{@{}>{\raggedright\arraybackslash}p{0.15\linewidth}
                >{\raggedright\arraybackslash}p{0.3\linewidth}
                >{\raggedright\arraybackslash}p{0.52\linewidth}@{}}
\toprule
Diagnosis & Module / stage & Mechanism and evidence \\
\midrule
\multicolumn{3}{@{}l}{\textit{Critic-side (primary failure)}} \\[-1pt]
\settowidth{\hangindent}{\cref{diag:scale}:~}\hangafter=1\cref{diag:scale}: Scale Spread
& \textbf{Per-task value normalizer}\par
  \textcolor{gray}{value-loss scale; PopArt~\citep{vanhasselt2016popart,hessel2019popartmultitask}}
& Compresses critic-Gram diagonal scale. %
\par
  \textcolor{gray}{Evidence: \cref{fig:f3-early,fig:f3,cor:fg-popart,tab:intervention-isolation,app:v1-sensitivity}.} \\
\settowidth{\hangindent}{\cref{diag:colinear}:~}\hangafter=1\cref{diag:colinear}: Co-Linear Collapse
& \textbf{LayerNorm-Critic (LN-c)}\par
  \textcolor{gray}{critic representation layer~\citep{ba2016layernorm}; RL precedent~\citep{lee2024simba,hansen2024tdmpc2}}
& Regularizes the critic feature Jacobian and lowers critic-Gram off-diagonal cosine collapse while leaving scale structure intact.\par
  \textcolor{gray}{Evidence: \cref{fig:f3-early,fig:f3,fig:fae}.} \\
\settowidth{\hangindent}{\cref{diag:unfair}:~}\hangafter=1\cref{diag:unfair}: Unfair Aggregation
& \textbf{FairGrad-Critic (FG-c)}\par
  \textcolor{gray}{$\alpha{=}1$ critic gradient combiner $\bG^c\bw=\bw^{-1}$~\citep{ban2024fairgrad,mo2000fairness,kelly1997charging}}
& Pins the combined-gradient norm at $\sqrt{K}$-level and absorbs positive task-wise rescaling at the aggregator output.\par
  \textcolor{gray}{Evidence: \cref{thm:scale-invariance,fig:f1,fig:f4}.} \\
\midrule
\multicolumn{3}{@{}l}{\textit{Actor-side (secondary failure)}} \\[-1pt]
Direction Conflict
& \textbf{PCGrad-Actor (PCGrad-a)}\par
  \textcolor{gray}{actor gradient combiner on $\bg_i^a$~\citep{yu2020pcgrad}}
& Actor norm spread is small, so the actor-side combiner targets negative-cosine pairs rather than scale.\par
  \textcolor{gray}{Evidence: \cref{fig:f3b,fig:f3a,tab:per-side}.} \\
\bottomrule
\end{tabular}
\end{table}

Soft Actor-Critic (SAC)~\citep{haarnoja2018sac} has been the default backbone for Multi-Task Reinforcement Learning (MTRL), due to its off-policy nature and sample efficiency under online data-collection budgets: the replay buffer reuses past experiences, and the slow target network softens per-task TD-error magnitudes across history so that cross-task gradient imbalance is not directly visible to the optimizer.
On Meta-World~\citep{yu2020metaworld}, the canonical multi-task manipulation benchmark, this preference is near-universal: architectural variants (Soft Modularization, \citealp{yang2020softmodules}; CARE, \citealp{sodhani2021care}; PaCo, \citealp{sun2022paco}; MOORE, \citealp{hendawy2023moore}; ARS, \citealp{cho2025ars}) and gradient-aggregation methods (PCGrad, \citealp{yu2020pcgrad}; CAGrad, \citealp{liu2021cagrad}; FAMO, \citealp{liu2023famo}; FairGrad, \citealp{ban2024fairgrad}; Nash-MTL, \citealp{navon2022nashmtl}; loss-level reweighting, \citealp{kendall2018uncertainty,liu2019dwa}) all build on SAC-style training~\citep{metaworldplus2025}.
Despite this, every published non-ARS multi-task SAC method floors at $0\%$ success rate on the ten hardest MT50 tasks, while the strongest ARS variant (ARS-LN-1024) clears the floor only at $43.3\%$ and at a substantially higher parameter cost.

In contrast, on-policy methods such as Proximal Policy Optimization (PPO)~\citep{schulman2017ppo} remain underexplored in MTRL. The original Meta-World Multi-Task PPO (MT-PPO) baseline scored only $\sim\!30\%$ on MT10 with a $(512,512)$ MLP, and this comparison has not been revisited under architectures matched to modern MTRL methods. PPO is also a useful testbed for cross-task gradient ill-conditioning: its on-policy GAE exposes task imbalance at each minibatch, without replay-buffer smoothing or target-network lag. We therefore revisit MT-PPO with a $[400{\times}3]$ backbone, matching CAGrad/CARE/PaCo/MOORE, and find that even vanilla MT-PPO lifts the worst-10 tail above the zero floor of prior non-ARS methods while using far fewer parameters.

Although vanilla MT-PPO already breaks the floor, the hardest tasks remain far from solved, and we trace the remaining gap to a self-reinforcing trap on the critic.
Early in training, a few easier tasks learn fast and produce large critic gradients, pulling the shared critic network toward features that reduce their TD errors. Harder tasks that contribute little to this aggregate update must then fit their value functions through shared features optimized primarily for other tasks, which biases their advantage estimates rather than simply making them vanish. The actor updates on those tasks then follow this distorted signal; in the online loop, the resulting policies collect worse or less informative rollouts, leaving the next critic update with little useful signal from the same tasks.
Crucially, loss-level or actor-side reweighting methods such as DWA~\citep{liu2019dwa} and UW~\citep{kendall2018uncertainty} cannot break this loop when they leave the critic-gradient geometry uncorrected.

Empirically, per-task critic gradient norms span $\sim\!497\times$ across tasks, while per-task actor gradient norms span only $\sim\!4.1\times$ after per-task advantage normalization (\cref{fig:f3b}).
This asymmetry is stable across training, and concurrent benchmarking work also identifies the critic, not the actor, as where cross-task gradient conflict concentrates in MT-PPO~\citep{joshi2025mtbench}.
The critic ill-conditioning then decomposes into three aspects (\cref{tab:three-layer-map}): (\cref{diag:scale}) per-task gradient-norm spread, (\cref{diag:colinear}) representation-level co-linear collapse in the per-task Gram matrix, and (\cref{diag:unfair}) unfair aggregation of an already-imbalanced gradient set.

To address these three aspects of critic ill-conditioning together with the residual actor-side gradient conflict, we propose TOPPO (Tail-Optimized PPO), which incorporates four modules at four distinct stages of the per-minibatch update (\cref{tab:three-layer-map}).
PopArt normalizes per-task value targets via a per-task affine head, putting each task's squared TD loss on a comparable scale before the critic gradients are formed.
LN-c adds LayerNorm in the critic to improve hidden-feature conditioning.
FairGrad ($\alpha{=}1$) reweights per-task critic gradients with an $\alpha$-fair objective, selecting a scale-invariant aggregate critic update.
Finally, PCGrad-a handles the remaining actor-side conflict by projecting negative-cosine actor-gradient pairs after advantage normalization.

In our experiments, with only $717$K parameters, TOPPO reaches $90.9 \pm 1.6\%$ final-checkpoint mean success on MT50 and $56.5 \pm 7.8\%$ on the worst-10 tail tasks.
This is $5.1$ and $13.2$ points above ARS-LN-1024, the strongest ARS variant, in terms of mean success and worst-10 tail success rates, with $22.7\times$ fewer parameters (\cref{tab:combined-headline-tail}). 
TOPPO reaches the full-budget mean of ARS-LN-1024 after only $\sim\!28\%$ of the environment-step budget (\cref{fig:train-curve}(a)).

\paragraph{Summary of contributions.} Our main contribution is a critic-centered diagnosis of MT-PPO failure and a stage-specific reformulation across the target, critic, and aggregation pathways.
\begin{itemize}\itemsep0pt\parskip0pt
\item We diagnose three major aspects of MT-PPO critic ill-conditioning (\cref{diag:scale,diag:colinear,diag:unfair}) and a minor aspect of actor ill-conditioning, and identify their empirical signatures in MT10 (\cref{fig:f3,fig:fae}).
\item We propose TOPPO, a four-module intervention (PopArt, LN-c, FairGrad ($\alpha{=}1$), PCGrad-a), applied at distinct stages of the per-minibatch update.
\item Theoretically, we show that FairGrad with $\alpha{=}1$ fixes the combined-gradient $L^2$ norm at $\sqrt{K}$ and is invariant to positive task-wise rescaling at the aggregator output, explaining the predictable composition of PopArt and FG-c. We also introduce a projected Newton solver that makes FG-c tractable at $K{=}50$ ($\sim\!50$--$1{,}450\times$ faster than a generic \texttt{scipy} nonlinear-solver baseline).
\item We show through experiments on Meta-World+ that TOPPO outperforms existing methods on MT50. Along the ablation study the full stack TOPPO gives the best MT50 headline and tail results, while individual marginal effects depend strongly on the reward regime.
\end{itemize}

\section{Preliminaries}\label{sec:method:prelim}

We consider Markov Decision Processes (MDPs) $\{\mathcal{M}_i\}_{i=1}^K$ sharing the same state space $\mathcal{S}$ and action space $\mathcal{A}$ across $K$ tasks indexed by $i$, with initial state $s_0\sim \mu_i$, task-specific transition dynamics $P_i(s' \mid s, a)$ and reward $r_i(s, a)$~\citep{yu2020metaworld,metaworldplus2025}. In online MTRL, the agent jointly learns a task-aware policy $\pi_{\btheta^a}(a \mid s, i)$ and state-value critic $V_{\btheta^c}(s, i)$ with parameters $\btheta^a$ and $\btheta^c$ shared across all $K$ tasks. The objective is to maximize the average cross-task expected discounted return $\frac{1}{K}\sum_{i=1}^K \E_{\pi_{\btheta^a},\, P_i, \mu_i}\!\big[\sum_t \gamma^t r_i(s_t, a_t)\big]$ with some discount factor $\gamma \in (0, 1)$. 

In MT-PPO, for each task $i$, we collect rollouts $\{(s_t, a_t, r_i(s_t, a_t))\}$ from $\mathcal{M}_i$ under the old policy $\pi_{\btheta^a_{\text{old}}}$, then calculate the importance ratio $\rho_t^{(i)}(\btheta^a) = \pi_{\btheta^a}(a_t \mid s_t, i) \,/\, \pi_{\btheta^a_{\text{old}}}(a_t \mid s_t, i)$, the per-task TD residual $\delta_t^{(i)} = r_i(s_t, a_t) + \gamma V_{\btheta^c_{\text{old}}}(s_{t+1}, i) - V_{\btheta^c_{\text{old}}}(s_t, i)$, the GAE advantage $\hat A_t^{(i)} = \sum_{l \geq 0} (\gamma \lambda)^l \delta_{t+l}^{(i)}$, and the return target $\hat V_t^{(i),\text{targ}} = \hat A_t^{(i)} + V_{\btheta^c_{\text{old}}}(s_t, i)$~\citep{schulman2016gae}. The vanilla MT-PPO objective averages per-task losses, each combining a clipped policy surrogate, a squared value loss, and an entropy bonus~\citep{schulman2017ppo}:
\begin{equation}\label{eq:ppo}
\begin{aligned}
L(\btheta) &\;=\; K^{-1}\sum\nolimits_{i=1}^K L_i(\btheta), \qquad \text{where }
L_i(\btheta) \;=\; L_i^a(\btheta^a) + c_v\, L_i^c(\btheta^c) - c_e\, H_i(\btheta^a), \\[2pt]
L_i^a(\btheta^a) &\;=\; -\,\E_t\!\left[\min\!\big(\rho_t^{(i)}(\btheta^a)\, \hat A_t^{(i)},\; \mathrm{clip}(\rho_t^{(i)}(\btheta^a), 1{-}\epsilon, 1{+}\epsilon)\, \hat A_t^{(i)}\big)\right], \\[2pt]
L_i^c(\btheta^c) &\;=\; \E_t\!\left[\big(V_{\btheta^c}(s_t, i) - \hat V_t^{(i),\text{targ}}\big)^2\right], \qquad
H_i(\btheta^a) \;=\; \E_t\!\big[\mathcal{H}(\pi_{\btheta^a}(\cdot \mid s_t, i))\big],
\end{aligned}
\end{equation}
where $\E_t$ denotes the empirical average over steps in the rollout batch, $c_v, c_e > 0$ are scalar coefficients, and $\mathcal{H}(\bullet)$ is the entropy of a distribution. 
Per-task advantage normalization (computed within each task's slice of a minibatch) is applied to $\hat A^{(i)}$ before the actor surrogate. This is a task-wise variant of the per-minibatch normalization option studied in PPO implementation analyses~\citep{andrychowicz2021whatmatters}. Here it is used to remove actor-side cross-task scale differences, following the MT-PPO convention in Meta-World+~\citep{metaworldplus2025}.

Across all $K$ tasks, $\btheta^a$ and $\btheta^c$ are shared, so every gradient step must aggregate $K$ per-task gradients over shared parameters into a single update direction. The relative magnitudes and pairwise alignment of these gradients determine whether the aggregated gradient represents the multi-task objective fairly or is dominated by a subset of tasks. 

To diagnose the gradient ill-conditioning, we capture the per-task gradients $\bg_i^a := \nabla_{\btheta^a} L_i^a$ and $\bg_i^c := \nabla_{\btheta^c} L_i^c$ on the shared parameters for actor and critic, respectively. We omit global scalar coefficients such as $c_v$ because they do not affect cross-task relative geometry. We also consider the Gram matrices $\bG^\bullet \in \R^{K\times K}$, $\bG^\bullet_{ij} = \langle \bg_i^\bullet, \bg_j^\bullet\rangle$ ($\bullet \in \{a, c\}$, i.e., actor and critic), where the diagonal $\bG^\bullet_{ii} = \|\bg_i^\bullet\|^2$ records per-task scale, and the off-diagonal cosines $\bG^\bullet_{ij}/\sqrt{\bG^\bullet_{ii}\bG^\bullet_{jj}}$ record pairwise direction conflict. Throughout, we contrast each aggregator against the default average $\bar{\bg}^\bullet = \frac{1}{K}\sum_i \bg_i^\bullet$ applied to each shared parameter set for actor and critic.

\section{Gradient Ill-conditioning}\label{sec:method:pathology}
At each shared-parameter update, MT-PPO aggregates $K$ per-task gradients into a single direction. We diagnose this aggregation along three fairness axes: preventing raw-scale domination, preserving task-specific directions, and allocating update budget according to gradient geometry rather than a few dominant tasks. Vanilla MT-PPO's critic fails along all three axes, as summarized in \cref{tab:three-layer-map} and detailed below.

\begin{enumerate}[label=\textbf{(D\arabic*)},itemsep=0pt,parsep=0pt,topsep=0pt]
\item \refstepcounter{diagaxis}\label{diag:scale}\textbf{Per-task scale spread.}
Large diagonal entries $\bG^c_{ii}$ make the mean
$\bar{\bg}^c := K^{-1}\sum_i \bg_i^c$ inherit the scale of high-gradient tasks, leaving small-gradient tasks with little influence.

\item \refstepcounter{diagaxis}\label{diag:colinear}\textbf{Co-linear collapse.}
A high mean off-diagonal $|\!\cos\!| = |\bG^c_{ij}|/\sqrt{\bG^c_{ii}\bG^c_{jj}}$ for $i\neq j$ indicates that per-task critic gradients have collapsed toward a small set of shared-feature directions, which leaves little task-specific direction for any aggregator to act on and makes $\bG^c$ near-singular.

\item \refstepcounter{diagaxis}\label{diag:unfair}\textbf{Unfair aggregation.}
Even when scale spread and co-linearity are reduced, the default mean update can still allocate most of the combined direction to a few tasks. The aggregation rule therefore needs to be geometry-aware, not merely scale-normalized.
\end{enumerate}

\begin{figure}[!t]
\centering
\includegraphics{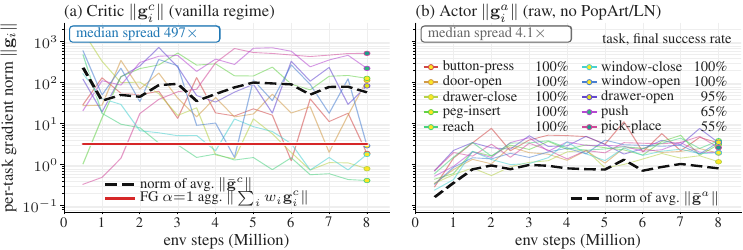}\par
\vspace{-0.2cm}
\caption{Per-task gradient norms $\|\bg_i\|$ (thin lines) on MT10, $0.5$--$8.0$M steps. (a) Critic: the mean update's norm $\|\bar{\bg}^c\|$ (dashed) and FG-c ($\alpha{=}1$) aggregate $ \|\sum_i w_i \bg_i^c\| = \sqrt{K}$ (red, \cref{thm:scale-invariance}). (b) Actor: $\|\bar{\bg}^a\|$ (dashed) lies below due to actor gradients $\bg_i^a$ partially canceled each other (\cref{fig:f3a}).}
\label{fig:f3b}
\end{figure}

\paragraph{Critic is the dominant issue.}
In our diagnostics, the per-task gradient-norm spread on the critic,
$\max_i\|\bg_i^c\|/\min_i\|\bg_i^c\|$, is two orders of magnitude larger than the actor-side spread,
$\max_i\|\bg_i^a\|/\min_i\|\bg_i^a\|$, and this gap persists across training (\cref{fig:f3b}; full diagnostics in \cref{app:ablations,app:mech-figs}). This asymmetry is expected: per-task advantage normalization absorbs much of the reward-scale variation in the actor surrogate, whereas the critic fits unnormalized value targets unless explicitly corrected. We therefore focus the diagnosis and interventions on the critic side. Actor-side aggregation is secondary in our work, and we use PCGrad~\citep{yu2020pcgrad} as a representative actor-gradient aggregator.

\paragraph{A Self-reinforcing tail trap.}
Early critic imbalance is damaging because it can turn small initial gaps into a self-reinforcing failure mode. In vanilla MT-PPO, the critic update averages task gradients, so the shared critic parameters $\btheta^c$ are driven mainly by tasks with large critic gradients (\cref{diag:scale,diag:colinear,diag:unfair}). Small-gradient tasks must then fit their value functions through shared features shaped primarily by other tasks' TD errors, yielding biased advantages $\hat A_t^{(i)}$ and distorted actor updates. The resulting policies collect less-informative data, which further reduces useful critic signal for the same tasks in later updates. This feedback loop explains why actor-side reweighting or gradient surgery alone gives limited gains without Critic Balancing. Empirically, vanilla MT-PPO's worst-10 MT50 tasks average only $6.6\%$ success at the training budget, and its five hardest tasks remain below $20\%$ (\cref{tab:worst-k,tab:per-task-mt50}). This motivates correcting critic gradient ill-conditioning early in training. See \cref{fig:f3-early} for \textit{early-stage} Gram diagnostics and \cref{fig:ladder} for TOPPO's effect.

\section{Method: TOPPO}\label{sec:method:framework}

\begin{wrapfigure}[17]{r}{0.59\textwidth}
\vspace{-1.8cm}
\begin{minipage}{\linewidth}
\begin{algorithm}[H]
\caption{\textbf{TOPPO}: a PPO collect-and-update phase. See full version in \cref{alg:toppo}.}
\label{alg:toppo-compact}
\begin{algorithmic}[1]
\REQUIRE Tasks $\{\mathcal{M}_i\}_{i=1}^K$, $\btheta = (\btheta^a, \btheta^c)$ with \textbf{[LN-c]} and PopArt $\{(\sigma_i, \mu_i)\}$, repeats $R$, max grad-norm $\tau$
\STATE Rollout $\mathcal{B} = \bigcup_i \mathcal{B}_i$ by $\pi_{\btheta^a_{\text{old}}}$; calculate $\hat A_t^{(i)}, \hat V_t^{(i),\text{targ}}$
\STATE \textbf{[PopArt]} update $(\mu_i, \sigma_i)$, PopArt-renorm affine head
\FOR{$r = 1$ \TO $R$ (recompute adv.\ if $r > 1$)}
  \FOR{each stratified minibatch $\mathcal{B}^{\text{mb}} \subset \mathcal{B}$}
    \STATE normalize advantages within each task; use critic targets $(\hat V_t^{(i),\text{targ}}-\mu_i)/\sigma_i$ as in \cref{sec:method:popart}
    \STATE for each $i$: compute $\bg_i^a$ and $\bg_i^{c,\mathrm{PopArt}}$ from the actor and PopArt critic losses
    \STATE \textbf{[FG-c]} $\bd^c\! \gets\!\! \sum_i\! w_i \bg_i^{c,\mathrm{PopArt}}$ from \cref{alg:ynewton}
    \STATE \textbf{[PCGrad-a]} $\bd^a \gets \mathrm{PCGrad}(\{\bg_i^a\})$
    \STATE rescale $(\bd^a, \bd^c)$ by $\min(1,\, \tau / \|(\bd^a, \bd^c)\|_2)$; Adam step on $\btheta$
  \ENDFOR
\ENDFOR
\end{algorithmic}
\end{algorithm}
\end{minipage}
\vspace{-1.5em}
\end{wrapfigure}

To address the gradient ill-conditioning discussed in \cref{sec:method:pathology}, we propose TOPPO, which consists of three critic-side modules and an actor-side gradient aggregator, listed in \cref{tab:three-layer-map} and laid out in \cref{alg:toppo-compact}. 

On the critic side, a \textbf{per-task value normalizer} puts each task's TD loss on the same scale before computing critic-side gradients (\cref{diag:scale}); \textbf{LN-c} (LayerNorm in the critic network) regularizes the hidden-activation conditioning that produces co-linear collapse (\cref{diag:colinear}); \textbf{FG-c ($\alpha{=}1$)} reweights per-task gradients at the aggregator so that the pre-clipping combined-gradient norm is pinned at $\sqrt{K}$ (\cref{diag:unfair}).

On the actor side, per-task advantage normalization makes reward-scale variation much smaller, leaving direction conflict as the main residual issue. Therefore, we apply PCGrad to actor-gradient aggregation (\textbf{PCGrad-a}).
Empirical comparisons with alternative actor-side aggregators are reported in \cref{app:ablations,tab:per-side}.

In subsections below, we discuss the three critic-side interventions in more detail.

\subsection{Per-task value normalization}\label{sec:method:popart}
In heterogeneous environments the per-task value-target variation can be extreme,
so the squared TD loss $L_i^c$ can inherit this variation and spread quadratically across tasks (\cref{diag:scale}).
A per-task value normalizer addresses this by estimating each task's return scale and rescales its value targets, rewards, or prediction errors, so that value losses have comparable scales across tasks.
Several normalizers can be used, for example per-task reward normalization, return-based scaling~\citep{schaul2021returnscaling}, or PopArt~\citep{vanhasselt2016popart}.
We adopt PopArt in this work.

PopArt~\citep{vanhasselt2016popart,hessel2019popartmultitask} normalizes value targets by inserting a per-task affine layer between the shared critic network and the value prediction. For the ``Art'' (Adaptively Rescaling Targets) part, the shared critic outputs a single normalized scalar $z_{\btheta^c}(s, i) \in \R$, and a per-task scale-and-shift $(\sigma_i, \mu_i)$ maps it back to the raw value range by
\[
V_{\btheta^c}(s, i) \;=\; \sigma_i\, z_{\btheta^c}(s, i) + \mu_i,
\qquad \text{where } (\mu_i, \sigma_i^2) \;\approx\; \big(\,\widehat{\mathbb{E}}_t[\hat V_t^{(i),\text{targ}}],\; \widehat{\mathrm{Var}}_t[\hat V_t^{(i),\text{targ}}]\,\big).
\]
These running estimates of each task's target mean and variance are accumulated online via Welford's algorithm~\citep{welford1962} (canonical PopArt uses an EMA, see \cref{app:popart-welford}).
The shared critic is then trained to predict normalized per-step targets $(\hat V_t^{(i),\text{targ}} - \mu_i)/\sigma_i$, which have approximately zero mean and unit variance per task by construction, so every task's value-loss sits on the same scale.

The ``Pop'' (Preserving Outputs Precisely) part is a small bookkeeping step: whenever $(\mu_i, \sigma_i)$ update, the per-task affine head is simultaneously re-parameterized so the raw prediction $V_{\btheta^c}(s, i)$ does not jump~\citep{vanhasselt2016popart}, and the underlying critic remains a single shared MLP.

PopArt target normalization acts as a task-wise rescaling of critic gradients. Let $\bg_i^{c,\text{PopArt}}$ be the shared-critic gradient from the PopArt-normalized critic loss, and let $\bg_i^{c,\text{raw}}$ be the corresponding gradient from the unnormalized-target loss. For shared critic parameters $\btheta^c$ and a fixed PopArt statistic $\sigma_i$ during the update,
\begin{equation}\label{eq:popart-rescale}
\bg_i^{c,\text{PopArt}} \;=\; \sigma_i^{-2}\, \bg_i^{c,\text{raw}},
\end{equation}
with one factor of $\sigma_i^{-1}$ from the normalized residual and another $\sigma_i^{-1}$ from the chain rule. This rescaling is observed empirically in \cref{fig:f3-early}.

\subsection{LN-c: LayerNorm on the critic network}\label{sec:method:lnc}
Pre-activation LayerNorm~\citep{ba2016layernorm} is applied to the critic's hidden linear layers to stabilize task-wise hidden activation scales before forming $\bG^c$. In early-stage critic Gram diagnostics, critic LN reduces the mean off-diagonal $|\!\cos\!|$ from $0.34$ to $0.20$ while roughly preserving diagonal magnitude (\cref{fig:f3-early}; full diagnostics in \cref{fig:f3,fig:fae}). This is consistent with prior RL reports on LayerNorm~\citep{lee2024simba,hansen2024tdmpc2}. In the MT50 cumulative ablation ladder, LN-c gives the largest single-step gain (\cref{fig:ladder}).

\subsection{\texorpdfstring{FairGrad ($\alpha{=}1$)}{FairGrad (alpha=1)}: per-task critic-gradient aggregation}\label{sec:method:fgc}

To address unfair gradient aggregation (\cref{diag:unfair}), we replace the critic-gradient mean aggregation $\frac{1}{K}\sum_i g_i^c$ with FairGrad aggregation $\bd^c=\sum_i w_i \bg_i^c$~\citep{ban2024fairgrad} with task weights $\bw = (w_1, \ldots, w_K)^\top \in\R^K_{>0}$.
FairGrad downweights tasks whose gradients would otherwise dominate the combined direction, so that small-gradient tasks are not sacrificed to large ones.
For example, for orthogonal gradients $\bg_1,\bg_2$ with norms $10$ and $1$, the mean direction has cosine $0.995$ with $\bg_1$, whereas FairGrad with $\alpha{=}1$ sets $w_i=1/\|\bg_i\|$, so the two weighted gradients have equal norms.

Formally, FairGrad chooses $\bd^c$ within a norm ball to maximize the $\alpha$-fairness utility $\sum_i U_\alpha(x_i)$~\citep{mo2000fairness,kelly1997charging}, where
\begin{equation}\label{eq:alpha-utility}
U_\alpha(x) \;=\; \begin{cases} (1-\alpha)^{-1}\, x^{1-\alpha}, & 0\le\alpha \neq 1, \\ \log x, & \alpha = 1, \end{cases}
\end{equation}
and $x_i := \langle\bg_i^c,\bd^c\rangle$ is task $i$'s first-order loss decrease along $\bd^c$.
Writing $\bd^c=\sum_j w_j\bg_j^c$ gives $x_i=(\bG^c\bw)_i$, and under the standard FairGrad scaling the KKT condition is $\bG^c\bw = \bw^{-1/\alpha}$ componentwise (\cref{app:alpha-fairness}).

\begin{wrapfigure}[14]{r}{0.6\textwidth}
\centering
\vspace{-0.65cm}
\includegraphics{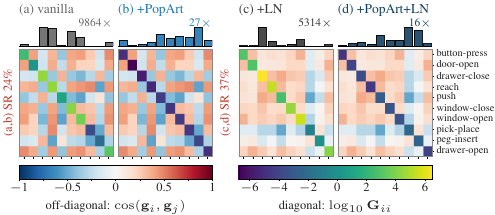}\par
\vspace{-0.15cm}
\caption{Early-stage ($\sim\!0.5$M steps) MT10 Gram matrices of critic gradients. Off-diagonal values are $\cos\angle(\bg_i^c,\bg_j^c)$ and diagonal $\log_{10}\bG^c_{ii}$. Histograms: $\|\bg_i^c\|/K$ and magnitude-spread ratio. 
See \cref{fig:f3} for mid/late stages.}
\label{fig:f3-early}
\end{wrapfigure}
\paragraph{Choice of $\alpha$.}
The parameter $\alpha$ in \eqref{eq:alpha-utility} trades total improvement against fairness. 
We consider a case of conflicting gradients with different magnitudes of norm. As illustrated in \cref{fig:fg-alpha}, take $\bg_1,\bg_2$ with norms $10$ and $1$ pointing in nearly opposite directions ($\cos\angle(\bg_1,\bg_2)={-0.9}$).
When $\alpha{=}0$, $\bd$ follows $\bg_1$ and locally increases task 2's loss.
On another extreme, $\alpha{\to}\infty$ equalizes the raw decreases $x_i$, which forces $\bd$ to be nearly orthogonal to the large $\bg_1$ ($\cos\angle(\bd,\bg_1)=0.04$), so the large task barely improves.
Our choice $\alpha{=}1$ instead equalizes the task contributions $\langle w_i\bg_i,\bd\rangle = w_i x_i=1$ by the KKT condition, with the same angle $\cos\angle(\bd,\bg_1) = \cos\angle(\bd,\bg_2)=0.22$.
Therefore, even if the two tasks conflict in their critic gradients, both tasks improve slowly in the combined direction $\bd$, which can potentially benefit other tasks at the time of the update.
Theoretically, at $\alpha{=}1$ the KKT condition becomes $\bG^c\bw=\bw^{-1}$, and its solution has two good properties that fail for every other $\alpha$ as explained in \cref{thm:scale-invariance} below.

\begin{theorem}[FairGrad ($\alpha{=}1$): fixed norm and scale invariance]
\label{thm:scale-invariance}
Let $\bg_1,\ldots,\bg_K\in\R^d$ be linearly independent, with Gram matrix
$\bG=[\langle \bg_i,\bg_j\rangle]_{ij}\succ0$. Let $\bw\in\R^K_{>0}$ be the unique positive solution of
$\bG\bw=\bw^{-1}$, where the inverse is componentwise. Then $\bd=\sum_{i=1}^K w_i\bg_i$ has fixed norm 
$\|\bd\|^2=\bw^\top\bG\bw=K$.
Moreover, it is invariant to arbitrary positive rescaling. Specifically, for any
$\bc\in\R^K_{>0}$ and $\tilde\bg_i=c_i\bg_i$, the rescaled Gram matrix
$\tilde\bG$ has positive solution $\tilde\bw=\bw\oslash\bc$, and
$\sum_i \tilde w_i\tilde\bg_i=\sum_i w_i\bg_i$.
\end{theorem}

\begin{figure}[!t]
\centering
\includegraphics[width=\linewidth]{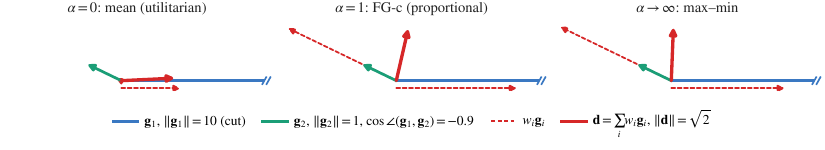}\par
\vspace{-0.2cm}
\caption{The $\alpha$-fair update on the two conflicting gradients of \cref{sec:method:fgc}. Updates are rescaled to $\|\bd\|{=}\sqrt{K}$ for comparison; $\bg_1$ is truncated.}
\label{fig:fg-alpha}
\end{figure}

\cref{thm:scale-invariance} shows the two properties for the critic update.
The fixed norm property ensures that each task contributes $w_i x_i=1$, so $\|\bd^c\|^2=K$ and every task receives an equal $1/K$ share of the critic update.
The critic gradient thus has the same norm on every minibatch, instead of growing and shrinking with the TD errors of the large-gradient tasks.
The scale invariance property ensures that the update depends only on gradient directions, so a task's influence does not shrink with its gradient norm.
Importantly, it also lets FG-c absorb PopArt's per-task rescaling, so the two modules compose predictably (\cref{cor:fg-popart}). The proof of \cref{thm:scale-invariance} is in \cref{app:scale-invariance-proof}.

\begin{corollary}[PopArt rescaling is absorbed at the FairGrad ($\alpha{=}1$) output]\label{cor:fg-popart}
Within one PPO minibatch, before PopArt statistics $(\mu_i,\sigma_i)$ are refreshed, PopArt rescales each shared-critic gradient by $\bg_i^{c,\mathrm{PopArt}}=\sigma_i^{-2}\bg_i^{c,\mathrm{raw}}$ (\cref{eq:popart-rescale}). If the raw critic Gram is positive definite, then by \cref{thm:scale-invariance} the $\alpha{=}1$ FairGrad combined critic gradient is identical on raw and PopArt-rescaled gradients.
\end{corollary}

\begin{wrapfigure}{r}{0.42\linewidth}
\vspace{-0.4cm}
\centering
\includegraphics[width=\linewidth]{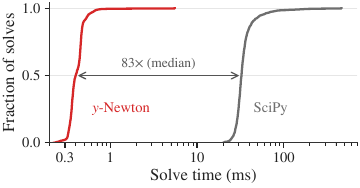}\par
\vspace{-0.2cm}
\caption{Solve time of FairGrad ($\alpha{=}1$) solvers on $1{,}984$ real MT50 Gram matrices (full version in \cref{fig:revision-solver}).}
\label{fig:solver-main}
\vspace{-0.4cm}
\end{wrapfigure}
\paragraph{A fast solver.}
Solving $\bG^c\bw=\bw^{-1}$ is the main computational bottleneck for FG-c. Existing FairGrad/Nash-MTL solvers~\citep{ban2024fairgrad,navon2022nashmtl} are too slow for per-minibatch critic updates. The fixed-budget SGD solver in \citet{ban2024fairgrad} is faster, but it misses the KKT target on every real critic Gram we tested (\cref{fig:solver-mt10}). We propose a projected Newton solver with a closed-form diagonal warm start, which reaches the target KKT tolerance with much lower wall-clock time (see \cref{fig:solver-main} and details in \cref{app:solver}).

\section{Experiments}\label{sec:experiments}

\subsection{Setup}\label{sec:exp:setup}

\textbf{Benchmark.} We train and evaluate all models on Meta-World V2~\citep{yu2020metaworld,metaworldplus2025}, MT10 ($K{=}10$) and MT50 ($K{=}50$) over all tasks with a task identifier. We use V2 reward throughout the paper unless specifically labeled V1. Beyond Meta-World, in \cref{app:revision-amp}, we demonstrate our method to a two-objective humanoid control task \citep{peng2021amp}.

\textbf{Architecture.} TOPPO, our PPO-based method, adopts a fully task-shared $[400{\times}3]$ ReLU MLP, the standard Meta-World multi-task backbone used by most prior works, including the SAC-family in Meta-World+~\citep{metaworldplus2025} and the structured-RL methods CAGrad / CARE / PaCo / MOORE / FAMO / FairGrad, with separate actor and critic of the same size. The only per-task structure is the input one-hot, the action log standard deviation, and the PopArt output head (\cref{sec:method:popart}). 
ARS~\citep{cho2025ars} uses architectures with deeper $[400{\times}4]$ MLPs by default, and wider $[1024{\times}4]$ MLPs for ARS-LN-1024 (a scaled version). Our TOPPO has $685$K (MT10) / $717$K (MT50) parameters, which on MT50 are $22.5$--$22.7\times$ smaller than the strongest ARS variants, ARS-LoRA and ARS-LN-1024 (See \cref{tab:params} and \cref{app:setup-details}).

\textbf{Training.} Our models are trained with $20$M (MT10) / $100$M (MT50) total environment steps ($2$M per task), matching the ARS / MOORE / PaCo / Meta-World+ convention. Our TOPPO \emph{headline} configuration is PCGrad-a $+$ FG-c ($\alpha{=}1$) $+$ PopArt $+$ LN-c; the \emph{vanilla MT-PPO} baseline removes all four interventions while retaining standard MT-PPO machinery (per-task advantage normalization, stratified minibatches, value clip, post-aggregation $\|\cdot\| \leq 1$ clip), so the contrast isolates the four interventions. Full hyperparameters are in \cref{tab:hyperparameters}; PPO-specific adaptations of SAC-based gradient-surgery methods are in \cref{app:ppo-adaptations}.

\textbf{Baselines and reporting.} For comparability, each comparison keeps the reward version and reporting statistic fixed. For V2 reward, we follow the final-checkpoint mean $\pm$ standard deviation (std) protocol of~\citet{cho2025ars}. \Cref{tab:combined-headline-tail} lists their published numbers (8 seeds in their Tables~1 and~2) alongside our $10$-seed reruns of SAC-MT and the four ARS variants with the authors' code and configurations (protocol differences in \cref{app:revision-baselines}). For the V1 head-to-head with Meta-World+~\citep{metaworldplus2025}, both our rows and the baseline rows follow the Meta-World+ $10$-seed final-checkpoint Interquartile Mean (IQM)\,$\pm$ std (\cref{tab:v1-mwp}). The appendix gives parameter counts (\cref{app:param-audit}), step budgets, reporting protocols, and all four success metrics (\cref{tab:step-convention,tab:protocol,tab:4metric}), and per-task success rates (\cref{app:per-task}).

\textbf{Code and reproducibility.} Code Repository: \repo \newline Training logs: \wandblogs

\subsection{Main Results}\label{sec:exp:headline}
\paragraph{MT50 mean and Worst-$k$ tail tasks.} On MT50, TOPPO ($90.9\%$ mean, $56.5\%$ worst-10) beats the strongest baseline, ARS-LN-1024 ($85.7\%$ mean, $43.3\%$ worst-10), on both metrics: by $5.1$ points on the mean (\cref{tab:headline}) and $13.2$ points on the worst-10 tail tasks (\cref{tab:worst-k}), at $22.7\times$ fewer parameters. SAC-MT and all published non-ARS multi-task SAC methods score $0.0\%$ on the worst-10 tail tasks.

\paragraph{MT10: Saturating performance.} The MT10 \emph{final}-checkpoint mean success rate in \cref{tab:headline} ($97.2\%$) sits $1.2$ points below ARS-LoRA ($98.4\%$, within seed variation), but every seed in fact reaches $\geq 99\%$ at its own best checkpoint (see per-seed best columns in \cref{tab:4metric}). This gap reflects PPO's well-known late-training instability, for which linear learning-rate annealing is a documented code-level practice~\citep{engstrom2020implementation,andrychowicz2021whatmatters}. We do not apply such a learning-rate schedule, to keep training-time hyperparameters identical across all our configurations and to align with the published-baseline protocol that reports a single fixed-schedule final-checkpoint number. We report \emph{final} as the headline statistic as it is the most conservative and matches the protocol used in most published baselines (see \cref{tab:protocol}). For more detailed results, see per-seed argmax-step distribution and the full $4$-metric breakdown in \cref{tab:bestckpt-steps,tab:4metric}.

\begin{table}[t]
\centering
\caption{Final-checkpoint Meta-World V2 success (\%) at $20$M/$100$M steps (MT10/MT50), mean with seed std in gray. Worst-$k$: mean of each seed's $k$ lowest tasks. Published rows use a different protocol and worst-$k$ aggregation (\cref{app:revision-baselines}). Params include target networks (\cref{tab:params}). Bold: column maximum; --: not reported or not run.}
\label{tab:combined-headline-tail}%
\label{tab:headline}%
\label{tab:worst-k}
\scriptsize
\setlength{\tabcolsep}{2.4pt}
\renewcommand{\arraystretch}{0.92}
\setlength{\aboverulesep}{0.25ex}
\setlength{\belowrulesep}{0.45ex}
\begin{tabular}{@{}l r@{\,/\,}l r@{\,}l r@{\,}l r@{\,}l r@{\,}l r@{\,}l r@{\,}l r@{\,}l@{}}
\toprule
 & \multicolumn{2}{c}{Params (K)} & \multicolumn{2}{c}{MT10} & \multicolumn{12}{c}{MT50} \\
\cmidrule(lr){2-3} \cmidrule(lr){4-5} \cmidrule(l){6-17}
Algorithm & \multicolumn{2}{c}{MT10/MT50} & \multicolumn{2}{c}{Mean} & \multicolumn{2}{c}{Worst-5} & \multicolumn{2}{c}{Worst-10} & \multicolumn{2}{c}{Worst-20} & \multicolumn{2}{c}{Worst-30} & \multicolumn{2}{c}{Worst-40} & \multicolumn{2}{c}{Mean} \\
\midrule
\multicolumn{17}{@{}l}{\textit{Published (\citet{cho2025ars}, Tables~1 and~2, 8 seeds; ARS-LN-1024 from Table~6; SMT is \citet{cho2024smt})}} \\
SAC-MT & 2517 & 2597 & $69.4$ & {\fontsize{6pt}{7pt}\selectfont\textcolor{gray}{$\pm\phantom{0}0.8$}} & $0.0$ & {\fontsize{6pt}{7pt}\selectfont\textcolor{gray}{$\pm\phantom{0}0.0$}} & $0.0$ & {\fontsize{6pt}{7pt}\selectfont\textcolor{gray}{$\pm\phantom{0}0.0$}} & $0.0$ & {\fontsize{6pt}{7pt}\selectfont\textcolor{gray}{$\pm\phantom{0}0.0$}} & $14.1$ & {\fontsize{6pt}{7pt}\selectfont\textcolor{gray}{$\pm\phantom{0}3.1$}} & $34.5$ & {\fontsize{6pt}{7pt}\selectfont\textcolor{gray}{$\pm\phantom{0}2.5$}} & $47.6$ & {\fontsize{6pt}{7pt}\selectfont\textcolor{gray}{$\pm\phantom{0}2.0$}} \\
MT-MH-SAC & 2546 & 2754 & $77.2$ & {\fontsize{6pt}{7pt}\selectfont\textcolor{gray}{$\pm11.9$}} & $0.0$ & {\fontsize{6pt}{7pt}\selectfont\textcolor{gray}{$\pm\phantom{0}0.0$}} & $0.0$ & {\fontsize{6pt}{7pt}\selectfont\textcolor{gray}{$\pm\phantom{0}0.0$}} & $0.0$ & {\fontsize{6pt}{7pt}\selectfont\textcolor{gray}{$\pm\phantom{0}0.0$}} & $14.1$ & {\fontsize{6pt}{7pt}\selectfont\textcolor{gray}{$\pm\phantom{0}1.7$}} & $34.0$ & {\fontsize{6pt}{7pt}\selectfont\textcolor{gray}{$\pm\phantom{0}2.0$}} & $47.2$ & {\fontsize{6pt}{7pt}\selectfont\textcolor{gray}{$\pm\phantom{0}1.7$}} \\
Soft Modularization & 3454 & 3534 & $74.4$ & {\fontsize{6pt}{7pt}\selectfont\textcolor{gray}{$\pm10.5$}} & $0.0$ & {\fontsize{6pt}{7pt}\selectfont\textcolor{gray}{$\pm\phantom{0}0.0$}} & $0.0$ & {\fontsize{6pt}{7pt}\selectfont\textcolor{gray}{$\pm\phantom{0}0.0$}} & $1.8$ & {\fontsize{6pt}{7pt}\selectfont\textcolor{gray}{$\pm\phantom{0}3.7$}} & $23.7$ & {\fontsize{6pt}{7pt}\selectfont\textcolor{gray}{$\pm12.3$}} & $42.6$ & {\fontsize{6pt}{7pt}\selectfont\textcolor{gray}{$\pm\phantom{0}9.5$}} & $54.1$ & {\fontsize{6pt}{7pt}\selectfont\textcolor{gray}{$\pm\phantom{0}7.6$}} \\
PCGrad-SAC & 2517 & 2597 & $74.8$ & {\fontsize{6pt}{7pt}\selectfont\textcolor{gray}{$\pm13.7$}} & $0.0$ & {\fontsize{6pt}{7pt}\selectfont\textcolor{gray}{$\pm\phantom{0}0.0$}} & $0.0$ & {\fontsize{6pt}{7pt}\selectfont\textcolor{gray}{$\pm\phantom{0}0.0$}} & $0.0$ & {\fontsize{6pt}{7pt}\selectfont\textcolor{gray}{$\pm\phantom{0}0.0$}} & $21.0$ & {\fontsize{6pt}{7pt}\selectfont\textcolor{gray}{$\pm12.9$}} & $39.9$ & {\fontsize{6pt}{7pt}\selectfont\textcolor{gray}{$\pm10.7$}} & $51.9$ & {\fontsize{6pt}{7pt}\selectfont\textcolor{gray}{$\pm\phantom{0}8.5$}} \\
PaCo & \multicolumn{2}{c}{--} & $82.4$ & {\fontsize{6pt}{7pt}\selectfont\textcolor{gray}{$\pm14.2$}} & $0.0$ & {\fontsize{6pt}{7pt}\selectfont\textcolor{gray}{$\pm\phantom{0}0.0$}} & $0.0$ & {\fontsize{6pt}{7pt}\selectfont\textcolor{gray}{$\pm\phantom{0}0.0$}} & $4.6$ & {\fontsize{6pt}{7pt}\selectfont\textcolor{gray}{$\pm\phantom{0}8.2$}} & $26.1$ & {\fontsize{6pt}{7pt}\selectfont\textcolor{gray}{$\pm15.0$}} & $44.6$ & {\fontsize{6pt}{7pt}\selectfont\textcolor{gray}{$\pm11.2$}} & $55.6$ & {\fontsize{6pt}{7pt}\selectfont\textcolor{gray}{$\pm\phantom{0}9.1$}} \\
MOORE & 6890 & 7403 & $86.0$ & {\fontsize{6pt}{7pt}\selectfont\textcolor{gray}{$\pm\phantom{0}4.8$}} & $0.0$ & {\fontsize{6pt}{7pt}\selectfont\textcolor{gray}{$\pm\phantom{0}0.0$}} & $0.0$ & {\fontsize{6pt}{7pt}\selectfont\textcolor{gray}{$\pm\phantom{0}0.0$}} & $15.9$ & {\fontsize{6pt}{7pt}\selectfont\textcolor{gray}{$\pm\phantom{0}7.0$}} & $41.2$ & {\fontsize{6pt}{7pt}\selectfont\textcolor{gray}{$\pm\phantom{0}5.0$}} & $55.3$ & {\fontsize{6pt}{7pt}\selectfont\textcolor{gray}{$\pm\phantom{0}3.9$}} & $64.2$ & {\fontsize{6pt}{7pt}\selectfont\textcolor{gray}{$\pm\phantom{0}3.1$}} \\
SMT & \multicolumn{2}{c}{--} & $86.8$ & {\fontsize{6pt}{7pt}\selectfont\textcolor{gray}{$\pm\phantom{0}8.6$}} & $0.0$ & {\fontsize{6pt}{7pt}\selectfont\textcolor{gray}{$\pm\phantom{0}0.0$}} & $0.0$ & {\fontsize{6pt}{7pt}\selectfont\textcolor{gray}{$\pm\phantom{0}0.0$}} & $8.0$ & {\fontsize{6pt}{7pt}\selectfont\textcolor{gray}{$\pm\phantom{0}8.9$}} & $26.8$ & {\fontsize{6pt}{7pt}\selectfont\textcolor{gray}{$\pm13.1$}} & $45.0$ & {\fontsize{6pt}{7pt}\selectfont\textcolor{gray}{$\pm\phantom{0}9.9$}} & $56.0$ & {\fontsize{6pt}{7pt}\selectfont\textcolor{gray}{$\pm\phantom{0}8.0$}} \\
ARS & 2517 & 2597 & $97.3$ & {\fontsize{6pt}{7pt}\selectfont\textcolor{gray}{$\pm\phantom{0}2.1$}} & \multicolumn{2}{c}{--} & $9.1$ & {\fontsize{6pt}{7pt}\selectfont\textcolor{gray}{$\pm\phantom{0}6.6$}} & $26.3$ & {\fontsize{6pt}{7pt}\selectfont\textcolor{gray}{$\pm\phantom{0}6.1$}} & $45.3$ & {\fontsize{6pt}{7pt}\selectfont\textcolor{gray}{$\pm\phantom{0}4.7$}} & $57.5$ & {\fontsize{6pt}{7pt}\selectfont\textcolor{gray}{$\pm\phantom{0}3.6$}} & $65.9$ & {\fontsize{6pt}{7pt}\selectfont\textcolor{gray}{$\pm\phantom{0}2.9$}} \\
ARS-LN-400 & 2530 & 2611 & $98.6$ & {\fontsize{6pt}{7pt}\selectfont\textcolor{gray}{$\pm\phantom{0}1.5$}} & \multicolumn{2}{c}{--} & $21.0$ & {\fontsize{6pt}{7pt}\selectfont\textcolor{gray}{$\pm11.4$}} & $49.1$ & {\fontsize{6pt}{7pt}\selectfont\textcolor{gray}{$\pm\phantom{0}6.5$}} & $64.3$ & {\fontsize{6pt}{7pt}\selectfont\textcolor{gray}{$\pm\phantom{0}6.2$}} & $72.9$ & {\fontsize{6pt}{7pt}\selectfont\textcolor{gray}{$\pm\phantom{0}4.7$}} & $78.3$ & {\fontsize{6pt}{7pt}\selectfont\textcolor{gray}{$\pm\phantom{0}3.7$}} \\
ARS-LN-1024 & -- & 16267 & \multicolumn{2}{c}{--} & \multicolumn{2}{c}{--} & \multicolumn{2}{c}{--} & \multicolumn{2}{c}{--} & \multicolumn{2}{c}{--} & \multicolumn{2}{c}{--} & $88.9$ & {\fontsize{6pt}{7pt}\selectfont\textcolor{gray}{$\pm\phantom{0}2.6$}} \\
ARS-LoRA & 3878 & 16123 & $\mathbf{99.5}$ & {\fontsize{6pt}{7pt}\selectfont\textcolor{gray}{$\pm\phantom{0}0.8$}} & \multicolumn{2}{c}{--} & $29.3$ & {\fontsize{6pt}{7pt}\selectfont\textcolor{gray}{$\pm\phantom{0}7.7$}} & $58.4$ & {\fontsize{6pt}{7pt}\selectfont\textcolor{gray}{$\pm\phantom{0}4.5$}} & $72.0$ & {\fontsize{6pt}{7pt}\selectfont\textcolor{gray}{$\pm\phantom{0}3.0$}} & $79.0$ & {\fontsize{6pt}{7pt}\selectfont\textcolor{gray}{$\pm\phantom{0}2.2$}} & $83.2$ & {\fontsize{6pt}{7pt}\selectfont\textcolor{gray}{$\pm\phantom{0}1.8$}} \\
\midrule
\multicolumn{17}{@{}l}{\textit{SAC-MT and ARS family, our runs (10 seeds, \cref{app:revision-baselines}; SAC critic-balancing ports are in \cref{tab:revision-sac-transfer})}} \\
SAC-MT & 2517 & 2597 & $61.7$ & {\fontsize{6pt}{7pt}\selectfont\textcolor{gray}{$\pm\phantom{0}8.2$}} & $0.0$ & {\fontsize{6pt}{7pt}\selectfont\textcolor{gray}{$\pm\phantom{0}0.0$}} & $0.0$ & {\fontsize{6pt}{7pt}\selectfont\textcolor{gray}{$\pm\phantom{0}0.0$}} & $0.0$ & {\fontsize{6pt}{7pt}\selectfont\textcolor{gray}{$\pm\phantom{0}0.0$}} & $10.5$ & {\fontsize{6pt}{7pt}\selectfont\textcolor{gray}{$\pm\phantom{0}3.8$}} & $32.3$ & {\fontsize{6pt}{7pt}\selectfont\textcolor{gray}{$\pm\phantom{0}3.3$}} & $45.8$ & {\fontsize{6pt}{7pt}\selectfont\textcolor{gray}{$\pm\phantom{0}2.6$}} \\
ARS & 2517 & 2597 & $93.9$ & {\fontsize{6pt}{7pt}\selectfont\textcolor{gray}{$\pm\phantom{0}4.1$}} & $0.0$ & {\fontsize{6pt}{7pt}\selectfont\textcolor{gray}{$\pm\phantom{0}0.0$}} & $0.6$ & {\fontsize{6pt}{7pt}\selectfont\textcolor{gray}{$\pm\phantom{0}1.6$}} & $24.4$ & {\fontsize{6pt}{7pt}\selectfont\textcolor{gray}{$\pm\phantom{0}9.0$}} & $45.1$ & {\fontsize{6pt}{7pt}\selectfont\textcolor{gray}{$\pm\phantom{0}7.8$}} & $58.6$ & {\fontsize{6pt}{7pt}\selectfont\textcolor{gray}{$\pm\phantom{0}6.0$}} & $66.8$ & {\fontsize{6pt}{7pt}\selectfont\textcolor{gray}{$\pm\phantom{0}4.8$}} \\
ARS-LN-400 & 2530 & 2611 & $96.2$ & {\fontsize{6pt}{7pt}\selectfont\textcolor{gray}{$\pm\phantom{0}2.4$}} & $0.0$ & {\fontsize{6pt}{7pt}\selectfont\textcolor{gray}{$\pm\phantom{0}0.0$}} & $11.7$ & {\fontsize{6pt}{7pt}\selectfont\textcolor{gray}{$\pm\phantom{0}9.6$}} & $44.0$ & {\fontsize{6pt}{7pt}\selectfont\textcolor{gray}{$\pm11.3$}} & $61.2$ & {\fontsize{6pt}{7pt}\selectfont\textcolor{gray}{$\pm\phantom{0}9.1$}} & $70.8$ & {\fontsize{6pt}{7pt}\selectfont\textcolor{gray}{$\pm\phantom{0}7.1$}} & $76.6$ & {\fontsize{6pt}{7pt}\selectfont\textcolor{gray}{$\pm\phantom{0}5.7$}} \\
ARS-LN-1024 & 16062 & 16267 & $96.9$ & {\fontsize{6pt}{7pt}\selectfont\textcolor{gray}{$\pm\phantom{0}2.6$}} & $16.9$ & {\fontsize{6pt}{7pt}\selectfont\textcolor{gray}{$\pm\phantom{0}9.1$}} & $43.3$ & {\fontsize{6pt}{7pt}\selectfont\textcolor{gray}{$\pm\phantom{0}9.8$}} & $65.9$ & {\fontsize{6pt}{7pt}\selectfont\textcolor{gray}{$\pm\phantom{0}7.2$}} & $76.3$ & {\fontsize{6pt}{7pt}\selectfont\textcolor{gray}{$\pm\phantom{0}5.3$}} & $82.2$ & {\fontsize{6pt}{7pt}\selectfont\textcolor{gray}{$\pm\phantom{0}4.0$}} & $85.7$ & {\fontsize{6pt}{7pt}\selectfont\textcolor{gray}{$\pm\phantom{0}3.2$}} \\
ARS-LoRA & 3878 & 16123 & $98.4$ & {\fontsize{6pt}{7pt}\selectfont\textcolor{gray}{$\pm\phantom{0}3.1$}} & $1.0$ & {\fontsize{6pt}{7pt}\selectfont\textcolor{gray}{$\pm\phantom{0}2.8$}} & $19.0$ & {\fontsize{6pt}{7pt}\selectfont\textcolor{gray}{$\pm11.6$}} & $55.9$ & {\fontsize{6pt}{7pt}\selectfont\textcolor{gray}{$\pm\phantom{0}7.6$}} & $70.5$ & {\fontsize{6pt}{7pt}\selectfont\textcolor{gray}{$\pm\phantom{0}5.2$}} & $77.9$ & {\fontsize{6pt}{7pt}\selectfont\textcolor{gray}{$\pm\phantom{0}3.9$}} & $82.3$ & {\fontsize{6pt}{7pt}\selectfont\textcolor{gray}{$\pm\phantom{0}3.1$}} \\
\midrule
\multicolumn{17}{@{}l}{\textit{Ours (10 seeds)}} \\
Vanilla MT-PPO & 684 & 716 & $85.2$ & {\fontsize{6pt}{7pt}\selectfont\textcolor{gray}{$\pm\phantom{0}6.5$}} & $0.0$ & {\fontsize{6pt}{7pt}\selectfont\textcolor{gray}{$\pm\phantom{0}0.0$}} & $6.6$ & {\fontsize{6pt}{7pt}\selectfont\textcolor{gray}{$\pm\phantom{0}9.1$}} & $46.4$ & {\fontsize{6pt}{7pt}\selectfont\textcolor{gray}{$\pm\phantom{0}8.0$}} & $64.1$ & {\fontsize{6pt}{7pt}\selectfont\textcolor{gray}{$\pm\phantom{0}5.4$}} & $73.1$ & {\fontsize{6pt}{7pt}\selectfont\textcolor{gray}{$\pm\phantom{0}4.1$}} & $78.5$ & {\fontsize{6pt}{7pt}\selectfont\textcolor{gray}{$\pm\phantom{0}3.3$}} \\
\quad+LN-c & 685 & 717 & $96.2$ & {\fontsize{6pt}{7pt}\selectfont\textcolor{gray}{$\pm\phantom{0}3.4$}} & $6.9$ & {\fontsize{6pt}{7pt}\selectfont\textcolor{gray}{$\pm11.4$}} & $42.3$ & {\fontsize{6pt}{7pt}\selectfont\textcolor{gray}{$\pm10.1$}} & $68.9$ & {\fontsize{6pt}{7pt}\selectfont\textcolor{gray}{$\pm\phantom{0}5.9$}} & $79.2$ & {\fontsize{6pt}{7pt}\selectfont\textcolor{gray}{$\pm\phantom{0}4.0$}} & $84.4$ & {\fontsize{6pt}{7pt}\selectfont\textcolor{gray}{$\pm\phantom{0}3.0$}} & $87.5$ & {\fontsize{6pt}{7pt}\selectfont\textcolor{gray}{$\pm\phantom{0}2.4$}} \\
\quad+LN-c+PopArt & 685 & 717 & $96.8$ & {\fontsize{6pt}{7pt}\selectfont\textcolor{gray}{$\pm\phantom{0}1.5$}} & $18.0$ & {\fontsize{6pt}{7pt}\selectfont\textcolor{gray}{$\pm12.5$}} & $48.9$ & {\fontsize{6pt}{7pt}\selectfont\textcolor{gray}{$\pm\phantom{0}9.0$}} & $72.3$ & {\fontsize{6pt}{7pt}\selectfont\textcolor{gray}{$\pm\phantom{0}5.6$}} & $81.5$ & {\fontsize{6pt}{7pt}\selectfont\textcolor{gray}{$\pm\phantom{0}3.8$}} & $86.1$ & {\fontsize{6pt}{7pt}\selectfont\textcolor{gray}{$\pm\phantom{0}2.8$}} & $88.9$ & {\fontsize{6pt}{7pt}\selectfont\textcolor{gray}{$\pm\phantom{0}2.3$}} \\
\quad+LN-c+PopArt+FG-c & 685 & 717 & $94.8$ & {\fontsize{6pt}{7pt}\selectfont\textcolor{gray}{$\pm\phantom{0}5.1$}} & $22.3$ & {\fontsize{6pt}{7pt}\selectfont\textcolor{gray}{$\pm17.3$}} & $52.4$ & {\fontsize{6pt}{7pt}\selectfont\textcolor{gray}{$\pm11.3$}} & $75.1$ & {\fontsize{6pt}{7pt}\selectfont\textcolor{gray}{$\pm\phantom{0}5.7$}} & $83.4$ & {\fontsize{6pt}{7pt}\selectfont\textcolor{gray}{$\pm\phantom{0}3.8$}} & $87.6$ & {\fontsize{6pt}{7pt}\selectfont\textcolor{gray}{$\pm\phantom{0}2.8$}} & $90.1$ & {\fontsize{6pt}{7pt}\selectfont\textcolor{gray}{$\pm\phantom{0}2.3$}} \\
\textbf{TOPPO} & 685 & 717 & $97.2$ & {\fontsize{6pt}{7pt}\selectfont\textcolor{gray}{$\pm\phantom{0}3.1$}} & $\mathbf{24.2}$ & {\fontsize{6pt}{7pt}\selectfont\textcolor{gray}{$\pm14.1$}} & $\mathbf{56.5}$ & {\fontsize{6pt}{7pt}\selectfont\textcolor{gray}{$\pm\phantom{0}7.8$}} & $\mathbf{77.2}$ & {\fontsize{6pt}{7pt}\selectfont\textcolor{gray}{$\pm\phantom{0}4.0$}} & $\mathbf{84.8}$ & {\fontsize{6pt}{7pt}\selectfont\textcolor{gray}{$\pm\phantom{0}2.6$}} & $\mathbf{88.6}$ & {\fontsize{6pt}{7pt}\selectfont\textcolor{gray}{$\pm\phantom{0}2.0$}} & $\mathbf{90.9}$ & {\fontsize{6pt}{7pt}\selectfont\textcolor{gray}{$\pm\phantom{0}1.6$}} \\
\bottomrule
\end{tabular}
\end{table}

\paragraph{Cumulative ladder.} \cref{fig:ladder} traces how the four intervention effects accumulate on MT50. Vanilla MT-PPO with LN-c and PopArt already exceeds ARS-LN-1024 on both the mean and the worst-10 tail tasks (\cref{tab:combined-headline-tail}). Modules FG-c ($\alpha{=}1$) and PCGrad-a then add mechanism-aligned residual gains. Adding LN-c is the largest single jump on both metrics. Note that this effect ladder is not a unique additive decomposition, as the three critic interventions interact.

\paragraph{Convergence speed and parameter efficiency.}
\Cref{fig:train-curve}(a) shows that on MT50, TOPPO reaches every ARS variant's $100$M-step final mean within its first $\sim\!28$M steps and the published MOORE mean at $\sim\!11$M (full curves in \cref{app:v2-curves}). With $22.7\times$ fewer parameters than ARS-LN-$1024$, TOPPO sits on the parameter and sample efficiency frontier among the compared V2 methods.

\subsection{Ablation and analysis}\label{sec:exp:ablations}

We isolate each component's contribution via four ablations. Full details are in \cref{app:ablations}.

\textbf{Critic-side intervention cube} (LN-c $\times$ PopArt $\times$ FG-c, $8$ cells $\times$ $10$ seeds, MT50 V2). Each component in TOPPO is mechanism-distinct and they interact. The interaction signature suggested by \cref{cor:fg-popart} appears at training outcome scale: \emph{PopArt's worst-10 marginal collapses from $+6.6$ points (over LN-c) to $-0.2$ points (over LN-c+FG-c)}. This is consistent with FairGrad ($\alpha{=}1$) absorbing PopArt's per-task gradient rescaling at the gradient aggregator (\cref{tab:intervention-isolation}).
\begin{wrapfigure}[17]{r}{0.65\textwidth}
\centering
\includegraphics[width=\linewidth]{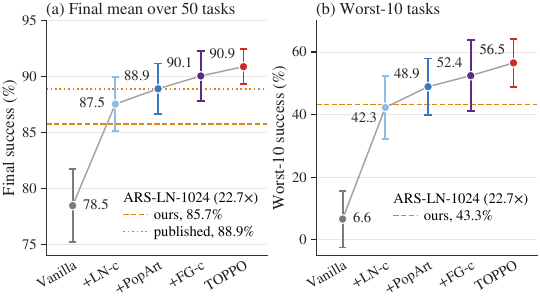}
\vspace{-0.5cm}
\caption{MT50 ablation ladder. (a) Final mean. (b) Final worst-10 tail tasks. $n{=}10$ seeds, $\pm$~std. Orange: ARS-LN-1024 (dashed: our $10$-seed runs. dotted: published)}
\label{fig:ladder}
\end{wrapfigure}
\textbf{Per-side combiner asymmetry.} Actor-side scale spread is small after per-task advantage normalization, leaving direction conflict as the main residual issue (\cref{fig:f3b}). This motivates our PCGrad-a/FG-c asymmetric design: adding PCGrad-a on top of the critic stack raises MT50 worst-10 from $52.4\%$ to $56.5\%$ and lowers its seed std from $11.3$ to $7.8$ points (\cref{tab:worst-k}). FG-a on top of FG-c does not improve on default-actor FG-c ($89.2\%$ vs.\ $90.1\%$, \cref{tab:per-side}). FG-c vs.\ PCGrad-c is a tie within seed noise. We adopt FG-c on the critic for the fixed $\sqrt{K}$-norm identity and scale-invariance of \cref{thm:scale-invariance} (see all per-side combiners in \cref{app:cone-containment}).

\textbf{V1 reward sensitivity.} V1 rewards have much wider per-task scale heterogeneity than V2, so PopArt's channels that FG-c does not absorb (running statistics, optimizer state, downstream actor) should matter more. \Cref{tab:v1-mwp} confirms this on MT50 V1: the PopArt marginal over LN-c+FG-c is $+17.0$ IQM points (vs.\ $+0.3$ on V2, \cref{tab:v1-sensitivity}), while FG-c without PopArt slightly hurts ($-2.3$ points over LN-c). The full stack remains the strongest: TOPPO reaches $79.3\%$ IQM versus $61.8\%$ for the strongest published V1 baseline (MOORE), with $14.5\times$ fewer parameters (learning curves in \cref{fig:train-curve-v1}).

\begin{table}[t]
\centering
\caption{Meta-World V1 (MT50): final-checkpoint IQM\,$\pm$\,std (10 seeds). Top: our V1 ablation. Bottom: published Meta-World+ baselines. Match step: first common checkpoint at which TOPPO's seed IQM reaches the published value.}
\label{tab:v1-mwp}
\footnotesize
\setlength{\tabcolsep}{3pt}
\begin{tabular}{@{}l@{\hspace{6pt}}c@{\hspace{6pt}}r@{\,}l@{\hspace{6pt}}c@{\hspace{14pt}}l@{}}
\toprule
Method & Params & \multicolumn{2}{c}{Success rate \%} & TOPPO step to match & Backbone \\
\midrule
Ours: Vanilla MT-PPO & 716K & $45.2$ & {\fontsize{5pt}{6pt}\selectfont\textcolor{gray}{$\pm\phantom{0}2.4$}} & -- & $[400]{\times}3$ \\
\quad +LN-c & 717K & $64.3$ & {\fontsize{5pt}{6pt}\selectfont\textcolor{gray}{$\pm\phantom{0}3.1$}} & -- & $[400]{\times}3$ \\
\quad +LN-c +FG-c & 717K & $62.0$ & {\fontsize{5pt}{6pt}\selectfont\textcolor{gray}{$\pm\phantom{0}2.7$}} & -- & $[400]{\times}3$ \\
\quad +LN-c +PopArt & 717K & $76.5$ & {\fontsize{5pt}{6pt}\selectfont\textcolor{gray}{$\pm\phantom{0}2.0$}} & -- & $[400]{\times}3$ \\
\quad +LN-c +PopArt +FG-c & 717K & $79.0$ & {\fontsize{5pt}{6pt}\selectfont\textcolor{gray}{$\pm\phantom{0}1.2$}} & -- & $[400]{\times}3$ \\
\textbf{TOPPO} & 717K & $\mathbf{79.3}$ & {\fontsize{5pt}{6pt}\selectfont\textcolor{gray}{$\pm\phantom{0}1.3$}} & -- & $[400]{\times}3$ \\
\midrule
\multicolumn{6}{l}{\emph{Meta-World+ MT50 SAC-family V1 baselines~\citep{metaworldplus2025}}} \\
\midrule
MT-MH-SAC & 2031K & $31.9$ & {\fontsize{5pt}{6pt}\selectfont\textcolor{gray}{$\pm\phantom{0}2.7$}} & $4$M (4\%) & $[400]{\times}3$ shared, per-task heads \\
Soft Modularization & 8030K & $60.6$ & {\fontsize{5pt}{6pt}\selectfont\textcolor{gray}{$\pm\phantom{0}3.9$}} & $14$M (14\%) & routing $[256]{\times}4{\times}4$ \\
MOORE & 10385K & $61.8$ & {\fontsize{5pt}{6pt}\selectfont\textcolor{gray}{$\pm\phantom{0}2.7$}} & $14$M (14\%) & $[400]{\times}3$, $6$ experts \\
PaCo & 33909K & $18.6$ & {\fontsize{5pt}{6pt}\selectfont\textcolor{gray}{$\pm15.4$}} & $4$M (4\%) & $[400]{\times}3$, $K{=}20$ \\
PCGrad-SAC & 2031K & $45.8$ & {\fontsize{5pt}{6pt}\selectfont\textcolor{gray}{$\pm\phantom{0}5.9$}} & $6$M (6\%) & $[400]{\times}3$ shared, per-task heads \\
\bottomrule
\end{tabular}
\end{table}

\textbf{Alternative critic normalizers.}
We also compare PopArt with reward normalization, return-based scaling~\citep{schaul2021returnscaling}, and an HL-Gauss critic~\citep{farebrother2024stopregressing} on MT50 with LN-c fixed (\cref{app:revision-normalizers,app:revision-hlgauss}).
Reward normalization and PopArt perform about the same under both reward versions: $79.8\%$ vs.\ $79.5\%$ IQM on V1 with value clipping off (paired $p=0.64$), and $89.5\%$ vs.\ $88.9\%$ mean on V2.
PopArt is thus one valid per-task value normalizer rather than a superior one. We keep it as default because it changes neither the reward nor the actor objective.

\section{Conclusion}\label{sec:conclusion}

We rethink PPO for MTRL under a vanilla shared MLP. Our results show that earlier conclusions drawn from MT-PPO's $\sim\!30\%$ MT10 success rate in the original Meta-World benchmark~\citep{yu2020metaworld} underestimated PPO by confounding the choice of optimizer family with the condition of the critic.
We identify critic-side gradient ill-conditioning as the main failure mode: large per-task scale spread, representation-level co-linear collapse, and unfair aggregation of imbalanced gradients.
TOPPO targets these failures with PopArt, critic LayerNorm, and critic-side FairGrad ($\alpha{=}1$), linked by the scale-invariance result of \cref{thm:scale-invariance,cor:fg-popart}.
With a compact shared MLP, TOPPO outperforms the SAC- and ARS-family baselines on MT50, most clearly on the tail tasks, while using far fewer parameters and reaching their final performance in fewer environment steps.

\paragraph{Limitations.}
We do not claim empirical dominance of FG-c over all critic-side combiners: FG-c and PCGrad-c are tied within seed noise under LN-c+PopArt, and we adopt FG-c for the property in \cref{thm:scale-invariance}. We also did not exhaustively test actor-side aggregator choices beyond the per-side cube in \cref{tab:per-side} and the actor sweep in \cref{app:revision-actor}. Therefore, PCGrad-a should be read as a mechanism-aligned working choice rather than an optimal actor combiner.
Our reported architecture also uses separate actor and critic networks. If an actor--critic model uses a shared feature extractor before branching into policy and value heads, assigning critic-side balancing and actor-side surgery on the shared parameters becomes a coupled design choice that requires separate study.
Our evidence covers task-conditioned, state-based MTRL, not visual inputs, missing task IDs, or unlabeled tasks.
TOPPO also costs more computation per environment step: $19.8$ versus $8.7$ RTX-3090 hours per $100$M steps for vanilla MT-PPO (\cref{app:revision-compute}).
Porting the critic-balancing to MTSAC gives no reliable final gain (\cref{app:revision-sac}), so our claims are limited to the on-policy PPO setting.

\paragraph{Potential use cases.}
The same critic-side ill-conditioning can arise in other PPO settings where one shared critic fits value targets of different scales, such as multi-objective control.
In PPO-based humanoid walk steering with an adversarial motion prior~\citep{peng2021amp}, applying TOPPO to the task and style rewards substantially improves motion style, with a tradeoff against task return (\cref{app:revision-amp}).

\paragraph{Future work.}
Our study uses compact shared MLPs, where per-task critic gradients show large scale spread and strong co-linearity.
In much larger networks, per-task gradients are closer to orthogonal, so cross-task interference, and with it the benefit of critic-side balancing, may shrink. Adapting it to model scale is left for future work.

\begin{ack}
Miao's research is partially supported by Texas AI Research Institute.
Qu's research is partially supported by NCI Grant R01CA297869 and NSF Grant DMS 2627378. 
\end{ack}

\clearpage
\bibliographystyle{plainnat}
\bibliography{references}

\newpage
\appendix
\crefalias{section}{appendix}
\crefalias{subsection}{appendix}

\let\origappendixsubsection\subsection
\renewcommand{\subsection}{\FloatBarrier\origappendixsubsection}

\counterwithin{figure}{section}
\counterwithin{table}{section}
\counterwithin{equation}{section}
\counterwithin{algorithm}{section}
\counterwithin{theorem}{section}
\counterwithin{proposition}{section}
\counterwithin{corollary}{section}
\counterwithin{remark}{section}
\counterwithin{lemma}{section}
\counterwithin{definition}{section}
\renewcommand{\thefigure}{\thesection.\arabic{figure}}
\renewcommand{\thetable}{\thesection.\arabic{table}}
\renewcommand{\theequation}{\thesection.\arabic{equation}}
\renewcommand{\thealgorithm}{\thesection.\arabic{algorithm}}
\renewcommand{\thetheorem}{\thesection.\arabic{theorem}}
\renewcommand{\theproposition}{\thesection.\arabic{proposition}}
\renewcommand{\thecorollary}{\thesection.\arabic{corollary}}
\renewcommand{\theremark}{\thesection.\arabic{remark}}
\renewcommand{\thelemma}{\thesection.\arabic{lemma}}
\renewcommand{\thedefinition}{\thesection.\arabic{definition}}

\etocdepthtag.toc{appendix}
\begingroup
\hypersetup{linkcolor=black}  %
\etocsettagdepth{maintext}{none}
\etocsettagdepth{appendix}{subsection}
\etocsetnexttocdepth{subsection}
\etocsettocstyle
  {\section*{Appendix Contents}\noindent\rule{\linewidth}{0.4pt}\par\smallskip}
  {\par\smallskip\noindent\rule{\linewidth}{0.4pt}\par\bigskip}
\tableofcontents
\endgroup

\section{Extended Related Work}\label{app:related-extended}

\paragraph{Structure-based multi-task RL.}
Architectural interventions form one line of multi-task RL: Soft Modularization~\citep{yang2020softmodules}, CARE~\citep{sodhani2021care}, PaCo~\citep{sun2022paco}, and MOORE~\citep{hendawy2023moore}.
The original Meta-World paper~\citep{yu2020metaworld} reported MT-PPO at $\sim 30\%$ on MT10, anchoring a community shift to SAC backbones for subsequent multi-task RL methods (CARE; PaCo; MOORE; PCGrad-SAC; ARS, \citealp{cho2025ars}); benchmarks have largely standardized on the SAC stack~\citep{metaworldplus2025}.
We keep a fully task-shared $[400{\times}3]$ MLP backbone with no task-specific subnetworks, experts, or routers: vanilla MT-PPO on this backbone revisits the SAC--PPO comparison, and architecturally TOPPO adds only critic LayerNorm and a per-task PopArt affine head (\cref{tab:hyperparameters}).
Concurrent work~\citep{joshi2025mtbench} observes via cross-task gradient cosine that conflicts dominate on the critic in MT-PPO and bypasses the critic via Monte-Carlo returns, whereas we diagnose a complementary scale-imbalance axis ($497\times$ critic vs.\ $4.1\times$ actor per-task gradient-norm spread) and address it with the three critic-side modules of \cref{tab:three-layer-map}.

\paragraph{Imbalance interventions across pipeline layers.}
\emph{Per-task gradient aggregation:} scalar-signal reweighting (loss-level: uncertainty weighting, \citealp{kendall2018uncertainty}; DWA, \citealp{liu2019dwa}; FAMO, \citealp{liu2023famo}; per-task gradient-norm: GradNorm, \citealp{chen2018gradnorm}) is Gram-blind; direction surgery uses Gram inner products (PCGrad, \citealp{yu2020pcgrad}; CAGrad, \citealp{liu2021cagrad}); $\alpha$-fairness reweighting solves $\bG\bw = \bw^{-1/\alpha}$ on the full Gram (FairGrad, \citealp{ban2024fairgrad}; at $\alpha{=}1$ equivalent to Nash bargaining, \citealp{navon2022nashmtl}, \cref{rem:fg-nash}).
\citet{kurin2022unitary} argue that across supervised MTL benchmarks and Meta-World multi-task SAC, the simple uniform sum of losses (unitary scalarization) matches specialized multi-task optimizers once standard regularization is matched; on MT-PPO, our diagnostics show a $497\times$ critic per-task gradient-norm spread (\cref{fig:f3b}), which motivates a scale-invariant aggregator; FairGrad ($\alpha{=}1$) provides one (\cref{thm:scale-invariance,cor:fg-popart}).
On MT50 V1, which has wider per-task return-scale heterogeneity, PopArt's marginal is much larger than on V2 ($+17.0$ vs.\ $+0.3$ IQM points over LN-c+FG-c, \cref{tab:v1-sensitivity}), and FG-c without PopArt does not substitute for it: adding FG-c to LN-c lowers V1 IQM from $64.3$ to $62.0$.
\Cref{app:cone-containment} positions these methods as different points in the same shared cone $\mathrm{cone}^+\{\bg_i\}$; only FairGrad ($\alpha{=}1$) carries the scale-invariance of \cref{thm:scale-invariance} that links it to PopArt via \cref{cor:fg-popart}.
\emph{Value-target normalization} (PopArt, \citealp{vanhasselt2016popart,hessel2019popartmultitask}; ARS rescales the per-task reward fed into the Bellman target, \citealp{cho2025ars}) and \emph{representation regularization} (LayerNorm, \citealp{ba2016layernorm}; applied in deep RL by \citealp{lee2024simba,hansen2024tdmpc2}) are standard.
\Cref{cor:fg-popart} shows that, within a minibatch, FG-c absorbs PopArt's per-task gradient rescaling by $\sigma_i^{-2}$; PopArt's other effects (running statistics, optimizer state, and the downstream actor) are not absorbed.

\FloatBarrier
\section{Implementation Details}\label{app:details}
\subsection{Detailed experimental setup}\label{app:setup-details}

This appendix expands the main-text \cref{sec:exp:setup} with the full benchmark, architecture, training, and evaluation description; hyperparameter values are tabulated in \cref{tab:hyperparameters} (\cref{app:hyperparameters}), PPO-specific adaptations of SAC-based gradient-surgery methods are documented in \cref{app:ppo-adaptations}, the source-counted parameter audit against published baselines is in \cref{app:param-audit}, and cross-paper protocol differences are mapped in \cref{app:reporting-protocol}.

\paragraph{Benchmark and observations.}
We use Meta-World~3.0.0~\citep{yu2020metaworld,metaworldplus2025} with V2 reward semantics throughout the paper unless explicitly labeled V1.
The $K$ tasks (reaching, picking, pushing, opening doors, peg-insertion, etc.) share the same continuous observation space $\mathcal{S}$ and continuous action space $\mathcal{A}$; what varies across tasks is the reward function and the dynamics.
The agent receives the state concatenated with a one-hot task identifier of length $K$ as input, and is trained on all $K$ tasks simultaneously via stratified minibatches.
Meta-World offers two reward variants: V1 is the original per-task reward set from~\citet{yu2020metaworld} with substantial per-task scale heterogeneity, while V2~\citep{metaworldplus2025} redesigns the per-task rewards with PPO solvability and comparable per-task scales as explicit design criteria.
We evaluate TOPPO on \emph{both} variants.
Published V2 means and V1 IQMs use our reporting statistic, whereas published worst-$k$ values are the bottom-$k$ scores reported by \citet{cho2025ars} (\cref{app:revision-baselines}); none of these published values uses our environment and evaluation protocol.
Our SAC-MT and ARS runs share TOPPO's Meta-World~3.0.0 environment and evaluation protocol (\cref{app:revision-baselines}).
V2 carries the headline (MT50 V2 is the main number; MT10 V2 saturates near ceiling and is reported for completeness, \cref{sec:exp:headline}); V1 is the sensitivity sweep that also tests \cref{cor:fg-popart}'s scope under heterogeneous per-task return scales (\cref{app:v1-sensitivity}).

\paragraph{Architecture.}
Actor and critic are separate networks of identical $[400, 400, 400]$ shape (full architecture in \cref{tab:hyperparameters}); this is the standard SAC-family Meta-World+ backbone used by CAGrad~\citep{liu2021cagrad}, CARE~\citep{sodhani2021care}, PaCo~\citep{sun2022paco}, MOORE~\citep{hendawy2023moore}, FAMO~\citep{liu2023famo}, and FairGrad~\citep{ban2024fairgrad}, while the ARS codebase~\citep{cho2025ars}, including its SAC-MT baseline, is the architectural outlier with a deeper $[400 \times 4]$ default or a wider $[1024 \times 4]$ ARS-LN-1024 variant.
The task-specific quantities are the input one-hot, an $|\mathcal{A}|$-dimensional action log-stddev per task, and, when PopArt is enabled, two trainable scalar affine parameters $(W_i,b_i)$ per task together with running-statistic buffers $(\mu_i,\sigma_i)$.
The actor mean and the pre-PopArt scalar value projections are shared; there are no task-specific hidden subnetworks, expert mixtures, LoRA adapters, or routers.
Headline parameter counts are $685$K (MT10) and $717$K (MT50); per-component breakdowns and ratios against published baselines are tabulated in \cref{tab:params} (\cref{app:param-audit}).

\paragraph{Training pipeline.}
Our PPO implementation builds on Tianshou~2.0.0~\citep{weng2022tianshou} (\url{https://github.com/thu-ml/tianshou}; MIT License).
Stratified minibatch sampling guarantees the per-task gradients $\bg_i^a$, $\bg_i^c$ on the shared parameters are well-defined every minibatch.
Per-task advantage normalization (zero-mean unit-variance computed within each task's slice of a minibatch) is applied before the actor surrogate; this task-wise adaptation of a PPO implementation option is the actor-side mechanism that drives the actor-critic asymmetry diagnosed in \cref{sec:method:pathology}.
The headline and intervention configurations retain PPO's value clip and a post-aggregation gradient-norm clip.
Additional normalizer ablations explicitly vary value clipping; HL-Gauss uses an unclipped cross-entropy critic loss.

\paragraph{Evaluation protocol.}
All of our reported success rates come from \emph{evaluation} rollouts under the deterministic policy mean (no exploration noise), not training rollouts.
Following the Meta-World convention~\citep{yu2020metaworld,metaworldplus2025}, at each evaluation checkpoint the agent is rolled out once per training goal position per task: $50$ goal positions $\times$ $K$ tasks $=$ $500$ episodes per evaluation on MT10 and $2{,}500$ on MT50.
A per-task ``success rate'' is the fraction of those episodes whose success flag is $1$ at \emph{any} step during the episode (the default Meta-World success criterion); the reported scalar is the cross-task mean of per-task success rates.

\paragraph{Reporting choices.}
We run $10$ seeds and report final-checkpoint success rates.
To match each baseline's published convention, we report mean $\pm$ std against the V2 baselines from~\citet{cho2025ars} (\cref{tab:headline,tab:worst-k}) and IQM $\pm$ std against the Meta-World+ V1 baselines (\cref{tab:v1-mwp}); for our rows, both views are computed from the same $10$-seed run set.
\Cref{tab:combined-headline-tail} lists SAC-MT and the ARS family as run by us for $10$ seeds under our evaluation protocol (\cref{app:revision-baselines}).
The full $4$-metric breakdown of every reported configuration is in \cref{tab:4metric}, cross-paper convention conversions are in \cref{app:reporting-protocol}, and per-task breakdowns are in \cref{app:per-task}.

\paragraph{Compute resources.}
Each MT10 and MT50 V2 headline run trained one seed on one GPU with $16$ CPU cores: an NVIDIA RTX~3090 for every MT10 seed and for MT50 seeds $3$--$9$, and an NVIDIA H200 NVL for MT50 seeds $0$--$2$.
The allocated host memory was $24$\,GB per MT10 run and $32$--$64$\,GB per MT50 run.
Including its $51$ evaluations, one MT10 $20$M-step seed took $5$--$8$ wall-clock hours, and one MT50 $100$M-step seed took $23$--$32$ hours on an RTX~3090 and about $14$ hours on the H200 NVL.
On the RTX~3090, five MT50 seeds took $23$--$24$ hours each; the other two were preempted once, resumed from a checkpoint, and took $31$--$32$ hours from start to finish.
The ten MT50 V1 TOPPO runs each used one NVIDIA Tesla V100 (32\,GB) with $12$ CPU cores and $64$\,GB of host memory, and with the same evaluation schedule each took $40$--$43$ hours.
These are end-to-end wall-clock durations; \cref{app:revision-compute} reports controlled timings with evaluation disabled.
The headline schedule includes $51$ evaluation checkpoints during training (one every $2$ epochs across $100$ epochs, plus initialization); evaluation environment steps are excluded from the $20$M\,/\,$100$M training budget.
TOPPO reaches the full-budget mean of ARS-LN-1024 after only $\sim\!28\%$ of the environment-step budget (\cref{sec:exp:headline}, \cref{fig:train-curve}(a)); we train for the full budget to match the published final-checkpoint reporting convention.

\subsection{Hyperparameters and architecture}\label{app:hyperparameters}

\paragraph{Inherited from Meta-World+.}
The PPO numerical hyperparameters in \cref{tab:hyperparameters} (loss coefficients, GAE $\lambda$, clip range, learning rate, PPO repeats per collect, minibatch size, gradient-norm clip, MLP backbone, per-rollout step counts and total environment-step budgets) follow the Meta-World+ companion repository~\citep{metaworldplus2025} (\url{https://github.com/rainx0r/metaworld-algorithms}); we do not separately tune these values.
The realized minibatch count changes under our stratified sampler, as detailed below.
The one substantive deviation from their MT50 example config is that we keep value clipping on (their MT50 example disables it).
Per-task advantage normalization and stratified minibatch sampling are PPO-side multi-task adaptations on top of this baseline (\cref{app:ppo-adaptations}); FairGrad / PCGrad / CAGrad / PopArt / LayerNorm hyperparameters are listed under their respective entries below.

\begin{table}[!htbp]
\centering
\caption{Hyperparameters for the headline experiments. Normalizer and actor-combiner ablations in \cref{app:revision} state their overrides separately. The actor exponent $\alpha_a$ applies only to runs with a FairGrad actor (FG-a: \cref{tab:per-side} and the mechanism-figure runs of \cref{app:mech-figs}); the PopArt statistics follow \cref{app:popart-welford}. The three FairGrad solver entries below the per-side $\alpha$ rows ($y_{\max}$, $\eta$, max Newton iters) are fixed solver constants.}
\label{tab:hyperparameters}
\small
\resizebox{\textwidth}{!}{\begin{tabular}{lll}
\toprule
Parameter & MT10 & MT50 \\
\midrule
\multicolumn{3}{l}{\emph{Shared backbones with task-local action scale and PopArt affine}} \\
\quad actor $\btheta^a$ hidden sizes & $[400, 400, 400]$ & $[400, 400, 400]$ \\
\quad actor output projection & shared $\mathrm{Linear}(400, |\mathcal{A}|)$ for $\mu$ & shared $\mathrm{Linear}(400, |\mathcal{A}|)$ for $\mu$ \\
\quad critic $\btheta^c$ hidden sizes & $[400, 400, 400]$ & $[400, 400, 400]$ \\
\quad critic output projection & shared $\mathrm{Linear}(400, 1)$ for $V$ & shared $\mathrm{Linear}(400, 1)$ for $V$ \\
\quad per-task action log-stddev & $|\mathcal{A}|$ params per task & $|\mathcal{A}|$ params per task \\
\quad PopArt affine $(W_i,b_i)$ & $2$ params per task & $2$ params per task \\
\quad task identity input & one-hot, length $K = 10$ & one-hot, length $K = 50$ \\
\quad activation & ReLU & ReLU \\
\quad LayerNorm placement (TOPPO) & critic hidden layers only & critic hidden layers only \\
\midrule
\multicolumn{3}{l}{\emph{Training}} \\
\quad epochs & 100 & 100 \\
\quad env steps per epoch & 200{,}000 & 1{,}000{,}000 \\
\quad env steps per collect & 100{,}000 & 500{,}000 \\
\quad PPO repeats per collect & 16 & 16 \\
\quad minibatch size (= step/32) & 3{,}125 & 15{,}625 \\
\quad minibatch sampling & stratified & stratified \\
\quad realized minibatches per repeat & $31$ & $31$ \\
\quad learning rate (Adam) & $3\!\times\!10^{-4}$ & $3\!\times\!10^{-4}$ \\
\quad max gradient norm (post-aggregation) & $1.0$ & $1.0$ \\
\quad discount $\gamma$ & $0.99$ & $0.99$ \\
\midrule
\multicolumn{3}{l}{\emph{PPO}} \\
\quad GAE $\lambda$ & $0.97$ & $0.97$ \\
\quad clip range $\epsilon$ & $0.2$ & $0.2$ \\
\quad value clip & enabled & enabled \\
\quad value loss weight $c_v$ & $0.001$ & $0.001$ \\
\quad entropy bonus $c_e$ & $0.005$ & $0.005$ \\
\quad per-task advantage normalization & enabled & enabled \\
\quad recompute advantage per repeat & enabled & enabled \\
\midrule
\multicolumn{3}{l}{\emph{Gradient aggregation and critic normalization (TOPPO)}} \\
\quad actor combiner & PCGrad (PCGrad-a) & PCGrad (PCGrad-a) \\
\quad critic combiner & FairGrad, $\alpha_c = 1$ (FG-c) & FairGrad, $\alpha_c = 1$ (FG-c) \\
\quad critic normalization & LN-c $+$ PopArt (Welford, $\sigma_{\min} = 10^{-2}$) & LN-c $+$ PopArt (Welford, $\sigma_{\min} = 10^{-2}$) \\
\midrule
\multicolumn{3}{l}{\emph{FairGrad solver}} \\
\quad actor $\alpha_a$ (FG-a runs only) & $2.0$ & $2.0$ \\
\quad critic $\alpha_c$ & $1.0$ & $1.0$ \\
\quad solver y-cap $y_{\max}$ & $50$ & $50$ \\
\quad solver Newton tol $\eta$ & $10^{-2}$ & $10^{-2}$ \\
\quad solver max Newton iters & $50$ & $50$ \\
\midrule
\multicolumn{3}{l}{\emph{Evaluation}} \\
\quad eval episodes per task & $50$ & $50$ \\
\quad seeds reported & $10$ & $10$ \\
\bottomrule
\end{tabular}}
\end{table}
\FloatBarrier

\paragraph{LayerNorm placement notes.}
\begin{itemize}
\item Pre-activation LN on the hidden \texttt{Linear} layers of the critic only; actor hidden layers are LN-free.
On MT10 (PopArt, default actor and critic combiners, 10 seeds), SimBa/TD-MPC2-style symmetric LN~\citep{lee2024simba,hansen2024tdmpc2} (LN on both actor and critic hidden layers) reaches $93.1\%$ final mean success versus $96.8\%$ for critic-only LN, with about $3\times$ wider seed std (\cref{app:ablations}).
\item LN-enabled critic hidden \texttt{Linear} uses \texttt{bias=False}; the LayerNorm affine parameters remain trainable, and the critic output projection retains its bias.
\end{itemize}

\subsection{Parameter audit (source-counted)}\label{app:param-audit}

\paragraph{Convention.}
Throughout the paper, every ``parameter count'' refers to the \emph{total number of model parameters held in memory during training, including target copies and excluding optimizer state}.
This measures model storage, not total GPU memory, which also includes activations, gradients, optimizer states, and framework buffers.
For our PPO this is just $\text{actor} + \text{critic}$ (no target network).
For SAC-style baselines this is $\text{actor} + 2{\times}\text{critic} + 2{\times}\text{target\_critic} + \alpha$, where the twin Q-critics are stacked on a $\texttt{num\_critics}{=}2$ axis and the Polyak-averaged target copy is updated by EMA rather than gradient descent.
We include the target copy because it occupies the same GPU memory as the live critic.
A trainable-only accounting removes the target critic ensemble from each SAC row in \cref{tab:params}; ARS-LoRA holds two target ensembles and also freezes its base network during LoRA training, leaving $820$K (MT10) / $8{,}235$K (MT50) trainable parameters.
For example, on MT50, ARS-LN-1024 drops from $16{,}267$K resident to $9{,}758$K trainable parameters, still $13.6\times$ TOPPO's $717$K; TOPPO remains the smallest model under either accounting.

\paragraph{Provenance.}
Numbers in \cref{tab:params} are counted from code, since most of the corresponding papers do not report scalar parameter counts: we instantiate each method's example configuration and sum the leaves of the parameter pytree.
ARS-paper baselines (SAC-MT, MT-MH-SAC, Soft Modularization, PCGrad-SAC, MOORE, ARS, ARS-LN, ARS-LoRA) come from the official ARS JAX/Flax release~\citep{cho2025ars} (\url{https://github.com/ms-cho/Adaptive-Reward-Scaling-ARS-for-MTRL}); the Meta-World+ MT-family rows (MT-MH-SAC, PCGrad-SAC, Soft Modularization, MOORE, PaCo) come from the companion \texttt{metaworld-algorithms} JAX repository~\citep{metaworldplus2025} (\url{https://github.com/rainx0r/metaworld-algorithms}).

\paragraph{Cross-implementation differences.}
Same-named methods can differ between the two source codebases.
For example, MOORE is $6{,}890$K (MT10) $/$ $7{,}403$K (MT50) in the ARS reproduction (4 orthogonal $400$-wide experts on a SAC-MT backbone) and $6{,}810$K (MT10) $/$ $10{,}385$K (MT50) in the Meta-World+ reproduction (4 experts at MT10, 6 at MT50, plus per-task output heads).
We never mix counts across implementations within a single comparison row: \cref{tab:headline} uses the ARS-paper figures for every V2 baseline implemented in the ARS release (PaCo and SMT are not, so their counts are left blank), matching the published success rates, which come from Tables~1, 2, and~6 of~\citet{cho2025ars} (SMT from~\citet{cho2024smt}); our SAC-MT and ARS runs use the same architectures and therefore the same counts.
\Cref{tab:v1-mwp} uses the Meta-World+ figures for V1 baselines.

\paragraph{Headline ratios.}
TOPPO at the headline configuration (LN-c $+$ PopArt on the critic; FG-c on the critic gradient aggregator adds $0$ parameters) is $685$K (MT10) / $717$K (MT50).
LN-c adds $3\times400=1{,}200$ parameters net of dropped hidden-layer Linear biases; PopArt adds $2K$ trainable scalar affine parameters.
Their combined increment is $1{,}220$ parameters on MT10 and $1{,}300$ on MT50.
The no-LN-c, no-PopArt rung (the vanilla MT-PPO entry in \cref{tab:combined-headline-tail}) has rounded counts $684$K / $716$K.
Rounded ratios below all use the MT50 headline denominator ($717$K):
\begin{itemize}
    \item vs.\ ARS default ($2{,}597$K): $27.6\%$ of its size ($3.6\times$ smaller).
    \item vs.\ ARS-LN-$1024$ ($16{,}267$K): $4.4\%$ of its size ($22.7\times$ smaller).
    \item vs.\ ARS-LoRA ($16{,}123$K): $4.4\%$ of its size ($22.5\times$ smaller).
    \item vs.\ PaCo K=$20$ in Meta-World+ ($33{,}909$K): $2.1\%$ of its size ($47.3\times$ smaller).
    \item vs.\ MOORE in Meta-World+ ($10{,}385$K): $6.9\%$ of its size ($14.5\times$ smaller).
\end{itemize}
Full table: \cref{tab:params}.

\begin{table}[!htbp]
\centering
\caption{\textbf{Source-counted parameter audit (training-memory totals).} TOPPO ($685$K MT10 $/$ $717$K MT50) is $3.6\times$, $22.7\times$, and $47.3\times$ smaller than ARS, ARS-LN-1024, and PaCo K=$20$, respectively, at MT50. The ARS-paper block~\citep{cho2025ars} and Meta-World+ block~\citep{metaworldplus2025} are non-identical reimplementations of same-named methods; counts are never mixed within a comparison row (see \cref{app:param-audit}).}
\label{tab:params}
\small
\resizebox{\textwidth}{!}{\begin{tabular}{l r r l}
\toprule
Method (architecture) & Parameters (MT10) & Parameters (MT50) & Notes \\
\midrule
\textbf{Ours: TOPPO} $[400\times 3]$ & 685K & 717K & headline; FG-c aggregator adds 0 params \\
\quad Ours: Vanilla MT-PPO $[400\times 3]$ & 684K & 716K & no LN-c, no PopArt (Vanilla MT-PPO rung in \cref{tab:combined-headline-tail}) \\
ARS paper: SAC-MT & 2,517K & 2,597K & actor $+2{\times}$critic $+2{\times}$target $+\alpha$ \\
\quad MT-MH-SAC & 2,546K & 2,754K & +per-task output head \\
\quad Soft Modularization & 3,454K & 3,534K & task-embed + 2 modules $\times$ depth 2 \\
\quad PCGrad-SAC & 2,517K & 2,597K & PCGrad arch-identical to SAC-MT \\
\quad MOORE & 6,890K & 7,403K & 4 orthogonal 400-wide experts \\
\quad ARS & 2,517K & 2,597K & SAC-MT backbone; differs only in loss \\
\quad ARS-LN-400 & 2,530K & 2,611K & +LN-c \\
\textbf{\quad ARS-LN-1024} & 16,062K & 16,267K & $22.7\times$ TOPPO (MT50) \\
\quad ARS-LoRA & 3,878K & 16,123K & base + LoRA + 2 targets \\
\quad PaCo / SMT & -- & -- & not implemented in ARS reference repo \\
Meta-World+: MT-MH-SAC / PCGrad-SAC & 1,759K & 2,031K & shared $[400]{\times}3$ torso, per-task heads \\
\quad Soft Modularization & 3,454K & 8,030K & (depth, num\_modules) = (2,2) MT10 / (4,4) MT50 \\
\quad MOORE & 6,810K & 10,385K & 4 experts MT10 / 6 experts MT50 \\
\quad PaCo (K=5 / K=20) & 8,476K & 33,909K & compositional Dense; K parameter sets \\
\bottomrule
\end{tabular}}
\end{table}

\subsection{PPO-specific adaptations of SAC-based gradient-surgery methods}\label{app:ppo-adaptations}

\paragraph{Provenance.}
PCGrad, CAGrad, and FairGrad are formulated as generic per-task gradient operators and evaluated primarily on supervised MTL benchmarks (CityScapes / NYUv2 / QM9); their Meta-World evaluations all use SAC (PCGrad / CAGrad on MT10 and MT50; FairGrad on MT10 only)~\citep{yu2020pcgrad,liu2021cagrad,ban2024fairgrad}.
Porting to PPO requires per-task loss decomposition, per-task backward, and a per-side combiner choice (\cref{sec:method:framework}); the standard PPO machinery (GAE, ratio clip, value clip, Adam, $[400{\times}3]$ MLP) is unchanged, and the four interventions of \cref{sec:method:framework} (\cref{tab:three-layer-map}; hyperparameters in \cref{tab:hyperparameters}) are the entire delta from vanilla MT-PPO.
This subsection documents the few items that are not pinned down by the table or the method section.

\paragraph{Baseline hyperparameters preserved verbatim.}
We do not re-tune any baseline's published hyperparameters.
CAGrad $c = 0.5$ (MT50) and $c = 0.9$ (MT10) per~\citet{liu2021cagrad}; PCGrad has no solver hyperparameters.
The per-side FairGrad exponents ($\alpha_c=1$, \cref{sec:method:fgc}; $\alpha_a=2$, swept in \cref{app:revision-actor}) are set in this work and listed in \cref{tab:hyperparameters}.

\paragraph{Solver call frequency.}
Whereas the SAC originals invoke the solver once per environment step, PPO invokes it once per minibatch inside each rollout's PPO-epoch loop.
Although the nominal minibatch size is the collect size divided by $32$, the stratified sampler allocates integer quotas of $312$ or $313$ samples per task and stops when any task cannot fill its quota.
With $10{,}000$ collected transitions per task, this yields $\lfloor 10{,}000/313\rfloor=31$ minibatches per repeat, dropping the remaining $1/32$ of the collected batch from that repeat.
Thus $16\times31=496$ combiner calls occur per collect on each enabled side, or $99{,}200$ calls across $200$ collects at the full budget.
This realized count is also recorded by the real-Gram benchmark in \cref{app:solver}.

\paragraph{Task-local parameters bypass surgery.}
Only the shared actor and critic backbone parameters feed the per-side combiner.
The task-local action log-stddev and PopArt affine parameters $(W_i,b_i)$ receive the mean of the per-task gradients, bypassing FairGrad and PCGrad.
They are reassembled with the combined shared gradients before the global norm clip and Adam update.
The actor/critic partition is exposed to the combiner by a parameter-filter helper that slices the per-task gradient list to each side's flat slab before invoking it.

\paragraph{FairGrad-specific containment.}
The released solver does not add a ridge term to the Gram matrix.
It floors Gram diagonals at $\varepsilon=10^{-8}$ for the warm-start calculation, caps Newton log-weights at $y_{\max}=50$, and floors final weights at $\varepsilon$.
The SciPy fallback has a positive lower bound but no upper weight bound.
\Cref{thm:scale-invariance} holds for an exact positive KKT solution; \cref{app:solver} quantifies how these safeguards and the absolute stopping tolerance affect it.
Per-task weights, contributions, residuals, and fallback use are logged, and the reassembled gradient is clipped to norm at most $1$ before Adam.

\subsection{PopArt: implementation note (Welford full-sample statistics)}\label{app:popart-welford}

Our PopArt implementation updates the per-task running statistics $(\mu_i, \sigma_i)$ via Chan--Golub--LeVeque parallel Welford merges~\citep{welford1962,chan1983variance} over the full per-task return history.
Canonical PopArt~\citep[Eq.~4]{vanhasselt2016popart} admits both options: $\beta_t = 1/t$ recovers the full-history sample mean, while a constant-rate EMA at $\beta$ is the recommended choice for non-stationary settings.
Canonical PopArt, including the multi-task setup of~\citet{hessel2019popartmultitask}, uses the constant-$\beta$ EMA in practice; our Welford accumulator instead corresponds to the $\beta_t = 1/t$ case of the same update.
After each collect of $n_{\text{batch}}$ return targets for task $i$ with batch mean $\bar R_i$ and batch sum-of-squared-deviations $M_2^{\text{batch}}$:
\begin{equation}\label{eq:welford}
\begin{aligned}
n_{\text{new}} &\leftarrow \mathrm{count}[i] + n_{\text{batch}}, \qquad
\delta \leftarrow \bar R_i - \mu_i, \\
\mu_i &\leftarrow \mu_i + \delta \cdot \tfrac{n_{\text{batch}}}{n_{\text{new}}}, \\
M_2[i] &\leftarrow M_2[i] + M_2^{\text{batch}} + \delta^2 \cdot \tfrac{\mathrm{count}[i]\, n_{\text{batch}}}{n_{\text{new}}}, \\
\sigma_i &\leftarrow \max\!\big(\sqrt{M_2[i] / n_{\text{new}}},\, \sigma_{\min}\big), \qquad
\mathrm{count}[i] \leftarrow n_{\text{new}},
\end{aligned}
\end{equation}
with $\sigma_{\min} = 10^{-2}$.
The PopArt renormalization that adjusts $(W_i, b_i)$ to preserve $V_i^{\text{raw}}$ across $(\mu_i, \sigma_i)$ updates runs identically to canonical PopArt and is unaffected by the choice of estimator.

The effective update rate is $n_{\text{batch}}/n_{\text{new}}$, which decays as $\sim 1/n$ and becomes small after many collects.
This estimates the cumulative return distribution over training.
The released EMA alternative instead uses a per-sample decay $\beta=3\times10^{-4}$, implemented as the exact closed-form composition of sequential updates within each rollout batch.
It therefore emphasizes recent targets, including their order within a collect.

\paragraph{Estimator used in each result.}
The results in \cref{tab:intervention-isolation,tab:worst-k,tab:per-side,tab:v1-mwp,tab:v1-sensitivity} use Welford PopArt; the EMA-PopArt runs are reported in \cref{app:revision-sac,app:revision-normalizers,app:revision-actor}.
\Cref{cor:fg-popart} depends only on the per-minibatch rescaling $\sigma_i^{-2}$, so it holds for either estimator whenever that rescaling is exact: for an unclipped squared loss, or when the raw and normalized clipping thresholds and active branches coincide.
The two estimators weight targets differently from the first collect on.
With $10{,}000$ targets per task, the EMA retains $(1-\beta)^{10{,}000}\approx0.050$ of its prior state and weights the new targets by recency, whereas Welford averages the first batch uniformly.
The two estimators therefore follow different training trajectories, and we report their cohorts separately.

\FloatBarrier
\section{Theoretical Details}\label{app:theory}
\subsection{TOPPO pseudocode}\label{app:toppo-algorithm}

\Cref{alg:toppo} expands the compact \cref{alg:toppo-compact} (\cref{sec:method:framework}) with the bookkeeping that the main-text version folds into surrounding loop headers, mirroring the released implementation.
Relative to our vanilla MT-PPO baseline, TOPPO adds critic LayerNorm, PopArt, and the per-side gradient combiners FG-c / PCGrad-a. Stratified minibatching and per-task advantage normalization are already present in that baseline. Per-task backward passes supply the combiner inputs.
PPO's GAE, ratio clipping, advantage recomputation, global gradient-norm clipping, and Adam update are retained. With PopArt enabled, value clipping operates in normalized value coordinates.

\begin{algorithm}[!htbp]
\caption{TOPPO: a PPO collect-and-update phase.}
\label{alg:toppo}
\begin{algorithmic}[1]
\REQUIRE Tasks $\{\mathcal{M}_i\}_{i=1}^K$, shared params $\btheta = (\btheta^a, \btheta^c)$ with \textbf{[LN-c]}, task-local action log-stddev and PopArt affine $(W_i,b_i)$, PopArt statistics $(\mu_i,\sigma_i)$, repeats $R$, minibatch size $b$, clip $\epsilon$, max gradient norm $\tau$
\STATE Rollout $\mathcal{B} = \bigcup_i \mathcal{B}_i$ under $\pi_{\btheta^a_{\text{old}}}$ (uniform per task); calculate per-task GAE $\hat A_t^{(i)}$ and value targets $\hat V_t^{(i),\text{targ}}$ on raw returns
\STATE \textbf{[PopArt]} update $(\mu_i, \sigma_i)$ via Welford on $\{\hat V_t^{(i),\text{targ}}\}$ (\cref{eq:welford}) and PopArt-renormalize the per-task affine head so raw $V$ is preserved \hfill \emph{(\cref{sec:method:popart})}
\FOR{$r = 1$ \TO $R$}
  \IF{$r > 1$}
    \STATE recompute $\hat A_t^{(i)}, \hat V_t^{(i),\text{targ}}$ on $\mathcal{B}$ under current $V_{\btheta^c}$
  \ENDIF
  \FOR{each task-stratified minibatch $\mathcal{B}^{\text{mb}} \subset \mathcal{B}$ of size $b$}
    \STATE normalize advantages within each task slice of $\mathcal{B}^{\text{mb}}$; map critic targets to PopArt space $(\hat V_t^{(i),\text{targ}} - \mu_i)/\sigma_i$ as in \cref{sec:method:popart}
    \FOR{$i = 1$ \TO $K$ such that $\mathcal{B}^{\text{mb}}_i \neq \emptyset$}
      \STATE $L_i \gets L_i^a(\btheta^a) + c_v L_i^{c,\mathrm{PopArt}}(\btheta^c) - c_e H_i(\btheta^a)$ on $\mathcal{B}^{\text{mb}}_i$ \hfill \emph{(\cref{eq:ppo})}
      \STATE backward $\bg_i \gets \nabla_{\btheta} L_i$ and cache flat actor and PopArt critic slabs $(\bg_i^a, \bg_i^{c,\mathrm{PopArt}})$
    \ENDFOR
    \STATE \textbf{[FG-c ($\alpha{=}1$)]} solve $\bG^c \bw = \bw^{-1}$ via \cref{alg:ynewton}, then $\bd^c \gets \sum_i w_i \bg_i^{c,\mathrm{PopArt}}$ \hfill \emph{(\cref{sec:method:fgc})}
    \STATE \textbf{[PCGrad-a]} $\bd^a \gets \mathrm{PCGrad}(\bg_1^a, \dots, \bg_K^a)$ \hfill \emph{actor-side surgery}
    \STATE reassemble $(\bd^a,\bd^c)$ with averaged task-local gradients; clip the full vector to norm $\tau$ and take an Adam step on all trainable parameters
  \ENDFOR
\ENDFOR
\end{algorithmic}
\end{algorithm}

\subsection{\texorpdfstring{$\alpha$-fairness family and KKT derivation}{alpha-fairness family and KKT derivation}}\label{app:alpha-fairness}

This appendix expands \cref{sec:method:fgc} with three pieces deferred from the main text: (i) the interpretation of $U_\alpha$ in \cref{eq:alpha-utility} as a one-parameter fairness family, (ii) the precise meaning of ``per-task contribution $x_i$'' in \cref{eq:fg-contribution} and the canonical scaling that pins down $\bw$ in the FairGrad program, and (iii) the derivation, via a strictly convex unconstrained program, of the KKT condition $\bG^c \bw = \bw^{-1/\alpha}$ stated in \cref{sec:method:fgc}.

\paragraph{Family interpretation.}
$U_\alpha$~\citep{mo2000fairness,kelly1997charging} is strictly concave on $\R_{>0}$ for every $\alpha > 0$, and as $\alpha$ varies over $(0, \infty)$ it interpolates classical social-welfare rules: utilitarian sum as $\alpha \to 0^+$ (each unit of $x_i$ has equal marginal value), Nash bargaining at $\alpha = 1$~\citep{nash1950bargaining} (the log-utility / geometric-mean solution), and Rawlsian max-min as $\alpha \to \infty$ (the smallest $x_i$ dominates). Larger $\alpha$ thus up-weights small-share agents --- equivalently, in our setting, gives more aggregator budget to tasks whose linearized improvement along the combined critic gradient is currently small.

\paragraph{Per-task contribution and the canonical scaling.}
Write $\bd := \sum_i w_i \bg_i^c$ for the combined critic gradient. Task $i$'s linearized improvement is the inner product
\begin{equation}\label{eq:fg-contribution}
x_i \;:=\; \bg_i^{c\,\top} \bd \;=\; (\bG^c \bw)_i,
\end{equation}
i.e.\ how much task $i$'s gradient agrees with the combined direction; $x_i > 0$ is the linearized loss decrease task $i$ obtains under a unit step along $\bd$. The bare welfare $\sum_i U_\alpha((\bG^c \bw)_i)$ has no interior maximizer along positive rays $\bw \mapsto t\bw$ (for $0<\alpha \leq 1$ it grows without bound as $t \to \infty$; for $\alpha > 1$ it approaches its supremum $0$ only in that limit), so a normalization is needed to pin down a canonical $\bw$. Geometrically, $\bw^\top \bG^c \bw = \|\bd\|^2$, so any normalization of $\bw$ corresponds to a choice of $\|\bd\|$. Throughout the paper we adopt the \emph{canonical scaling} (the standard FairGrad scaling) that yields the normalized KKT $\bG^c \bw = \bw^{-1/\alpha}$ used in \cref{sec:method:fgc,thm:scale-invariance,cor:fg-popart}; for $\alpha = 1$, this canonical $\bw$ satisfies $\|\bd\|^2 = \bw^\top \bG^c \bw = K$ rather than~$1$ (\cref{thm:scale-invariance}).

\paragraph{Derivation of the FairGrad KKT condition.}
For $\alpha = 1$ (the headline critic setting), the canonical-scaling solution is the unique minimizer of the strictly convex unconstrained program
\begin{equation}\label{eq:fg-convex-program}
\min_{\bw \in \R^K_{>0}} \;\; \Phi_1(\bw) \;:=\; \tfrac12\, \bw^\top \bG^c \bw \;-\; \sum_{i=1}^K \log w_i.
\end{equation}
Strict convexity follows from $\nabla^2 \Phi_1(\bw) = \bG^c + \diag(w_i^{-2}) \succ 0$ on $\R^K_{>0}$ (since $\bG^c \succ 0$ under the linear-independence hypothesis of \cref{thm:scale-invariance}); existence and uniqueness then follow because $\Phi_1(\bw) \to +\infty$ both on $\partial\R^K_{>0}$ (via $-\sum_i \log w_i \to +\infty$ as any $w_i \to 0^+$) and as $\|\bw\| \to \infty$ (via $\bw^\top \bG^c \bw \geq \lambda_{\min}(\bG^c)\|\bw\|^2$). Stationarity $\nabla \Phi_1(\bw^*) = \mathbf 0$ reads $\bG^c \bw^* - (\bw^*)^{-1} = \mathbf 0$, i.e.\ the normalized KKT
\begin{equation*}
\bG^c \bw^* \;=\; (\bw^*)^{-1} \quad \text{(componentwise)},
\end{equation*}
stated in \cref{sec:method:fgc}, on which \cref{thm:scale-invariance,cor:fg-popart} build. This canonical solution coincides with the constrained $\alpha$-fairness program $\max \sum_i \log((\bG^c \bw)_i)$ subject to $\bw^\top \bG^c \bw = K$: the constraint value~$K$ is precisely what makes the stationarity condition read $w_i(\bG^c\bw)_i = 1$ directly, eliminating the rescaling step needed under the conventional unit-norm normalization $\bw^\top \bG^c \bw = 1$. The unconstrained form \cref{eq:fg-convex-program} avoids carrying a Lagrange multiplier through the derivation. For general $\alpha > 0$, the same construction with $\sum_i \log w_i$ replaced by $\frac{\alpha}{\alpha-1}\sum_i w_i^{1-1/\alpha}$ ($\alpha\neq 1$) gives the KKT $\bG^c \bw = \bw^{-1/\alpha}$ via the same stationarity calculation; the boundary argument is replaced by gradient blow-up $\partial \Phi_\alpha/\partial w_i = (\bG^c \bw)_i - w_i^{-1/\alpha} \to -\infty$ as $w_i \to 0^+$, which still forces an interior minimizer. The KKT couples per-task scale (diagonal $\bG^c_{ii}$, how loud each task is on its own) with direction conflict (off-diagonals, how compatibly two tasks pull).

\subsection{Aggregator family taxonomy and cone containment}\label{app:cone-containment}

\begin{theorem}[Shared cone containment]\label{thm:reachable-set}
Let $\cone\{\bg_1, \ldots, \bg_K\} := \{\sum_i w_i \bg_i : w_i \geq 0\}$ denote the closed conic hull (non-negative combinations) of the per-task gradient set. The combined gradient produced by FairGrad (any $\alpha > 0$), PCGrad~\citep{yu2020pcgrad}, CAGrad~\citep{liu2021cagrad}, FAMO~\citep{liu2023famo}, DWA~\citep{liu2019dwa}, and uncertainty weighting~\citep{kendall2018uncertainty} lies in $\cone\{\bg_1, \ldots, \bg_K\}$.
\end{theorem}
\begin{proof}
For FairGrad/FAMO/DWA/UW, the combined gradient is $\sum_i w_i \bg_i$ with $w_i \geq 0$ by construction (FairGrad's KKT solves on $\R^K_{>0}$ for $\alpha > 0$; DWA/UW use loss-level reweighting with non-negative weights, so $\nabla \sum_i w_i L_i = \sum_i w_i \bg_i$; FAMO reweights $\log(L_i - L_i^*)$ with simplex weights $z_i$ and a lower bound $L_i^*$ satisfying $L_i > L_i^*$~\citep{liu2023famo}, so its update $\sum_i \frac{z_i}{L_i - L_i^*}\,\bg_i$ also has non-negative coefficients). For CAGrad, the combined gradient is $\sum_i (\tfrac{1}{K} + w_i^* \lambda^*)\, \bg_i$ with $w_i^* \geq 0$ on the simplex and $\lambda^* \geq 0$~\citep{liu2021cagrad}, so the coefficient on each $\bg_i$ is $\geq \tfrac{1}{K} > 0$. For PCGrad, induct over the projection schedule. Let $S_{<k}$ denote the (ordered) set of tasks already projected against in task $i$'s inner loop before step $k$. Base case ($k=0$): $\bg_i' = \bg_i$, with coefficient $1$ on $\bg_i$ and $\mu_{i,m} = 0$ for all $m \neq i$. Inductive step: at step $k$ with $\bg_i' = \bg_i + \sum_{m \in S_{<k}} \mu_{i,m} \bg_m$ ($\mu_{i,m} \geq 0$), the projection is taken against the \emph{original} $\bg_{j_k}$ (\citealp{yu2020pcgrad}, Algorithm 1); if $\langle \bg_i', \bg_{j_k}\rangle \geq 0$ no update; else the update adds $-\langle \bg_i', \bg_{j_k}\rangle/\|\bg_{j_k}\|^2 > 0$ to the (non-negative) coefficient on $\bg_{j_k}$. Invariant preserved; final $\bg_i^\text{PC}$ is a non-negative combination of $\{\bg_1, \ldots, \bg_K\}$. The combined gradient $\bg^\text{PC} = \tfrac{1}{K}\sum_i \bg_i^\text{PC}$ is then a non-negative combination of the same set, hence in $\cone\{\bg_1, \ldots, \bg_K\}$.
\end{proof}

\begin{remark}[FG at $\alpha{=}1$ recovers Nash-MTL]\label{rem:fg-nash}
At $\alpha{=}1$, the welfare is $\sum_i\log x_i$, where $x_i=\bg_i^{\!\top}\bd>0$ is the task's linearized improvement (\cref{eq:fg-contribution})~\citep{nash1950bargaining,kelly1997charging,mo2000fairness}. Under the canonical norm constraint $\|\bd\|_2^2\leq K$, maximizing this welfare over feasible directions gives the Nash bargaining direction used by Nash-MTL~\citep{navon2022nashmtl}. The norm constraint is necessary: the unconstrained log-welfare is unbounded under positive scaling. At canonical scaling, both formulations yield $w_i(\bG\bw)_i=1$, the identity underlying \cref{thm:scale-invariance}; their numerical solvers and practical stopping criteria differ.
\end{remark}

\paragraph{Implication.} \Cref{thm:reachable-set} reframes the gradient-aggregation design space: methods differ in which point inside the same shared cone they select, not in the cone itself. FAMO/DWA/UW pick a Gram-blind point; PCGrad picks a point set by random-order projection, which is triggered by conflicting entries ($\bG_{ij}<0$) and scaled by their magnitudes; CAGrad picks the solution of a constrained convex program on $\bG$; FairGrad picks the $\alpha$-fairness KKT solution on the full Gram. Among these, only $\alpha{=}1$ FG carries the closed-form $\sqrt{K}$-norm identity and scale-invariance of \cref{thm:scale-invariance}.

\subsection{Proof of \texorpdfstring{\cref*{thm:scale-invariance}}{Theorem 1} (fixed norm and scale invariance)}\label{app:scale-invariance-proof}
We write $\bw^*$ for the theorem's $\bw$ throughout, reserving $\bw$ as the optimization variable. Thus $\bw^* \in \R^K_{>0}$ denotes the unique positive solution of the $\alpha=1$ KKT $\bG\bw = \bw^{-1}$. Fix any $\bc \in \R^K_{>0}$, set $\tilde\bg_i := c_i \bg_i$ and $D := \diag(\bc)$, and write $\tilde\bG := D \bG D$ for the rescaled Gram. Steps~1--3 prove the theorem; Step~4 shows that exact scale invariance fails for every $\alpha\neq1$, as stated in \cref{sec:method:fgc}.

\smallskip
\emph{Step 1 (existence and uniqueness).}
The KKT $\bG\bw = \bw^{-1}$ is the $\alpha=1$ specialization of the strictly convex unconstrained program~\eqref{eq:fg-convex-program}, which admits a unique positive minimizer by the argument given there. Since $D$ is invertible and $\bG \succ 0$, congruence yields $\tilde\bG = D \bG D \succ 0$; applying the same argument to~\eqref{eq:fg-convex-program} with $\bG$ replaced by $\tilde\bG$ gives a unique positive solution of the rescaled KKT $\tilde\bG \tilde\bw = \tilde\bw^{-1}$.

\smallskip
\emph{Step 2 (norm identity $\|\sum_i w_i^* \bg_i\|^2 = K$).}
Componentwise, $(\bG\bw^*)_i = 1/w_i^*$, hence $w_i^* (\bG\bw^*)_i = 1$ for every $i$. Summing,
\[
(\bw^*)^\top \bG\bw^* \;=\; \sum_{i=1}^K w_i^* (\bG\bw^*)_i \;=\; K.
\]
Expanding the quadratic form via $\bG_{ij} = \langle \bg_i, \bg_j\rangle$ gives
\[
\Big\| \sum_i w_i^* \bg_i \Big\|^2 \;=\; \sum_{i,j} w_i^* w_j^* \langle \bg_i, \bg_j\rangle \;=\; (\bw^*)^\top \bG\bw^* \;=\; K.
\]

\smallskip
\emph{Step 3 (scale invariance at $\alpha = 1$).}
Define $\tilde\bw := \bw^* \oslash \bc$, i.e.\ $\tilde w_i = w_i^*/c_i$. We verify that $\tilde\bw$ solves the rescaled KKT and preserves the combined update.

\textbf{(3a) Rescaled KKT.}
Using $\tilde\bG_{ij} = c_i c_j \bG_{ij}$ and $c_j \tilde w_j = w_j^*$,
\[
(\tilde\bG\tilde\bw)_i
\;=\; c_i \sum_{j} \bG_{ij}\,(c_j \tilde w_j)
\;=\; c_i \sum_{j} \bG_{ij}\, w_j^*
\;=\; c_i\, (\bG\bw^*)_i
\;=\; \frac{c_i}{w_i^*}
\;=\; \frac{1}{\tilde w_i}.
\]
Hence $\tilde\bG\tilde\bw = \tilde\bw^{-1}$, and by Step~1's uniqueness $\tilde\bw$ is the rescaled FairGrad output.

\textbf{(3b) Combined-update preservation.}
Direct substitution gives
\[
\sum_{i} \tilde w_i \tilde\bg_i \;=\; \sum_{i} \frac{w_i^*}{c_i}\, (c_i \bg_i) \;=\; \sum_{i} w_i^* \bg_i.
\]
(In fact, combined-update preservation already pins down $\tilde\bw$: any $\tilde\bw \in \R^K_{>0}$ with $\sum_i \tilde w_i \tilde\bg_i = \sum_i w_i^* \bg_i$ satisfies $\sum_i (c_i \tilde w_i - w_i^*)\bg_i = 0$, and linear independence of $\{\bg_i\}$ forces $\tilde w_i = w_i^*/c_i$.)

\smallskip
\emph{Step 4 (necessity of $\alpha = 1$).}
Fix $0 < \alpha \neq 1$ and suppose, for contradiction, that the FairGrad output is invariant under all positive task-wise rescalings. The combined-update preservation argument in Step~3b uses only linear independence and is independent of $\alpha$, so by hypothesis the rescaled solution must again be $\tilde w_i = w_i^*/c_i$, where now $\bw^*$ solves the original $\alpha$-fairness KKT $(\bG\bw^*)_i = (w_i^*)^{-1/\alpha}$. The same factoring as in Step~3a then yields
\[
(\tilde\bG\tilde\bw)_i
\;=\; c_i\, (\bG\bw^*)_i
\;=\; c_i\, (w_i^*)^{-1/\alpha}
\;=\; c_i\, (c_i\, \tilde w_i)^{-1/\alpha}
\;=\; c_i^{\,1 - 1/\alpha}\, \tilde w_i^{-1/\alpha}.
\]
The rescaled $\alpha$-fairness KKT requires the right-hand side to equal $\tilde w_i^{-1/\alpha}$, so $c_i^{\,1 - 1/\alpha} = 1$ must hold for every $\bc \in \R^K_{>0}$, which forces $1 - 1/\alpha = 0$, i.e.\ $\alpha = 1$. Hence $\alpha = 1$ is the unique exponent in the family for which the FairGrad output is invariant under arbitrary positive task-wise rescalings.

\begin{remark}[Value clipping]
PPO value clipping is applied before $\bG^c$ is formed, so clipped samples contribute zero to the per-task critic gradients, and \cref{thm:scale-invariance} applies to the Gram matrix actually passed to the FG-c solver.
With PopArt, clipping acts in normalized coordinates, and a normalized threshold $\epsilon$ corresponds to a raw threshold $\sigma_i\epsilon$.
\Cref{eq:popart-rescale}, and hence \cref{cor:fg-popart}, therefore holds exactly for the clipped loss when the raw reference loss uses the per-task threshold $\sigma_i\epsilon$, because both losses then select the same branch on every sample.
\end{remark}

\subsection{\texorpdfstring{$y$-Newton solver derivation and benchmark protocol}{y-Newton solver derivation and benchmark protocol}}\label{app:solver}

This appendix expands \cref{sec:method:fgc} with derivations.

\paragraph{Setup and Newton step derivation.}
Reparameterize $\bw = e^{\bw_y}$ (positivity automatic; search variable on $\R^K$).
The $y$-space residual and Jacobian are
\begin{align}
F(\bw_y) &= \bG^c \exp(\bw_y) - \exp(-\bw_y/\alpha), \label{eq:y-residual}\\
J(\bw_y) &= \bG^c \diag(\exp(\bw_y)) + \tfrac{1}{\alpha}\diag(\exp(-\bw_y/\alpha)). \label{eq:y-jacobian}
\end{align}
Tier~1 declares convergence when $\|F(\bw_y)\|\leq\eta=10^{-2}$, with Armijo backtracking applied to $\|F\|$ itself rather than its square.

The componentwise derivative
\begin{equation}
\frac{\partial F_i}{\partial w_{y,j}} \;=\; \bG^c_{ij}\, \exp(w_{y,j}) \;+\; \frac{1}{\alpha}\,\delta_{ij}\, \exp(-w_{y,i}/\alpha)\,,
\end{equation}
assembles to~\eqref{eq:y-jacobian}; in code, $J$ is built by broadcasting $\exp(\bw_y)$ along the columns of $\bG^c$ plus a diagonal increment, avoiding explicit \texttt{diag} materialization. The full procedure --- diagonal warm-start, damped $y$-Newton with Armijo line search, and the two fallback tiers --- is given in \cref{alg:ynewton}.

\paragraph{Exact identities versus finite-tolerance solves.}
For returned positive weights $\widehat\bw$, define the realized combined gradient $\widehat\bd:=\sum_i\widehat w_i\bg_i^c$ and the KKT residual $\boldsymbol r=\bG^c\widehat\bw-\widehat\bw^{-1}$ at $\alpha=1$.
They satisfy
\begin{equation}\label{eq:fg-numerical-energy}
\|\widehat\bd\|_2^2
= K+\widehat\bw^{\!\top}\boldsymbol r,
\qquad
\left|\|\widehat\bd\|_2^2-K\right|
\leq \|\widehat\bw\|_2\,\|\boldsymbol r\|_2.
\end{equation}
Under the production tolerance $\|\boldsymbol r\|_2\leq\eta=10^{-2}$, the fixed-norm identity of \cref{thm:scale-invariance} therefore holds up to $\|\widehat\bw\|_2\,\eta$.
Scale invariance is exact for the KKT solution; for the numerical solver it holds up to the absolute stopping tolerance, the Gram-diagonal floor, the $y$-cap, and the fallback tiers (\cref{app:ppo-adaptations}).
Both identities concern the combined gradient before the global clip and Adam update.

\paragraph{Closed-form warm-start.}
The solver initializes from the exact solution of the diagonal case.
\begin{proposition}[Diagonal closed form]\label{prop:warmstart}
If $\bG^c$ is diagonal with strictly positive entries, the unique positive solution of the FG-KKT is $w^{\text{diag}}_i = G_{ii}^{-\alpha/(\alpha+1)}$.
\end{proposition}
\begin{proof}
Diagonal $\bG^c$ decouples the KKT into $G_{ii} w_i = w_i^{-1/\alpha}$. Multiply by $w_i^{1/\alpha}$: $G_{ii}\, w_i^{1+1/\alpha} = 1$, so $w_i = G_{ii}^{-\alpha/(\alpha+1)}$. The map $w \mapsto G_{ii} w - w^{-1/\alpha}$ is strictly increasing on $\R_{>0}$ (derivative $G_{ii} + (1/\alpha) w^{-1-1/\alpha} > 0$) — solution unique.
\end{proof}

\begin{algorithm}[!htbp]
\caption{$y$-space projected damped Newton solver for the FG-KKT $\bG^c \bw = \bw^{-1/\alpha}$, with two fallback tiers. Called once per PPO minibatch (\cref{app:ppo-adaptations}).}
\label{alg:ynewton}
\begin{algorithmic}[1]
\REQUIRE Critic Gram $\bG^c \in \R^{K\times K}_{\succ 0}$ (fp64), exponent $\alpha > 0$, residual target $\eta = 10^{-2}$, floor $\varepsilon = 10^{-8}$, soft cap $y_{\max} = 50$, max outer iters $N = 50$, max Armijo backtracks $B = 30$, Armijo $c_1 = 10^{-4}$
\ENSURE Positive weights $\bw$ and logged residual/fallback status; Tier~1 accepts $\|F(\bw_y)\|\leq\eta$, while Tier~2 uses SciPy's termination flag
\STATE $\bw_y \gets -\tfrac{\alpha}{\alpha+1} \log \max(\diag(\bG^c),\, \varepsilon)$
  \COMMENT{diagonal warm-start (\cref{prop:warmstart}); $\varepsilon$ guards $\log 0$}
\STATE $\bw_y \gets \mathrm{clip}(\bw_y,\, -y_{\max},\, y_{\max})$ \COMMENT{soft cap}
\FOR{$k = 1$ \TO $N$}
  \STATE $F \gets \bG^c \exp(\bw_y) - \exp(-\bw_y/\alpha)$
    \COMMENT{eq.~\eqref{eq:y-residual}}
  \IF{$\|F\| \leq \eta$}
    \RETURN $\max(\exp(\bw_y),\,\varepsilon)$
  \ENDIF
  \STATE $J \gets \bG^c \cdot \diag(\exp(\bw_y)) + \tfrac{1}{\alpha}\,\diag(\exp(-\bw_y/\alpha))$
    \COMMENT{eq.~\eqref{eq:y-jacobian}}
  \STATE solve $J\,\Delta\bw_y = -F$ in fp64
  \STATE $t \gets 1$;\ \ \emph{accepted} $\gets$ \texttt{false}
  \FOR{$j = 1$ \TO $B$}
    \STATE $\bw_y^{\text{trial}} \gets \min(\bw_y + t\,\Delta\bw_y,\, y_{\max})$
    \STATE $F^{\text{trial}} \gets \bG^c \exp(\bw_y^{\text{trial}}) - \exp(-\bw_y^{\text{trial}}/\alpha)$
    \IF{$\|F^{\text{trial}}\| < (1 - c_1 t)\,\|F\|$}
      \STATE $\bw_y \gets \bw_y^{\text{trial}}$;\ \ \emph{accepted} $\gets$ \texttt{true};\ \ \textbf{break out of inner loop}
    \ENDIF
    \STATE $t \gets t/2$
  \ENDFOR
  \IF{\emph{not accepted}}
    \STATE \textbf{break out of outer loop} \COMMENT{Armijo exhausted; pass current $\bw_y$ to Tier 2}
  \ENDIF
\ENDFOR
\STATE \emph{Tier 2: Newton did not reach $\eta$ $\to$ SciPy least-squares warm-started from $\exp(\bw_y)$}
\STATE $\bw \gets \texttt{least\_squares}\big(x\mapsto\bG^c x-x^{-1/\alpha},\; x_0{>}\varepsilon\text{ from }\exp(\bw_y),\; \text{bounds}{=}(\varepsilon,\infty)\big)$
\IF{\texttt{res.success} and $\bw$ is finite}
  \RETURN $\bw$
\ENDIF
\STATE \emph{Tier 3: SciPy also failed $\to$ uniform fallback, flag failure to online diagnostics}
\RETURN $\mathbf{1}_K / K$
\end{algorithmic}
\end{algorithm}

\paragraph{Why the warm-start matters.}
\begin{itemize}
\item \Cref{prop:warmstart} is exact for diagonal positive Grams before numerical floors are applied. For non-diagonal Grams it is an initialization heuristic; strong cross-coupling can prevent fast convergence, as the random $10^{8}\times$-spread stress case (row~3 of \cref{tab:solver}) shows.
\item On small-norm rows, the unguarded diagonal solution grows as $G_{ii}^{-\alpha/(\alpha+1)}$. The production initialization instead uses $\max(G_{ii},10^{-8})$, so at $\alpha=1$ it starts no higher than $10^4$, followed by the $y$-cap.
\item When Newton fails its stopping test, its last finite iterate initializes SciPy. The logs distinguish \texttt{used\_scipy\_fallback} from a final uniform \texttt{solver\_fallback}. On the $1{,}984$ real early-training Grams below, Newton terminates in a median of $2$ iterations (maximum $19$).
\end{itemize}

\paragraph{Newton step in the small-norm regime.}
\begin{itemize}
\item Row $i$ with $\bG^c_{i,:} \to 0$: $F_i \to -\exp(-w_{y,i}/\alpha)$, $J_{ii} \to (1/\alpha) \exp(-w_{y,i}/\alpha)$.
\item Newton step: $\Delta w_{y,i} = -F_i / J_{ii} \to \alpha$ — exactly $\alpha$ in the row-deficient limit $\bG^c_{i,:} \to 0$.
\item Under a full Newton step in this limiting regime, a uniform cold start $w_y=-\log K$ requires approximately $\lceil(w_y^*+\log K)/\alpha\rceil$ steps to reach a target log-weight $w_y^*$, before backtracking and coupling effects.
\end{itemize}

\paragraph{fp64 throughout.}
\begin{itemize}
\item Gram computation, residual, line search, final per-task contribution sum: all fp64.
\item Squaring very small fp32 gradients can enter the subnormal range or underflow (a $10^{-22}$ component has square $10^{-44}$, below the smallest normal fp32 value). Forming the Gram in fp64 preserves a wider exponent range.
\item The Newton branch caps $w_y$ at $50$, giving $w\leq e^{50}\approx5\times10^{21}$ on that branch. SciPy's fallback is not subject to this upper cap.
\item The weighted contribution sum is accumulated in fp64 to reduce underflow, overflow, and cancellation risk, then cast back to the training dtype.
\end{itemize}

\paragraph{Real-Gram benchmark.}
The release includes \texttt{solverbench\_rfv1.csv}, containing $1{,}984$ solves on real MT50 V1 Gram matrices from seed~0 of default-actor FG-c ($\alpha=1$) with PopArt and LN-c. The capture covers the first $2$M training steps ($4$ collects $\times16$ repeats $\times31$ minibatches), with evaluation disabled, on one RTX~3090 and $16$ CPU cores. Each production Gram was also solved by a SciPy \texttt{least\_squares} reference initialized at $1/K$, with positive lower bound $10^{-12}$ and used only for timing. Both timings exclude Gram formation.
The CSV gives median solve times of $0.397$\,ms for Newton and $32.520$\,ms for SciPy, a median per-solve speedup of $82.97\times$ and an aggregate-time ratio of $87.07\times$. All $1{,}984$ Newton solves recorded convergence and all SciPy solves returned success.
End-to-end wall-clock costs, which also include fp64 Gram formation, are in \cref{app:revision-compute}.

\begin{figure}[!htbp]
\centering
\includegraphics[width=\linewidth]{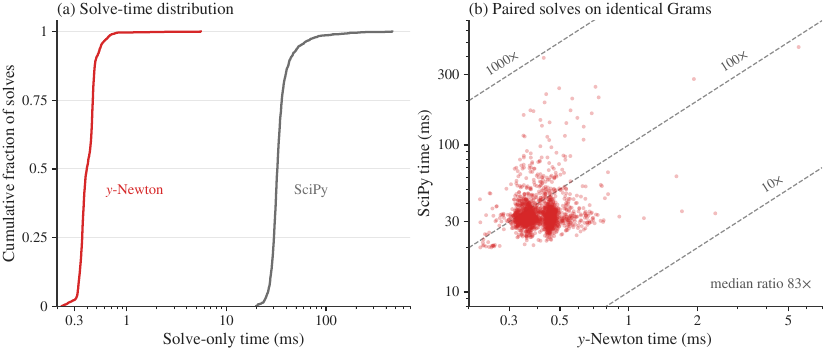}
\caption{Solve-only timing on $1{,}984$ real $K=50$, $\alpha=1$ Gram matrices from the first $2$M MT50 V1 steps of one seed. (a) Empirical cumulative distributions of solve time for $y$-Newton (red) and the SciPy \texttt{least\_squares} reference (gray). (b) Paired times, one dot per matrix; dashed lines mark $10\times$, $100\times$, and $1000\times$ ratios, and the median paired ratio is $83\times$. Timings exclude Gram formation, gradient extraction, recombination, and environment interaction; end-to-end costs are in \cref{app:revision-compute}.}
\label{fig:revision-solver}
\end{figure}

\paragraph{Real MT10 Grams with the SGD solver.}
Because the MT50 capture keeps no Gram matrices, the residual comparison in \cref{fig:solver-mt10} uses stored MT10 gradient snapshots. Four MT10 runs (default mean aggregation on both sides, or FG-a ($\alpha_a{=}2$) with FG-c ($\alpha_c{=}1$); each with PopArt, with or without LN-c; \cref{app:diagnostic-runs}) stored the full per-task critic gradients of one minibatch every five collects, from collect $5$ to collect $80$ or $95$; this gives $67$ critic Grams ($K{=}10$, formed in fp64), with a median per-task norm spread of $276\times$ (maximum $3.0{\times}10^{4}$). On each Gram we ran the SciPy reference solver and the fixed-budget SGD solver (learning rate $0.1$, momentum $0.5$, $20$ iterations) from the FairGrad release~\citep{ban2024fairgrad}, and the production $y$-Newton solver, and report the median of $5$ timed solves and the absolute KKT residual $\|\bG^c\bw-\bw^{-1}\|_2$. SGD misses the target $\eta=10^{-2}$ on all $67$ Grams (residual $\approx2.6$), while SciPy and Newton reach it on all $67$; Newton's median solve time is $0.024$\,ms versus $2.08$\,ms for SciPy (median paired ratio $126\times$). These solves were timed offline on a separate CPU host, so their absolute times differ in scale from those in \cref{fig:revision-solver} and \cref{fig:solver-main}.

\begin{figure}[!htbp]
\centering
\includegraphics[width=0.65\linewidth]{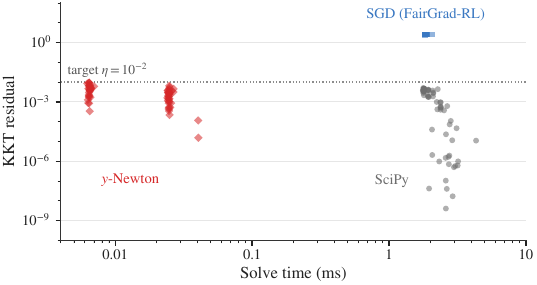}
\caption{Solve time and absolute KKT residual of the FairGrad SciPy reference, the FairGrad fixed-budget SGD solver~\citep{ban2024fairgrad}, and production $y$-Newton on $67$ real MT10 critic Gram matrices ($K{=}10$, $\alpha{=}1$), one marker per Gram and solver. The dotted line marks the target $\eta=10^{-2}$; SGD misses it on every Gram.}
\label{fig:solver-mt10}
\end{figure}

\paragraph{Synthetic benchmark protocol.}
\Cref{tab:solver} compares the four solvers on five synthetic Gram regimes that stress norm spread and off-diagonal structure.
Across these regimes, warm-started $y$-Newton is $48\times$--$1{,}442\times$ faster than SciPy, the range that \cref{sec:intro} rounds to ${\sim}50\times$--$1{,}450\times$; on $1{,}984$ real MT50 Grams the median speedup is $83\times$ (\cref{fig:revision-solver}).
\begin{itemize}
\item Synthetic Gaussian per-task gradients $\bg_i \in \R^{D}$ with $D = 1024$ and $\alpha = 1$ (matching the critic-side default of \cref{sec:method:framework}); Gram $\bG^c = [\bg_i^\top \bg_j]_{ij}$ assembled in fp64.
\item Five synthetic regimes spanning the (per-task norm spread, off-diagonal structure) plane: well-conditioned $K{=}50$ at spread $1\times$; heterogeneous $K{=}10$ with random-direction small-norm rows at $10^6\times$ spread; and three high-spread $K{=}50$ regimes at $10^8\times$ spread distinguished by structure: random Gaussian, orthogonal-to-active single small-norm row, and orthogonal-to-active $25\%$ small-norm rows.
\item Four solvers per regime: SciPy \texttt{least\_squares} (FairGrad SciPy reference, \citealp{ban2024fairgrad}; \texttt{bounds=(0,$\infty$)}, no $\max(x,\varepsilon)$ clip), SGD$+$momentum (FairGrad SGD solver, \texttt{lr=0.1}, \texttt{momentum=0.5}, \texttt{niter=20}, track-best), $y$-Newton with uniform cold start, and $y$-Newton with the diagonal warm-start of \cref{prop:warmstart} (our production solver).
\item Each (solver $\times$ regime) cell timed $7\times$ with warm timing (first run discarded); we report the median.
\item Speedup is the ratio of SciPy to warm Newton median wall times; residuals are reported alongside.
\item The release also includes a broader $85$-row synthetic sweep (speedups $9.2\times$ to $6{,}914\times$, median $74.4\times$); the real-Gram benchmark above gives the practical solver cost.
\item The release includes the table's generator, \texttt{selected\_regimes\_a1.py}.
\end{itemize}

\begin{table}[!htbp]
\centering
\caption{\textbf{Synthetic FairGrad solver benchmark ($\alpha=1$).} Median solve time and absolute KKT residual on five synthetic Gram regimes for the SciPy and fixed-budget SGD solvers of FairGrad~\citep{ban2024fairgrad} and for $y$-Newton ($y$-N) with cold and diagonal warm starts. In the regime labels, rand denotes random Gaussian gradient directions and $\perp$ denotes small-norm rows orthogonal to the active tasks (one row, or $25\%$ of rows). The target is $\eta=10^{-2}$; $^{\times}$: the SGD solver does not reach it. In row~3 the warm-started Newton solve stops at $2.0{\cdot}10^{-2}$, above the target. Speedup is SciPy over warm $y$-N time; bold marks the two largest speedups. \Cref{fig:revision-solver} reports timings on real MT50 Grams.}
\label{tab:solver}
\scriptsize
\setlength{\tabcolsep}{3pt}
\resizebox{\textwidth}{!}{\begin{tabular}{l r r r r | r r r r | c}
\toprule
& \multicolumn{4}{c|}{Wall time per solve (ms)} & \multicolumn{4}{c|}{KKT residual at termination} & \\
Regime & SciPy (FG) & SGD (FG) & $y$-N (cold) & $y$-N (warm) & SciPy (FG) & SGD (FG) & $y$-N (cold) & $y$-N (warm) & Speedup \\
\midrule
$K{=}50$, spread $1\times$ & 14.5 & 10.6$^{\times}$ & 0.43 & 0.30 & $1.8{\cdot}10^{-13}$ & $5.0{\cdot}10^{1}$ & $3.4{\cdot}10^{-7}$ & $2.4{\cdot}10^{-5}$ & $48\times$ \\
$K{=}10$, rand, spread $10^{6}\times$ & 195 & 9.0$^{\times}$ & 0.64 & 0.14 & $7.7{\cdot}10^{-10}$ & $5.2{\cdot}10^{6}$ & $2.7{\cdot}10^{-9}$ & $4.8{\cdot}10^{-5}$ & $\mathbf{1{,}442\times}$ \\
$K{=}50$, rand, spread $10^{8}\times$ (stress test) & 733 & 10.1$^{\times}$ & 0.84 & 7.7 & $5.6{\cdot}10^{-7}$ & $5.6{\cdot}10^{1}$ & $6.2{\cdot}10^{-3}$ & $2.0{\cdot}10^{-2}$ & $95\times$ \\
$K{=}50$, $\perp$, spread $10^{8}\times$ & 288 & 10.2$^{\times}$ & 0.84 & 0.22 & $9.3{\cdot}10^{-11}$ & $5.6{\cdot}10^{1}$ & $6.2{\cdot}10^{-3}$ & $3.2{\cdot}10^{-5}$ & $\mathbf{1{,}344\times}$ \\
$K{=}50$, $\perp$ 25\%, spread $10^{8}\times$ & 27.4 & 10.2$^{\times}$ & 0.93 & 0.21 & $5.3{\cdot}10^{-3}$ & $7.8{\cdot}10^{1}$ & $7.9{\cdot}10^{-3}$ & $4.7{\cdot}10^{-5}$ & $131\times$ \\
\bottomrule
\end{tabular}}
\end{table}

\FloatBarrier
\section{Additional Experimental Results}\label{app:results}
\subsection{Detailed ablation results}\label{app:ablations}

This appendix expands the main-text ablation summary (\cref{sec:exp:ablations}) with the underlying tables and per-cell analysis.

\paragraph{Critic-side intervention cube (V2).} \Cref{tab:intervention-isolation} reports the $2{\times}2$ PopArt$\times$FG-c sub-cube of the intervention cube (LN-c on, actor$=$default); the cumulative ladder and per-axis mechanism mapping are in \cref{tab:worst-k}, \cref{fig:ladder}, and \cref{sec:exp:ablations}.
The table measures the PopArt marginal at training-outcome scale; the per-minibatch identity of \cref{thm:scale-invariance} is verified synthetically in \cref{fig:f4} (\cref{app:mech-figs}), and its training-time analog at the aggregator-output level is the (c)$\equiv$(d) coincidence in \cref{fig:f1}.
The complementary regime, in which PopArt's marginal remains large with FG-c on, is the V1 sweep (\cref{fig:revision-regimes}; \cref{app:v1-sensitivity}).

\begin{figure}[!htbp]
\centering
\includegraphics[width=\linewidth]{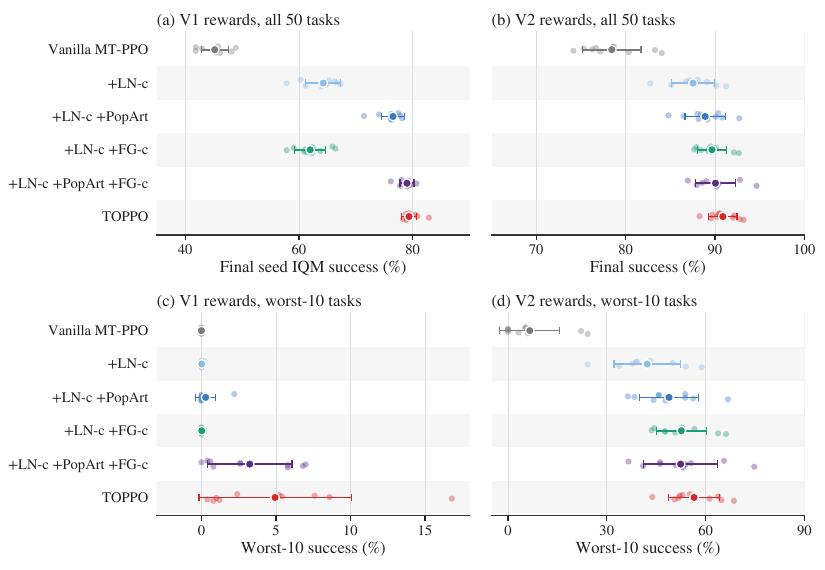}
\caption{Reward-regime dependence of the six intervention configurations, ten seeds at $100$M transitions. (a) V1 seed IQM (the Meta-World+ V1 statistic) and (b) V2 seed mean (the main-text headline statistic; V2 IQM values are in \cref{tab:v1-sensitivity}) of final success over 50 tasks; (c, d) seed mean of each seed's worst-10 success. Large markers are these centers, whiskers the untrimmed sample std, and pale points the ten seeds. Colors match \cref{fig:ladder}; +LN-c +FG-c (teal) branches off +LN-c without PopArt. Axis ranges differ between panels.}
\label{fig:revision-regimes}
\end{figure}

\begin{table}[!htbp]
\centering
\caption{\textbf{PopArt $\times$ $\alpha{=}1$ FG-c interaction on MT50.} The $2{\times}2$ sub-cube isolates whether adding FG-c absorbs PopArt's worst-10 marginal, the training-outcome signature suggested by \cref{cor:fg-popart}; the deltas below the table show this marginal collapsing from rows~1$\to$2 to rows~3$\to$4. All rows use actor$=$default with LN-c on; 10 seeds, final checkpoint, 100M env steps.}
\label{tab:intervention-isolation}
\small
\begin{tabular}{l c c r@{\,}l r@{\,}l}
\toprule
Configuration & PopArt & FG-c & \multicolumn{2}{c}{MT50 mean} & \multicolumn{2}{c}{MT50 worst-10} \\
\midrule
+LN-c (base) & -- & -- & $87.5$ & {\fontsize{5pt}{6pt}\selectfont\textcolor{gray}{$\pm\phantom{0}2.4$}} & $42.3$ & {\fontsize{5pt}{6pt}\selectfont\textcolor{gray}{$\pm10.1$}} \\
+LN-c +PopArt & \checkmark & -- & $88.9$ & {\fontsize{5pt}{6pt}\selectfont\textcolor{gray}{$\pm\phantom{0}2.3$}} & $48.9$ & {\fontsize{5pt}{6pt}\selectfont\textcolor{gray}{$\pm\phantom{0}9.0$}} \\
+LN-c +FG-c & -- & \checkmark & $89.7$ & {\fontsize{5pt}{6pt}\selectfont\textcolor{gray}{$\pm\phantom{0}1.6$}} & $52.6$ & {\fontsize{5pt}{6pt}\selectfont\textcolor{gray}{$\pm\phantom{0}7.6$}} \\
+LN-c +PopArt +FG-c & \checkmark & \checkmark & $90.1$ & {\fontsize{5pt}{6pt}\selectfont\textcolor{gray}{$\pm\phantom{0}2.3$}} & $52.4$ & {\fontsize{5pt}{6pt}\selectfont\textcolor{gray}{$\pm11.3$}} \\
\bottomrule
\end{tabular}
\\[2pt]
\footnotesize
$\Delta_{\mathrm{PopArt}}$ without FG-c (rows~1$\to$2): $+1.36$ mean and $\mathbf{+6.64}$ worst-10 points. \\
$\Delta_{\mathrm{PopArt}}$ with FG-c (rows~3$\to$4): $+0.39$ mean and $\mathbf{-0.22}$ worst-10 points.
\end{table}

\paragraph{LayerNorm placement ablation.}
On MT10 with PopArt held fixed (default actor and critic combiners, 10 seeds), symmetric LN on both actor and critic reaches $93.06 \pm 4.61\%$ final mean vs.\ $96.76 \pm 1.52\%$ for critic-only LN (+LN-c +PopArt row of \cref{tab:4metric}) --- a $3.70$-point mean drop with $\sim\!3\times$ wider seed std.
We hypothesize that actor LN interacts with PPO's policy-ratio clip: it changes the mapping from hidden features to the action mean across minibatch updates, which enlarges $|\log \pi(a|s)/\pi_{\text{old}}(a|s)|$ and triggers the ratio clip more often.
Given the MT10 result, symmetric LN was not extended to a 10-seed MT50 sweep.

\paragraph{Per-side combiner asymmetry.} \Cref{tab:per-side} compares per-side gradient combiners.
Switching the actor combiner from default to PCGrad on $+$LN-c$+$PopArt$+$FG-c raises the MT50 final mean by $0.83$ points ($90.05 \to 90.88$, \cref{tab:per-side}) and worst-10 by $4.08$ points ($52.42 \to 56.50$, \cref{tab:worst-k}), and lowers the worst-10 seed std from $11.3$ to $7.8$ points; PCGrad-a's gain is concentrated on the tail tasks.
FG-c and PCGrad-c select different points in the same shared cone $\mathrm{cone}^+\{\bg_i^c\}$ (\cref{thm:reachable-set}, \cref{app:cone-containment}), and on V2 their final success is within seed noise.
Only FG-c has the properties of \cref{thm:scale-invariance}: at an exact KKT solution the combined critic gradient has norm $\sqrt{K}$, and PopArt's per-task rescaling is absorbed (\cref{cor:fg-popart}).
These properties act on per-task return-scale differences, which are larger under V1 rewards than under V2~\citep[App.~E.3]{metaworldplus2025}.

\begin{table}[!htbp]
\centering
\caption{\textbf{Per-side combiner asymmetry on MT50 (10 seeds, 100M steps).} Holding $[400\times 3]$ MLP, PopArt, and LN-c fixed, we vary only the per-side gradient combiner; the asymmetric PCGrad-a $+$ FG-c ($\alpha{=}1$) row has the highest means, matching the diagnosis that critic gradients suffer scale spread ($497\times$ in the MT10 diagnostics, addressed by scale-invariant FG, \cref{thm:scale-invariance}) while actor gradients exhibit direction conflict at modest scale spread ($4.1\times$, addressed by PCGrad). FG-c and PCGrad-c are within seed noise (gaps $<\sigma$) and lie in the same shared cone (\cref{thm:reachable-set}); we adopt FG-c for its scale-invariance guarantee.}
\label{tab:per-side}
\small
\begin{tabular}{l l r@{\,}l r@{\,}l c}
\toprule
Actor method & Critic method & \multicolumn{2}{c}{MT50 best} & \multicolumn{2}{c}{MT50 final} & MT50 IQM \\
\midrule
default & default & $90.7$ & {\fontsize{5pt}{6pt}\selectfont\textcolor{gray}{$\pm\phantom{0}1.8$}} & $88.9$ & {\fontsize{5pt}{6pt}\selectfont\textcolor{gray}{$\pm\phantom{0}2.3$}} & $89.0$ \\
default & FG ($\alpha{=}1$) & $91.2$ & {\fontsize{5pt}{6pt}\selectfont\textcolor{gray}{$\pm\phantom{0}2.1$}} & $90.1$ & {\fontsize{5pt}{6pt}\selectfont\textcolor{gray}{$\pm\phantom{0}2.3$}} & $89.7$ \\
\textbf{PCGrad} & \textbf{FG ($\alpha{=}1$)} & $\mathbf{92.1}$ & {\fontsize{5pt}{6pt}\selectfont\textcolor{gray}{$\pm\phantom{0}1.7$}} & $\mathbf{90.9}$ & {\fontsize{5pt}{6pt}\selectfont\textcolor{gray}{$\pm\phantom{0}1.6$}} & $\mathbf{90.8}$ \\
FG ($\alpha{=}2$) & FG ($\alpha{=}1$) & $90.5$ & {\fontsize{5pt}{6pt}\selectfont\textcolor{gray}{$\pm\phantom{0}2.0$}} & $89.2$ & {\fontsize{5pt}{6pt}\selectfont\textcolor{gray}{$\pm\phantom{0}2.2$}} & $89.2$ \\
PCGrad & PCGrad & $91.9$ & {\fontsize{5pt}{6pt}\selectfont\textcolor{gray}{$\pm\phantom{0}1.5$}} & $90.5$ & {\fontsize{5pt}{6pt}\selectfont\textcolor{gray}{$\pm\phantom{0}1.4$}} & $90.4$ \\
CAGrad & CAGrad & $90.7$ & {\fontsize{5pt}{6pt}\selectfont\textcolor{gray}{$\pm\phantom{0}1.3$}} & $89.4$ & {\fontsize{5pt}{6pt}\selectfont\textcolor{gray}{$\pm\phantom{0}1.4$}} & $89.2$ \\
\bottomrule
\end{tabular}
\end{table}

\subsection{Diagnostic dump-run provenance}\label{app:diagnostic-runs}

The MT10 mechanism figures (\cref{fig:f1,fig:f3,fig:f3-early,fig:f3a,fig:f3b,fig:fae}) are derived from MT10 gradient-snapshot dump runs that are \emph{separate} from the headline 10-seed sweep used for the success-rate tables.
Each dump run instruments \texttt{MultiTaskPPO} to write per-task flat actor and critic gradients (plus PopArt $\sigma_i$, FairGrad weights when on, long-EMA per-task success rate) to disk every $5$ collects ($\sim\!0.5$M env steps, $\sim\!27$~MB of uncompressed fp32 gradients per snapshot) for offline replay.
The figures use four such dumps, one per intervention cell:
\begin{itemize}\itemsep1pt\parskip0pt
\item \textbf{default mean aggregation $+$ PopArt $+$ no LN} --- backs the no-LN columns of \cref{fig:f3,fig:f3-early}.
\item \textbf{default mean aggregation $+$ PopArt $+$ LN-c} --- backs the LN-c columns of \cref{fig:f3,fig:f3-early}, all of \cref{fig:f3a,fig:f3b} (so the $497\times$ vs.\ $4.1\times$ critic-vs-actor spread is read off the same minibatches).
\item \textbf{FG ($\alpha{=}2$ actor / $\alpha{=}1$ critic) $+$ PopArt $+$ no LN} --- backs all four panels of \cref{fig:f1} and the first two columns of \cref{fig:fae}.
\item \textbf{FG ($\alpha{=}2$ actor / $\alpha{=}1$ critic) $+$ PopArt $+$ LN-c} --- supplies the LN-c FG-trained comparison in \cref{fig:fae} and the late-training critic snapshot of \cref{fig:f4-historical}.
\end{itemize}

Cols (a, b) of \cref{fig:f1} and the no-PopArt columns (a, c) of \cref{fig:f3} rescale the dumped PopArt gradients by $\sigma_i^2$ at fixed network weights.
For unclipped squared loss this reconstructs the corresponding raw-loss gradient.
For clipped losses it corresponds to clipping thresholds matched in raw units.
The labels ``vanilla'' and ``no PopArt'' in these cells therefore denote raw-loss gradient units within PopArt-trained runs.
The $497\times$ critic spread in \cref{fig:f3b} uses this same reconstruction; actor gradients come directly from the same snapshots.
\Cref{fig:f1} cols (c, d) additionally rely on offline tight-tolerance FairGrad replay ($\eta\!=\!10^{-13}$ on the $y$-space Newton solver, vs.\ production training $\eta\!=\!10^{-2}$); these offline replays visualize the exact-solution identity.
Because the production tolerance is absolute, its relative accuracy decreases at small gradient scales (\cref{app:solver}, \cref{fig:f4}c).

\subsection{Mechanism evidence figures}\label{app:mech-figs}

On the ten MT50 V2 TOPPO seeds, the across-seed median of the per-seed temporal median task-norm spread is $645.1\times$ for the critic vs.\ $5.3\times$ for the actor (\cref{fig:headline-gradients}).
\Cref{fig:mt10-summary} collects the MT10 scalar diagnostics, and \cref{fig:f3,fig:f3a,fig:fae,fig:f1,fig:f4-historical} show the corresponding heatmaps and traces, rendered from the MT10 gradient-snapshot dump runs of \cref{app:diagnostic-runs}.

\begin{figure}[!htbp]
\centering
\includegraphics[width=\linewidth]{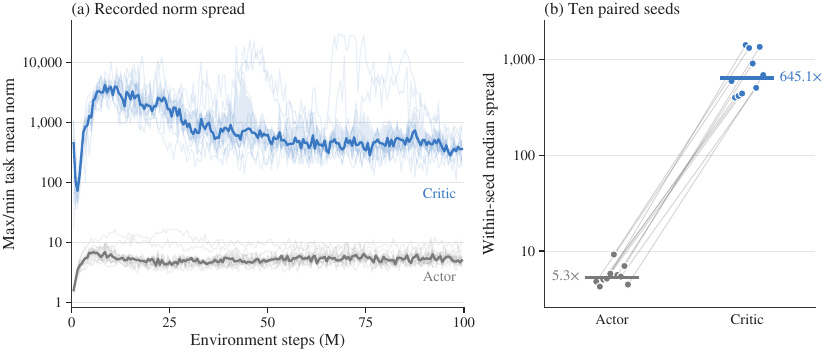}
\caption{Recorded task-gradient norm spread (max/min task mean norm) in the ten MT50 V2 TOPPO seeds. (a) Thin lines are individual seeds; the band and thick line are the across-seed interquartile range and median of per-flush ratios. Missing records are not imputed. (b) Each seed's temporal median on the $191$ shared logger steps ($4.5$--$99.5$M), with a gray line joining its actor and critic values; bars and labels are across-seed medians. All ten seeds train with PopArt.}
\label{fig:headline-gradients}
\end{figure}

\begin{figure}[!htbp]
\centering
\includegraphics[width=\linewidth]{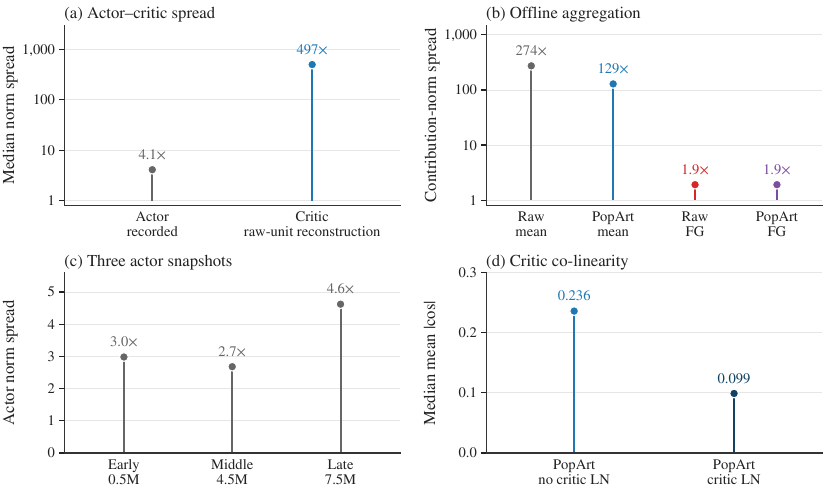}
\caption{MT10 scalar diagnostics redrawn from the recorded numerical summaries, one lollipop per recorded value, colored as in the heatmap and trace figures that follow. (a) Median actor and critic task-norm spread. (b) Median contribution-norm spread for raw or PopArt gradients under mean or FairGrad (FG) aggregation. (c) Actor norm spread at three snapshots. (d) Median mean absolute critic cosine without and with critic LN. Raw-unit values reconstruct loss units within PopArt-trained runs; FG denotes offline tight-KKT replay. Absolute cosine measures co-linearity rather than negative-conflict frequency. The corresponding heatmaps and traces follow.}
\label{fig:mt10-summary}
\end{figure}

\paragraph{Critic Gram heatmaps, intervention ablation across stages (\cref{fig:f3}).}
\begin{itemize}\itemsep1pt\parskip0pt
\item PopArt (cols a$\to$b, c$\to$d) preserves off-diagonal cosines exactly (closed-form identity from per-task $\sigma^{-2}$ rescaling) and compresses the diagonal magnitude palette in the early/late stages. The MID-row LN-c pair shows a transient inversion ($525\times \to 1428\times$) attributable to the Welford-accumulated $\sigma_i$ lagging mid-training return-distribution shifts.
\item LN (cols a$\to$c, b$\to$d) reduces mean absolute off-diagonal cosine (early-stage mean off-diag $|\cos|$ from $0.34$ no-LN to $0.20$ with LN-c on the default-mean dump of \cref{fig:f3}; over the full trace of the matched FG-trained dump in \cref{fig:fae}, the median falls from $0.24$ to $0.10$) while leaving the diagonal magnitude palette roughly intact --- a representation-level intervention, not a target-side fix.
\item Per-task gradient-norm spread $\max_i\|\bg_i^c\|/\min_i\|\bg_i^c\|$ (vanilla, $\sigma^2$-inverted) up to $\sim 10{,}000\times$ at early training, decaying to $\sim 200\times$ by mid-training and rebounding to $\sim 700\times$ by late training without LN --- still orders of magnitude worse than the actor's $2.7$--$4.6\times$ across the same training stages (\cref{fig:f3a}).
\end{itemize}

\begin{figure}[!htbp]
\centering
\includegraphics[width=\linewidth]{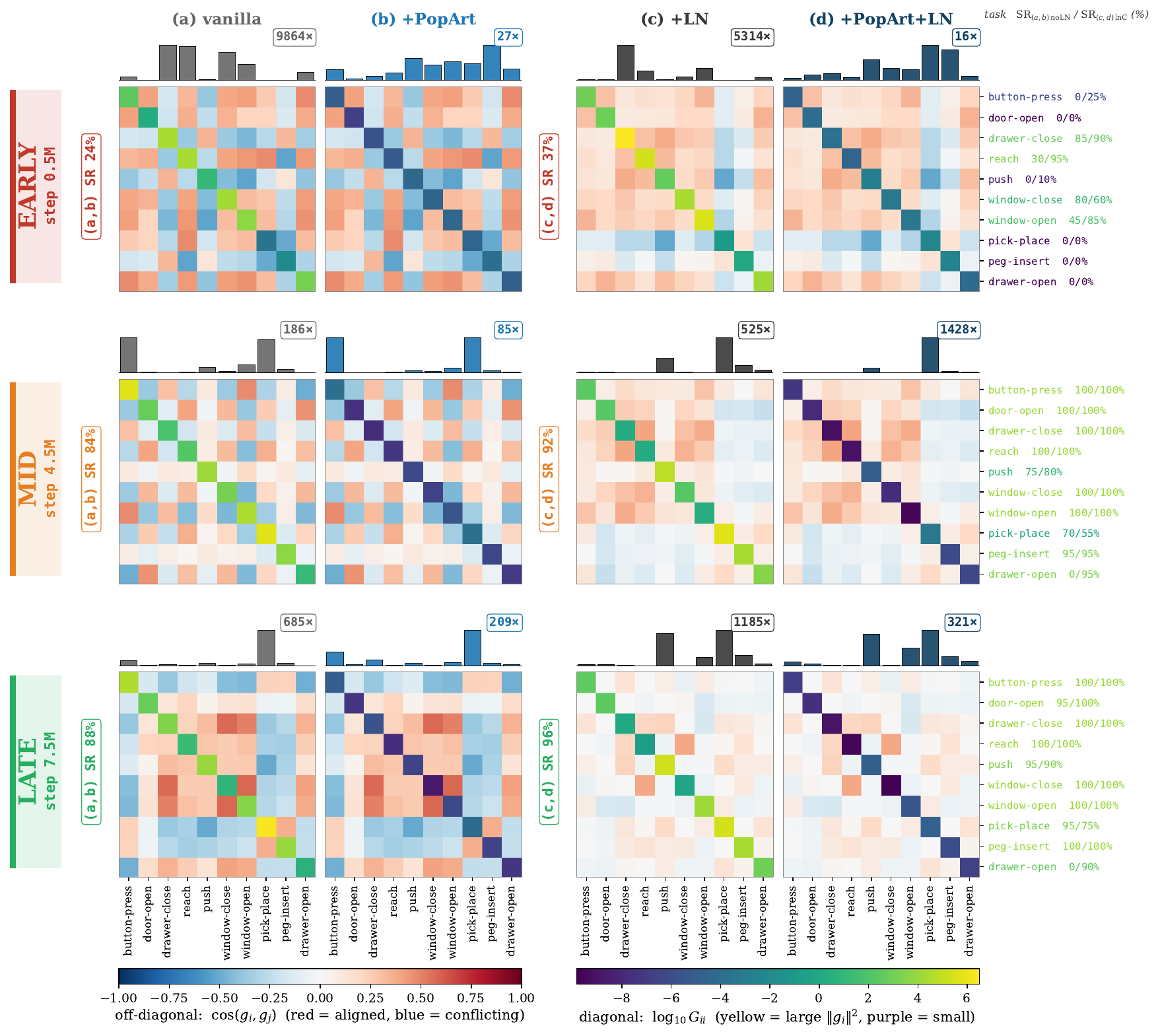}
\caption{\textbf{Critic Gram heatmap matrix: 3 stages (rows: early/mid/late) $\times$ 4 intervention conditions (cols: raw-unit reconstruction, +PopArt, +LN, +PopArt+LN).} Top sidebar = per-task mean-aggregator contribution $\|\bg_i\|/K = \sqrt{\bG_{ii}}/K$, in-bar badge = per-cell spread $\max_i\|\bg_i\|/\min_i\|\bg_i\|$; main grid off-diagonal = cosines $\bG_{ij}/\sqrt{\bG_{ii}\bG_{jj}}$; diagonal color = $\log_{10}\bG_{ii}$. PopArt (a$\to$b, c$\to$d) preserves cosines exactly and compresses the diagonal magnitude palette in early/late stages, with a transient mid-row spread inversion under LN-c ($525\times \to 1428\times$) tracked to Welford-lagged $\sigma_i$. LN (a$\to$c, b$\to$d) reduces absolute cosine similarity in the recorded representations while leaving the diagonal palette roughly intact.}
\label{fig:f3}
\end{figure}

\begin{figure}[!htbp]
\centering
\includegraphics[width=\linewidth]{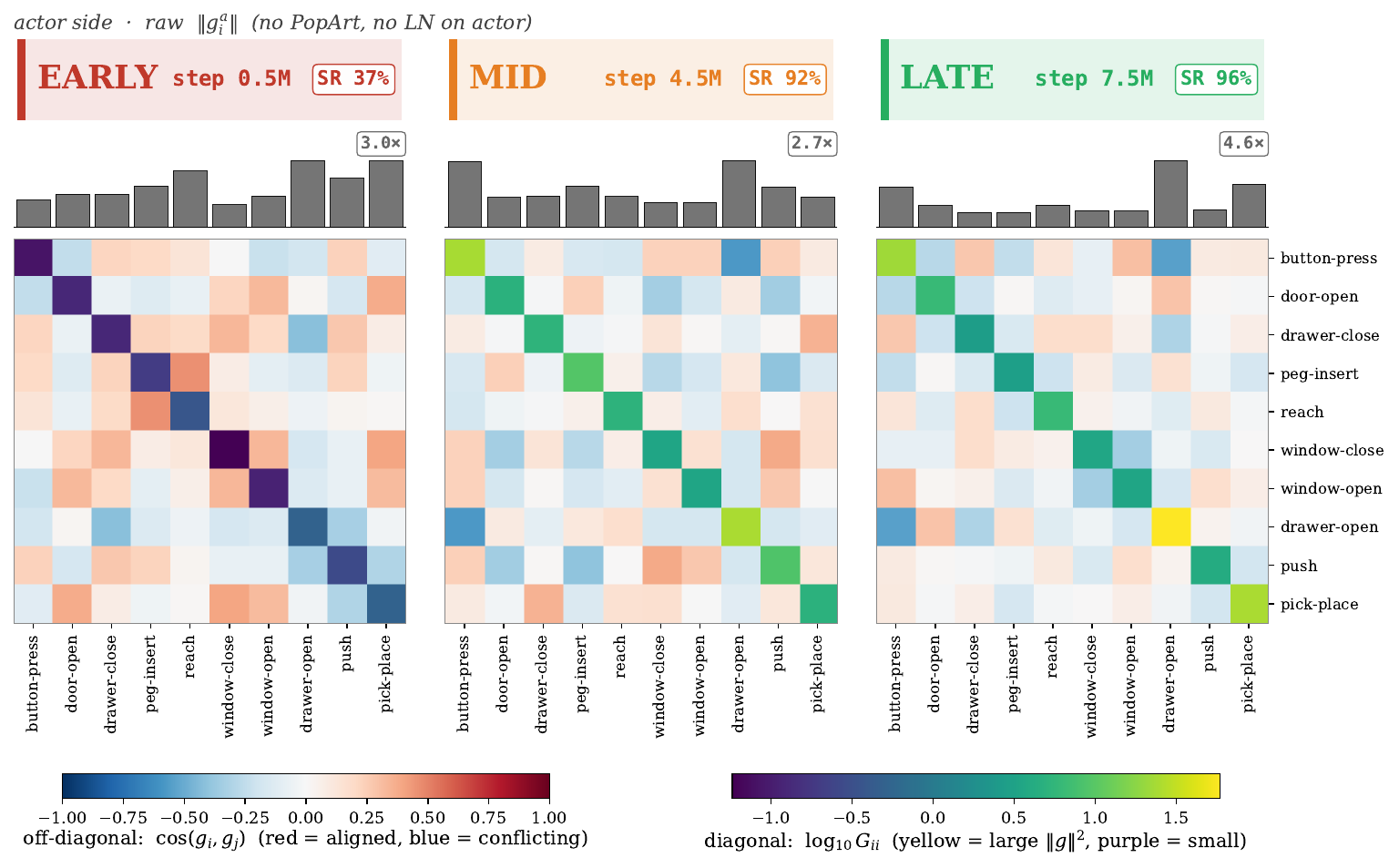}
\caption{\textbf{Actor Gram heatmaps at 3 training stages (EARLY $0.5$M / MID $4.5$M / LATE $7.5$M steps; raw $\|\bg_i^a\|$, no PopArt or LN on actor).} Top sidebar shows per-task gradient-norm bars with the in-bar badge giving the spread ratio $\max_i\|\bg_i^a\|/\min_i\|\bg_i^a\|$ ($3.0\times$, $2.7\times$, $4.6\times$); main grid plots off-diagonal cosines $\bG_{ij}/\sqrt{\bG_{ii}\bG_{jj}}$ (red aligned, blue conflicting) with the diagonal colored by $\log_{10}\bG_{ii}$ (yellow large, purple small). The mild $2.7$--$4.6\times$ actor diagonal spread contrasts sharply with the raw-loss-unit critic reconstruction's $9864\times$ (no LN) and $5314\times$ (LN-c) early spreads in \cref{fig:f3}, while off-diagonals shift from early inter-task conflict to late near-orthogonal alignment.}
\label{fig:f3a}
\end{figure}

\paragraph{LN $\times$ FG: complementary critic-Gram health metrics (\cref{fig:fae}).}
LN-c reshapes $\bG^c$ and mainly lowers cross-task cosine overlap; it does not reduce top-1 weighted contribution mass ($0.45 \to 0.57$ on this matched-seed pair).
FG-c changes only $\bw$: it leaves the cosine structure of a fixed snapshot unchanged and lowers top-1 mass ($0.57 \to 0.14$).

\begin{figure}[!htbp]
\centering
\includegraphics[width=\linewidth]{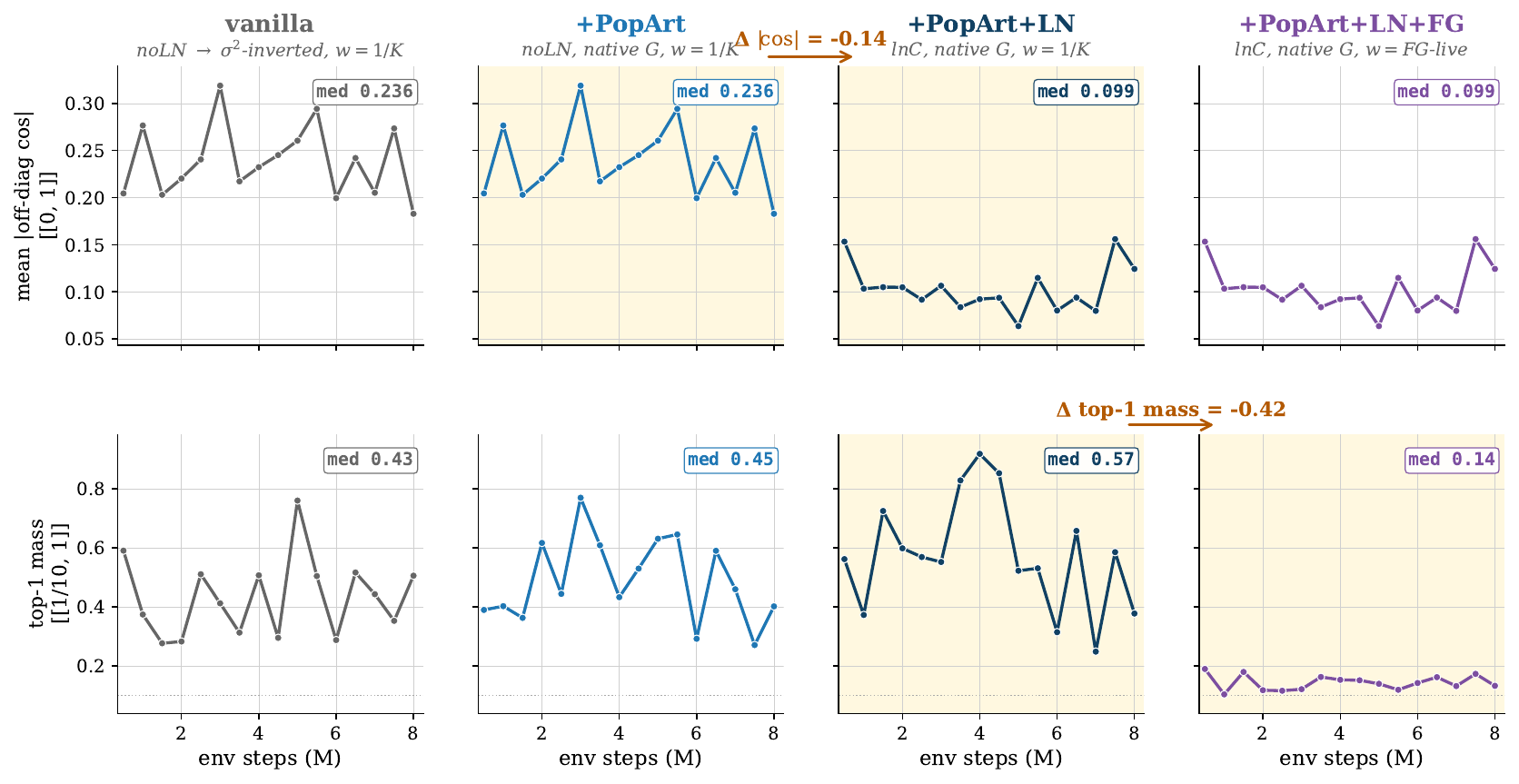}
\caption{\textbf{LN and FG act on complementary critic-Gram health metrics.} Top row: mean off-diagonal $|\cos|$; bottom row: top-1 weighted contribution mass. LN (col 2$\to$3) cuts off-diag $|\cos|$ from $0.24$ to $0.10$ without reducing top-1 mass; FG (col 3$\to$4) cuts top-1 mass from $0.57$ to $0.14$ while leaving cosine structure unchanged (FG modifies only $\bw$, not $\bG$).}
\label{fig:fae}
\end{figure}

\paragraph{FairGrad ($\alpha{=}1$) equalizes per-task contribution (\cref{fig:f1}).}
The (c)$\equiv$(d) coincidence is an offline replay illustration of \cref{cor:fg-popart}, complementary to the synthetic per-minibatch verification of \cref{thm:scale-invariance} in \cref{fig:f4}.
\Cref{app:diagnostic-runs} describes the offline tight-tolerance replay.

\begin{figure}[!htbp]
\centering
\includegraphics[width=\linewidth]{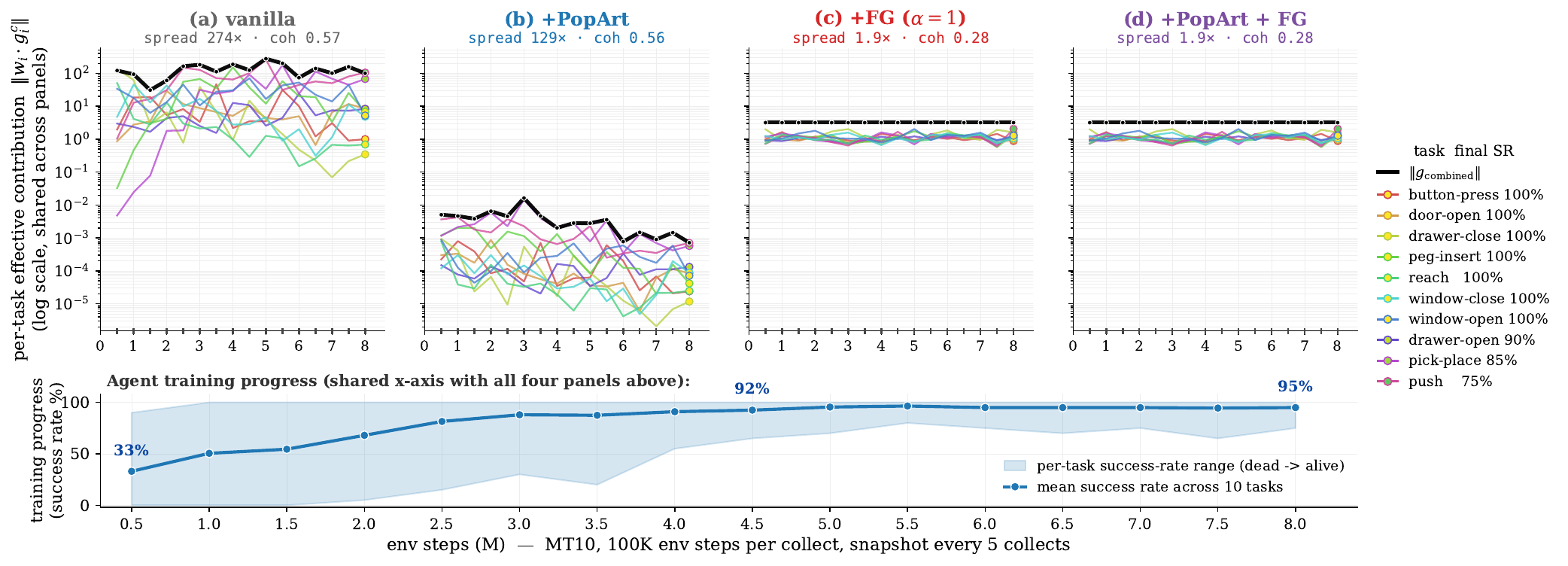}
\caption{\textbf{FairGrad ($\alpha{=}1$) equalizes per-task contribution; PopArt's task-wise rescaling effect is absorbed at the aggregator (\cref{cor:fg-popart}).} Per-task effective contribution $\|w_i \bg_i^c\|$ across training, with spreads (a) raw-loss-unit reconstruction $\sim 274\times$, (b) +PopArt $\sim 129\times$, (c) +FG ($\alpha{=}1$) on raw gradients $\sim 1.9\times$, and (d) +FG ($\alpha{=}1$) on PopArt-rescaled gradients $\sim 1.9\times$. The (c)$\equiv$(d) coincidence illustrates the per-minibatch corollary of \cref{thm:scale-invariance}.}
\label{fig:f1}
\end{figure}

\begin{figure}[!htbp]
\centering
\includegraphics[width=\linewidth]{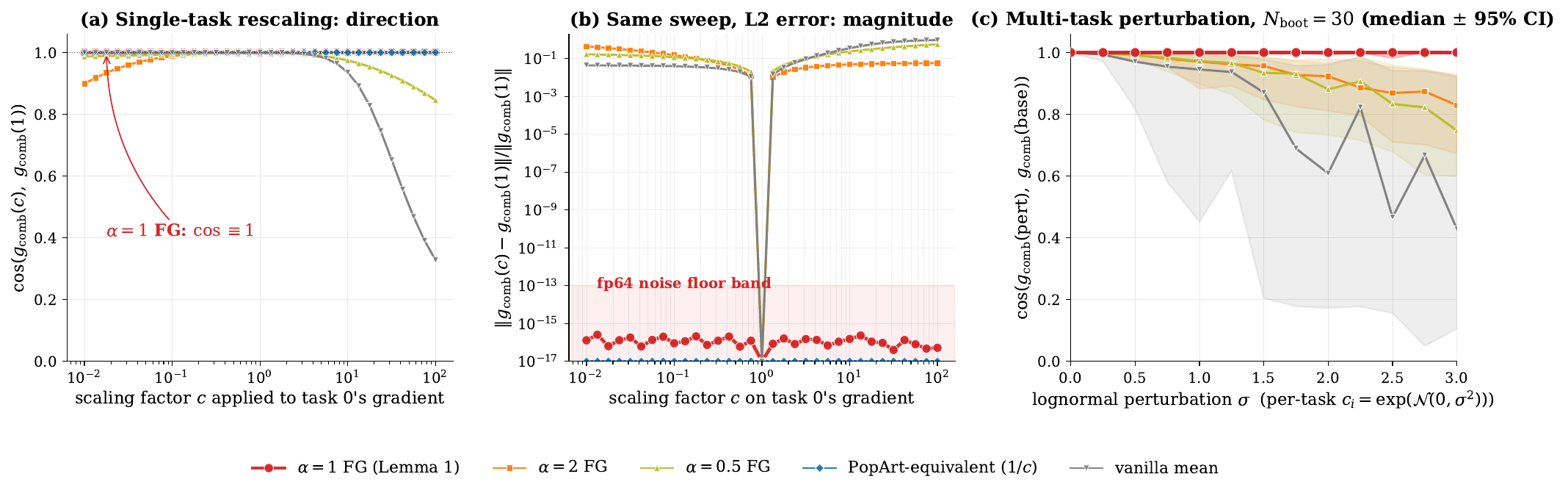}
\caption{Rescaling diagnostic on a late-training MT10 critic-gradient snapshot. The panels show single-task direction (a), relative vector change (b), and repeated multi-task perturbations (c). In (c), the line is the median over $30$ random multi-task perturbation trials and the band (labeled ``95\% CI'') spans their $2.5$--$97.5$ percentiles; ``Lemma~1'' in the legend refers to \cref{thm:scale-invariance}. \Cref{fig:f4} repeats the test on synthetic vectors.}
\label{fig:f4-historical}
\end{figure}

\paragraph{Synthetic verification and finite-tolerance behavior (\cref{fig:f4}).}
We independently solve the KKT system after each positive rescaling of a fixed set of ten linearly independent synthetic vectors.
The tight solver uses normalized coordinates $z_i=\|\bg_i\|w_i$, with residuals checked in the original KKT equations.
At $\alpha=1$, the rescaled and unperturbed aggregate vectors agree to floating-point precision.
\Cref{fig:f4-historical} shows the same rescaling test on an MT10 critic-gradient snapshot.
A separate sweep through the released production solver shows how an absolute stopping tolerance can break scale invariance for small gradients (\cref{fig:f4}c).

\begin{figure}[!htbp]
\centering
\includegraphics[width=\linewidth]{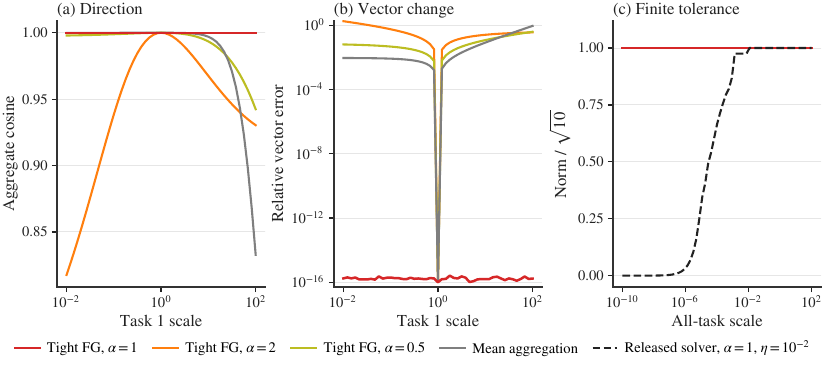}
\caption{Synthetic verification with $K=10$ independent vectors in $\R^{32}$. (a,b) Positive rescaling of one task, with a fresh tight KKT solve in normalized coordinates for each perturbation: cosine to the unperturbed aggregate and relative vector change for tight FairGrad ($\alpha=1,2,0.5$) and mean aggregation. The $\alpha=1$ aggregate agrees to floating-point precision; errors below $10^{-16}$ are displayed at that floor. (c) Global rescaling of all tasks: aggregate norm divided by the exact reference $\sqrt{10}$ for the tight $\alpha=1$ solve (solid) and the released solver at absolute tolerance $\eta=10^{-2}$ with its numerical guards (dashed), which can depart from the reference. Colors for $\alpha$ follow \cref{fig:f4-historical}. All vectors are synthetic.}
\label{fig:f4}
\end{figure}

\paragraph{Recorded production aggregate norms.}
The ten MT50 headline runs log the realized critic aggregate before global clipping.
The median absolute relative deviation of the logged mean norm from $\sqrt{50}$ is $1.15\%$, while logged minibatch extrema span $5.29$--$12.00$ (\cref{fig:production-norms}).
The production critic update therefore tracks the exact-KKT norm closely on average, and individual minibatches vary around it within the recorded extrema.

\begin{figure}[!htbp]
\centering
\includegraphics[width=\linewidth]{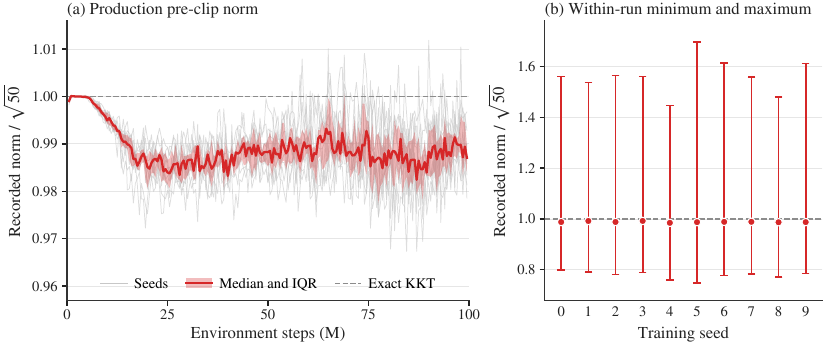}
\caption{Production critic aggregate norms in the ten MT50 V2 TOPPO seeds before global gradient clipping, divided by the exact-KKT norm $\sqrt{50}$ (dashed line at $1$). (a) Per-flush mean norms: thin lines are seeds; the band and thick line are the across-seed interquartile range and median. (b) Per seed, bars span the recorded minibatch minimum and maximum over the run, and dots mark the median flush mean. Missing logger records are not imputed.}
\label{fig:production-norms}
\end{figure}

\FloatBarrier  %

\subsection{Headline V2 learning curves}\label{app:v2-curves}

\Cref{fig:headline-curves} shows unsmoothed learning curves for the six intervention configurations, using complete recorded evaluation histories.
Each curve uses only checkpoints recorded for all ten seeds in that group, without an imputed origin or held endpoints; the one exception is the MT50 V2 TOPPO curve before $6$M, described next.
TOPPO's observed mean-curve peak is $97.92\pm2.23\%$ at $7.2$M on MT10 and $91.26\pm1.63\%$ at $96$M on MT50; these are distinct from per-seed best checkpoints.
For TOPPO MT50 V2, the common recorded grid begins at $6$M because two seeds lack earlier evaluations (one has none before $6$M, the other only at $0$M); the plotted points at $0$, $2$, and $4$M average the nine, eight, and eight seeds logged there, and all crossings below use the ten-seed grid.
On MT50, first observed crossings of the ARS final means occur at $12$M (ARS), $18$M (ARS-LN-400), $24$M (ARS-LoRA), and $28$M (ARS-LN-1024), and of the published MOORE final mean at $12$M.
Linear interpolation between checkpoints gives $11.7$M, $16.6$M, $22.1$M, $27.7$M, and $10.7$M, respectively, the basis of the approximate times quoted in \cref{sec:exp:headline}.

\begin{figure}[!htbp]
\centering
\includegraphics[width=\linewidth]{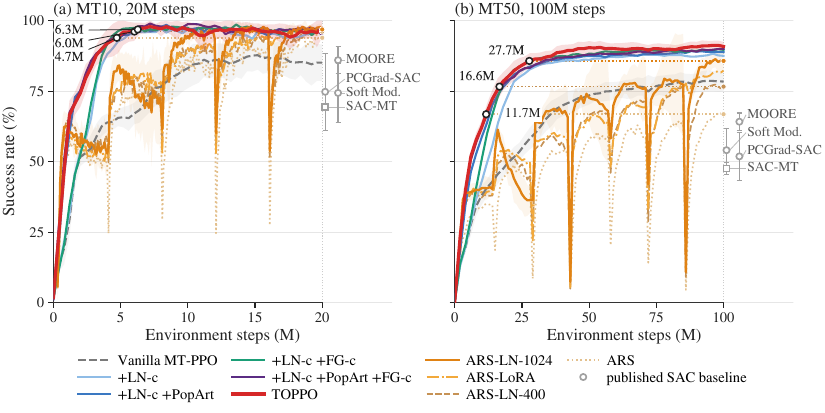}
\caption{MT10 (a) and MT50 (b) V2 learning curves of the six intervention configurations and the four ARS variants. Lines show ten-seed means on common recorded checkpoints, except that two MT50 TOPPO seeds lack some evaluation logs before 6M, so MT50 TOPPO points there average the seeds logged (nine at 0M, eight at 2M and 4M); bands show sample std for TOPPO, vanilla, and ARS-LN-1024. Hollow gray right-margin markers show published SAC-baseline final means $\pm$ std without training curves~\citep{cho2025ars}. Open circles mark where TOPPO first reaches a comparator's final mean (ARS, ARS-LN-400, and ARS-LN-1024; on MT10, TOPPO's mean never reaches ARS-LoRA), interpolated between recorded checkpoints; dotted lines lead to that mean. TOPPO has $685$K (MT10) and $717$K (MT50) parameters (vanilla $684$K/$716$K); ARS and ARS-LN-400 have $2.5$M (MT10) and $2.6$M (MT50), ARS-LoRA $3.9$M and $16.1$M, and ARS-LN-1024 $16.1$M and $16.3$M. Final PPO results are in \cref{tab:combined-headline-tail}.}
\label{fig:headline-curves}
\end{figure}

\subsection{V1 reward sensitivity}\label{app:v1-sensitivity}
Meta-World V1 rewards have substantially wider per-task scale heterogeneity than V2~\citep{metaworldplus2025}.
We reran the same intervention rungs and the PCGrad-a headline on MT50 V1 (ten seeds each, $100$M steps).
The upper panel of \cref{tab:v1-sensitivity} reports final-checkpoint IQM $\pm$ untrimmed sample std alongside V2; the lower panel transcribes the Meta-World+ SAC-family results.
\begin{table}[!htbp]
\centering
\caption{\textbf{V1/V2 reward sensitivity on MT50.} Our rows use ten seeds, 100M transitions, and final-checkpoint IQM ($25\%$ integer trimming of seed means, retaining six of ten) with untrimmed sample std. The lower panel transcribes the published Meta-World+ results~\citep{metaworldplus2025}. PopArt's IQM marginal over LN-c+FG-c is 17.01 points on V1 versus 0.32 on V2.}
\label{tab:v1-sensitivity}
\small
\begin{tabular}{l r@{\,}l r@{\,}l c}
\toprule
Configuration & \multicolumn{2}{c}{V1 IQM} & \multicolumn{2}{c}{V2 IQM} & V1 $\to$ V2 gap \\
\midrule
\multicolumn{6}{l}{\emph{Our intervention cube (10 seeds, IQM\,$\pm$\,std)}} \\
\midrule
Vanilla MT-PPO & $45.2$ & {\fontsize{5pt}{6pt}\selectfont\textcolor{gray}{$\pm\phantom{0}2.4$}} & $78.0$ & {\fontsize{5pt}{6pt}\selectfont\textcolor{gray}{$\pm\phantom{0}3.3$}} & $+32.76$ \\
+LN-c & $64.3$ & {\fontsize{5pt}{6pt}\selectfont\textcolor{gray}{$\pm\phantom{0}3.1$}} & $87.7$ & {\fontsize{5pt}{6pt}\selectfont\textcolor{gray}{$\pm\phantom{0}2.4$}} & $+23.43$ \\
+LN-c +PopArt & $76.5$ & {\fontsize{5pt}{6pt}\selectfont\textcolor{gray}{$\pm\phantom{0}2.0$}} & $89.0$ & {\fontsize{5pt}{6pt}\selectfont\textcolor{gray}{$\pm\phantom{0}2.3$}} & $+12.51$ \\
+LN-c +FG-c & $62.0$ & {\fontsize{5pt}{6pt}\selectfont\textcolor{gray}{$\pm\phantom{0}2.7$}} & $89.4$ & {\fontsize{5pt}{6pt}\selectfont\textcolor{gray}{$\pm\phantom{0}1.6$}} & $+27.39$ \\
+LN-c +PopArt +FG-c & $79.0$ & {\fontsize{5pt}{6pt}\selectfont\textcolor{gray}{$\pm\phantom{0}1.2$}} & $89.7$ & {\fontsize{5pt}{6pt}\selectfont\textcolor{gray}{$\pm\phantom{0}2.3$}} & $+10.70$ \\
\textbf{TOPPO} & $\mathbf{79.3}$ & {\fontsize{5pt}{6pt}\selectfont\textcolor{gray}{$\pm\phantom{0}1.3$}} & $\mathbf{90.8}$ & {\fontsize{5pt}{6pt}\selectfont\textcolor{gray}{$\pm\phantom{0}1.6$}} & $\mathbf{+11.49}$ \\
\midrule
\multicolumn{6}{l}{\emph{Meta-World+ MT50 SAC-family baselines (10 seeds, IQM\,$\pm$\,std)~\citep{metaworldplus2025}}} \\
\midrule
MT-MH-SAC & $31.9$ & {\fontsize{5pt}{6pt}\selectfont\textcolor{gray}{$\pm\phantom{0}2.7$}} & $54.9$ & {\fontsize{5pt}{6pt}\selectfont\textcolor{gray}{$\pm\phantom{0}1.1$}} & $+22.98$ \\
Soft Modularization & $60.6$ & {\fontsize{5pt}{6pt}\selectfont\textcolor{gray}{$\pm\phantom{0}3.9$}} & $65.8$ & {\fontsize{5pt}{6pt}\selectfont\textcolor{gray}{$\pm\phantom{0}1.9$}} & $+5.19$ \\
MOORE & $61.8$ & {\fontsize{5pt}{6pt}\selectfont\textcolor{gray}{$\pm\phantom{0}2.7$}} & $72.0$ & {\fontsize{5pt}{6pt}\selectfont\textcolor{gray}{$\pm\phantom{0}2.9$}} & $+10.22$ \\
PaCo & $18.6$ & {\fontsize{5pt}{6pt}\selectfont\textcolor{gray}{$\pm15.4$}} & $58.9$ & {\fontsize{5pt}{6pt}\selectfont\textcolor{gray}{$\pm\phantom{0}4.6$}} & $+40.30$ \\
PCGrad-SAC & $45.8$ & {\fontsize{5pt}{6pt}\selectfont\textcolor{gray}{$\pm\phantom{0}5.9$}} & $69.5$ & {\fontsize{5pt}{6pt}\selectfont\textcolor{gray}{$\pm\phantom{0}2.7$}} & $+23.65$ \\
\bottomrule
\end{tabular}
\\[2pt]
\footnotesize
$\Delta_{\mathrm{PopArt}}$ IQM on V1: $\mathbf{+12.26}$ points without FG-c (rows~2$\to$3) and $\mathbf{+17.01}$ points with FG-c (rows~4$\to$5). \\
On V2 the same two IQM deltas are $+1.35$ and $+0.32$ points, respectively.
\end{table}

PopArt's IQM marginal is larger on V1 both with and without FG-c (\cref{tab:v1-sensitivity}).
Within the V1 ladder, FG-c's marginal over LN-c+PopArt is $2.47$ points, whereas adding FG-c without PopArt reduces IQM by $2.29$ points.
\Cref{cor:fg-popart} covers a fixed positive task-wise rescaling at an exact KKT solution, under which \cref{thm:scale-invariance} keeps large FairGrad weights on small gradients from unbalancing the aggregate.
Training departs from this setting through value clipping, changing targets and statistics, optimizer history, and finite solver accuracy; these departures are the candidate sources of the V1 interaction.
\Cref{app:revision-normalizers} tests two of them directly, value clipping and the choice of normalizer.

\textbf{V1$\to$V2 gap: narrower for our PopArt-equipped rows than for three of the five SAC baselines.}
Under matched final-checkpoint IQM (\cref{tab:v1-sensitivity}), every row with PopArt on has a V1$\to$V2 gain in a $10$--$13$-point band ($12.51$ for +LN-c+PopArt, $10.70$ for +LN-c+PopArt+FG-c, $11.49$ for the headline).
The Meta-World+ SAC-family baselines span $5.19$ (Soft Modularization) to $40.30$ points (PaCo) under the same V1$\to$V2 reward switch, with three of the five methods gaining at least $20$ points (MT-MH-SAC $22.98$, PaCo $40.30$, PCGrad-SAC $23.65$).
This narrower gap is associated with PopArt in the reported ladder: our \emph{no}-PopArt rows (vanilla $32.76$, +LN-c $23.43$, +LN-c+FG-c $27.39$ points) sit in a $23$--$33$-point range, inside the $23$--$40$-point range of the three SAC-family baselines that gain at least $20$ points.
\Cref{fig:train-curve-v1} shows the corresponding V1 learning curves on the same axes for each intervention rung.

\begin{figure}[!htbp]
\centering
\includegraphics[width=\linewidth]{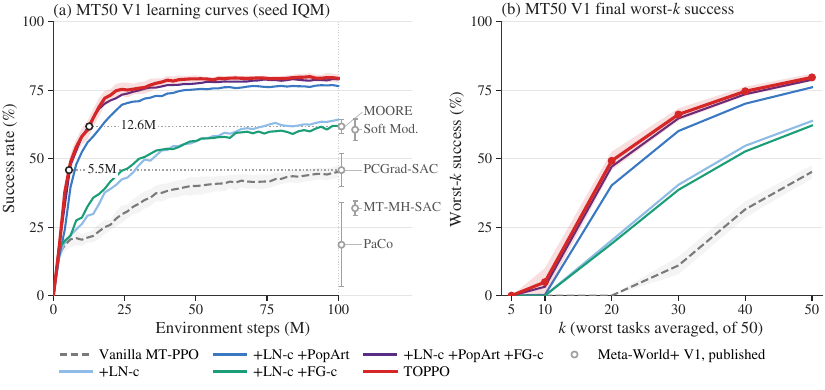}
\caption{MT50 V1 results for the six intervention configurations, ten seeds each. (a) Seed IQM learning curves on common recorded checkpoints, with untrimmed sample-standard-deviation bands for TOPPO and vanilla. Hollow gray right-margin markers show published Meta-World+ V1 IQM values with their reported uncertainty~\citep{metaworldplus2025}. Open circles mark where TOPPO first reaches the published PCGrad-SAC and MOORE values, interpolated between recorded checkpoints; \cref{tab:v1-mwp} reports first recorded crossings. (b) Final worst-$k$ success, computed within each seed and averaged across seeds, with sample-standard-deviation bands for TOPPO and vanilla.}
\label{fig:train-curve-v1}
\end{figure}

\paragraph{V1 worst-$k$ tail tasks.}
Mean and IQM do not show how much of the V1 failure is concentrated in the tail, so \cref{tab:worst-k-v1} reports the same bottom-$k$ diagnostic used for the V2 headline.
\begin{table}[!htbp]
\centering
\caption{\textbf{Worst-$k$ tail-task success rate on Meta-World MT50 under V1 rewards.}
V1 counterpart of the V2 worst-$k$ panel in \cref{tab:worst-k}: per-seed worst-$k$ average for $k\in\{5,10,20,30,40,50\}$, then mean $\pm$ std across $10$ seeds at the $[400{\times}3]$ MLP backbone, final-checkpoint, across the six intervention-stack rungs from Vanilla MT-PPO to TOPPO.
LN-c and PopArt each add about 20 points at $k{=}20$, and LN-c is the larger step at every $k\!\geq\!30$; FG-c lifts the tail only once PopArt is on ($k{=}20$: $40.2\!\to\!47.1$ with PopArt, $20.2\!\to\!19.0$ without); TOPPO is best at every $k\!\geq\!10$.
\textbf{Bold} = column maximum; $k{=}5$ is degenerate ($0\pm 0$ across all rows) and unbolded.}
\label{tab:worst-k-v1}
\scriptsize
\setlength{\tabcolsep}{3.5pt}
\resizebox{\textwidth}{!}{\begin{tabular}{l r@{\,}l r@{\,}l r@{\,}l r@{\,}l r@{\,}l r@{\,}l}
\toprule
 & \multicolumn{12}{c}{Average Worst-$k$ Success Rate (\%)} \\
\cmidrule(lr){2-13}
Algorithm & \multicolumn{2}{c}{$k{=}5$} & \multicolumn{2}{c}{$k{=}10$} & \multicolumn{2}{c}{$k{=}20$} & \multicolumn{2}{c}{$k{=}30$} & \multicolumn{2}{c}{$k{=}40$} & \multicolumn{2}{c}{$k{=}50$} \\
\midrule
Vanilla MT-PPO & $0.0$ & {\fontsize{5pt}{6pt}\selectfont\textcolor{gray}{$\pm\phantom{0}0.0$}} & $0.0$ & {\fontsize{5pt}{6pt}\selectfont\textcolor{gray}{$\pm\phantom{0}0.0$}} & $0.0$ & {\fontsize{5pt}{6pt}\selectfont\textcolor{gray}{$\pm\phantom{0}0.0$}} & $11.0$ & {\fontsize{5pt}{6pt}\selectfont\textcolor{gray}{$\pm\phantom{0}3.3$}} & $31.5$ & {\fontsize{5pt}{6pt}\selectfont\textcolor{gray}{$\pm\phantom{0}3.0$}} & $45.2$ & {\fontsize{5pt}{6pt}\selectfont\textcolor{gray}{$\pm\phantom{0}2.4$}} \\
+LN-c & $0.0$ & {\fontsize{5pt}{6pt}\selectfont\textcolor{gray}{$\pm\phantom{0}0.0$}} & $0.0$ & {\fontsize{5pt}{6pt}\selectfont\textcolor{gray}{$\pm\phantom{0}0.1$}} & $20.2$ & {\fontsize{5pt}{6pt}\selectfont\textcolor{gray}{$\pm\phantom{0}5.0$}} & $40.5$ & {\fontsize{5pt}{6pt}\selectfont\textcolor{gray}{$\pm\phantom{0}4.7$}} & $54.7$ & {\fontsize{5pt}{6pt}\selectfont\textcolor{gray}{$\pm\phantom{0}3.8$}} & $63.8$ & {\fontsize{5pt}{6pt}\selectfont\textcolor{gray}{$\pm\phantom{0}3.1$}} \\
+LN-c +PopArt & $0.0$ & {\fontsize{5pt}{6pt}\selectfont\textcolor{gray}{$\pm\phantom{0}0.0$}} & $0.3$ & {\fontsize{5pt}{6pt}\selectfont\textcolor{gray}{$\pm\phantom{0}0.7$}} & $40.2$ & {\fontsize{5pt}{6pt}\selectfont\textcolor{gray}{$\pm\phantom{0}5.0$}} & $60.1$ & {\fontsize{5pt}{6pt}\selectfont\textcolor{gray}{$\pm\phantom{0}3.3$}} & $70.1$ & {\fontsize{5pt}{6pt}\selectfont\textcolor{gray}{$\pm\phantom{0}2.5$}} & $76.1$ & {\fontsize{5pt}{6pt}\selectfont\textcolor{gray}{$\pm\phantom{0}2.0$}} \\
+LN-c +FG-c & $0.0$ & {\fontsize{5pt}{6pt}\selectfont\textcolor{gray}{$\pm\phantom{0}0.0$}} & $0.0$ & {\fontsize{5pt}{6pt}\selectfont\textcolor{gray}{$\pm\phantom{0}0.1$}} & $19.0$ & {\fontsize{5pt}{6pt}\selectfont\textcolor{gray}{$\pm\phantom{0}4.3$}} & $38.6$ & {\fontsize{5pt}{6pt}\selectfont\textcolor{gray}{$\pm\phantom{0}4.0$}} & $52.6$ & {\fontsize{5pt}{6pt}\selectfont\textcolor{gray}{$\pm\phantom{0}3.4$}} & $62.1$ & {\fontsize{5pt}{6pt}\selectfont\textcolor{gray}{$\pm\phantom{0}2.7$}} \\
+LN-c +PopArt +FG-c & $0.0$ & {\fontsize{5pt}{6pt}\selectfont\textcolor{gray}{$\pm\phantom{0}0.0$}} & $3.2$ & {\fontsize{5pt}{6pt}\selectfont\textcolor{gray}{$\pm\phantom{0}2.8$}} & $47.1$ & {\fontsize{5pt}{6pt}\selectfont\textcolor{gray}{$\pm\phantom{0}3.1$}} & $64.7$ & {\fontsize{5pt}{6pt}\selectfont\textcolor{gray}{$\pm\phantom{0}2.1$}} & $73.5$ & {\fontsize{5pt}{6pt}\selectfont\textcolor{gray}{$\pm\phantom{0}1.5$}} & $78.8$ & {\fontsize{5pt}{6pt}\selectfont\textcolor{gray}{$\pm\phantom{0}1.2$}} \\
\textbf{TOPPO} & $0.0$ & {\fontsize{5pt}{6pt}\selectfont\textcolor{gray}{$\pm\phantom{0}0.0$}} & $\mathbf{4.9}$ & {\fontsize{5pt}{6pt}\selectfont\textcolor{gray}{$\pm\phantom{0}5.1$}} & $\mathbf{49.2}$ & {\fontsize{5pt}{6pt}\selectfont\textcolor{gray}{$\pm\phantom{0}3.4$}} & $\mathbf{66.1}$ & {\fontsize{5pt}{6pt}\selectfont\textcolor{gray}{$\pm\phantom{0}2.2$}} & $\mathbf{74.6}$ & {\fontsize{5pt}{6pt}\selectfont\textcolor{gray}{$\pm\phantom{0}1.7$}} & $\mathbf{79.7}$ & {\fontsize{5pt}{6pt}\selectfont\textcolor{gray}{$\pm\phantom{0}1.3$}} \\
\bottomrule
\end{tabular}}
\end{table}

\Cref{tab:worst-k-v1} is the V1 counterpart of \cref{tab:worst-k}, with three V1-specific observations.
(i)~FG-c without PopArt matches LN-c within seed std at every $k$, the worst-$k$ analog of the FG-c-needs-PopArt-on-V1 finding from the mean column.
(ii)~The dead-task floor on V1 is thicker than on V2 (TOPPO $k{=}10$ is $4.9 \pm 5.1$ on V1 vs.\ $56.5 \pm 7.8$ on V2), so substantial tail-task failure remains on V1.
(iii)~$k{=}5$ is uniformly $0\pm 0$ across all six rows (every V1 seed has $\geq\!5$ tasks at zero), so this column is degenerate and not bolded.

The V1 per-task table (\cref{tab:per-task-mt50-v1}) and per-task learning curves (\cref{fig:per-task-curves-mt50-v1}) are reported alongside the V2 per-task in \cref{app:per-task}; they expose the FG-c-without-PopArt regression and the thicker V1 dead-task floor row-by-row.

\FloatBarrier  %

\subsection{Per-task success rates}\label{app:per-task}

This appendix reports final-checkpoint per-task success ($10$ seeds, mean $\pm$ std) and the matching per-task learning curves for the three reported regimes: MT10 V2 (\cref{tab:per-task-mt10,fig:per-task-curves-mt10-v2}), MT50 V2 (\cref{tab:per-task-mt50,fig:per-task-curves-mt50-v2}), and MT50 V1 (\cref{tab:per-task-mt50-v1,fig:per-task-curves-mt50-v1}). All tables are sorted worst-first by the TOPPO column; the six columns compare the same intervention configurations used throughout the paper (Vanilla MT-PPO, +LN-c, +LN-c+PopArt, +LN-c+FG-c, +LN-c+PopArt+FG-c, TOPPO).
Column definitions and bimodal-row ($\sigma > \mu$) reading guidance are in each table caption.
The companion learning-curve figures share the same seeds and intervention ladder, so the trajectory leading to each table cell can be read directly: which tasks the Vanilla baseline never lifts off zero, at what step the LN-c rung begins to make hard tasks tractable, and where TOPPO's PopArt + FG-c + PCGrad-a stack pulls ahead.

\paragraph{MT50 V2 (\cref{tab:per-task-mt50,fig:per-task-curves-mt50-v2}).} Of the $5$ hardest tasks under vanilla MT-PPO (assembly $0.0\%$, disassemble $0.2\%$, shelf-place $3.2\%$, stick-pull $8.6\%$, pick-place-wall $19.6\%$), $3/5$ recover to $\geq\!70\%$ under TOPPO (shelf-place, stick-pull, pick-place-wall); disassemble partially recovers ($0.2 \to 19.6\%$); only assembly remains at $0\%$.
TOPPO regresses on push-back ($30.0 \to 10.0\%$, bimodal $\sigma > \mu$).
MT10 V2 is near ceiling on most tasks under any rung beyond Vanilla MT-PPO (\cref{fig:per-task-curves-mt10-v2}), so rung-to-rung differences are visible only on \texttt{pick-place-v3}, \texttt{push-v3}, and \texttt{peg-insert-side-v3}; MT10 is reported for completeness, and MT50 carries the main per-task evidence.

\paragraph{MT50 V1 (\cref{tab:per-task-mt50-v1,fig:per-task-curves-mt50-v1}).}
The same drilldown under V1 rewards exposes two V1 effects discussed in \cref{app:v1-sensitivity}.
(i)~\emph{FG-c without PopArt}: the +LN-c+FG-c column stays near +LN-c instead of reaching +LN-c+PopArt on mid-difficulty tasks (e.g.\ \texttt{door-unlock}, \texttt{lever-pull}, \texttt{sweep}), and the $-2.29$-point V1 IQM regression relative to +LN-c is a net effect: drops on tasks such as \texttt{faucet-open} ($97.6 \to 75.4\%$), \texttt{push-back} ($76.8 \to 54.8\%$), and \texttt{window-open} ($76.4 \to 62.0\%$) outweigh gains on tasks such as \texttt{hammer} ($66.4 \to 90.2\%$) and \texttt{disassemble} ($0.0 \to 17.8\%$).
(ii)~\emph{Thicker dead-task floor}: every rung is pinned at $0\%$ on \texttt{hand-insert}, \texttt{handle-pull}, \texttt{peg-insert-side}, \texttt{stick-pull}, \texttt{stick-push}, accounting for the $k{=}5$ degenerate column of \cref{tab:worst-k-v1}.
Where PopArt lifts +LN-c+FG-c (e.g.\ \texttt{door-unlock} $2.4 \to 99.6\%$, \texttt{lever-pull} $35.6 \to 100.0\%$), the gain comes from PopArt effects other than the per-task gradient rescaling that \cref{cor:fg-popart} shows FG-c absorbs, namely value-target normalization and its running statistics.

\begin{figure}[!htbp]
\centering
\includegraphics[width=\linewidth]{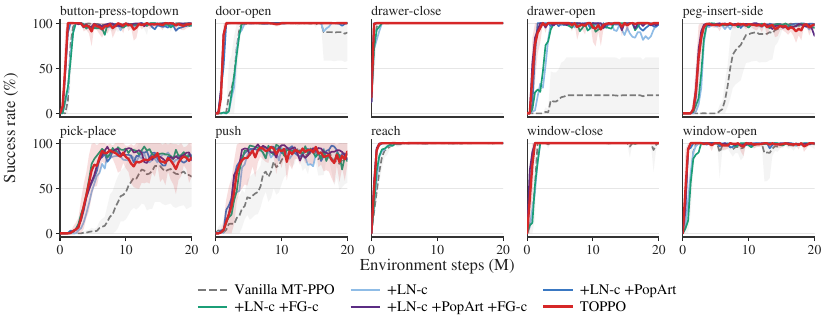}
\caption{Per-task MT10 V2 learning curves for the six intervention configurations. Lines are ten-seed means; shaded bands, drawn only for TOPPO and Vanilla MT-PPO, are sample standard deviations. Each configuration uses only evaluation checkpoints recorded for all ten seeds; the trajectories are unsmoothed. Companion to \cref{tab:per-task-mt10}.}
\label{fig:per-task-curves-mt10-v2}
\end{figure}
\begin{table}[!htbp]
\centering
\caption{\textbf{Per-task final-checkpoint success on MT10 V2 (10-seed mean $\pm$ std).} Rows are Meta-World MT10 tasks sorted worst-first by TOPPO; the six columns are the V2 intervention rungs Vanilla MT-PPO $\to$ $+$LN-c $\to$ $+$LN-c$+$PopArt, the branch $+$LN-c$+$FG-c ($\alpha{=}1$; FG-c added to $+$LN-c without PopArt), then $+$LN-c$+$PopArt$+$FG-c $\to$ TOPPO ($+$PCGrad-a), matching the curves in \cref{fig:per-task-curves-mt10-v2}. Rows with std $>$ mean (e.g.\ \texttt{drawer-open-v3} Vanilla, $20.0 \pm 42.2$) are bimodal across seeds and should be read as a frequency-of-success summary.}
\label{tab:per-task-mt10}
\scriptsize
\setlength{\tabcolsep}{4pt}
\resizebox{\textwidth}{!}{\begin{tabular}{l r@{\,}l r@{\,}l r@{\,}l r@{\,}l r@{\,}l r@{\,}l}
\toprule
Task & \multicolumn{2}{c}{Vanilla MT-PPO} & \multicolumn{2}{c}{+ LN-c} & \multicolumn{2}{c}{+ LN-c + PopArt} & \multicolumn{2}{c}{+ LN-c + FG-c} & \multicolumn{2}{c}{+ LN-c + PopArt + FG-c} & \multicolumn{2}{c}{\textbf{TOPPO}} \\
\midrule
\texttt{pick-place-v3} & $62.8$ & {\fontsize{5pt}{6pt}\selectfont\textcolor{gray}{$\pm32.8$}} & $86.2$ & {\fontsize{5pt}{6pt}\selectfont\textcolor{gray}{$\pm20.0$}} & $86.0$ & {\fontsize{5pt}{6pt}\selectfont\textcolor{gray}{$\pm11.7$}} & $85.2$ & {\fontsize{5pt}{6pt}\selectfont\textcolor{gray}{$\pm17.2$}} & $80.8$ & {\fontsize{5pt}{6pt}\selectfont\textcolor{gray}{$\pm17.3$}} & $83.8$ & {\fontsize{5pt}{6pt}\selectfont\textcolor{gray}{$\pm20.9$}} \\
\texttt{push-v3} & $83.2$ & {\fontsize{5pt}{6pt}\selectfont\textcolor{gray}{$\pm27.1$}} & $81.6$ & {\fontsize{5pt}{6pt}\selectfont\textcolor{gray}{$\pm28.4$}} & $87.6$ & {\fontsize{5pt}{6pt}\selectfont\textcolor{gray}{$\pm19.5$}} & $79.8$ & {\fontsize{5pt}{6pt}\selectfont\textcolor{gray}{$\pm25.2$}} & $81.8$ & {\fontsize{5pt}{6pt}\selectfont\textcolor{gray}{$\pm24.2$}} & $90.8$ & {\fontsize{5pt}{6pt}\selectfont\textcolor{gray}{$\pm14.8$}} \\
\texttt{peg-insert-side-v3} & $98.0$ & {\fontsize{5pt}{6pt}\selectfont\textcolor{gray}{$\pm\phantom{0}5.7$}} & $97.4$ & {\fontsize{5pt}{6pt}\selectfont\textcolor{gray}{$\pm\phantom{0}4.2$}} & $99.0$ & {\fontsize{5pt}{6pt}\selectfont\textcolor{gray}{$\pm\phantom{0}1.4$}} & $94.8$ & {\fontsize{5pt}{6pt}\selectfont\textcolor{gray}{$\pm10.5$}} & $85.8$ & {\fontsize{5pt}{6pt}\selectfont\textcolor{gray}{$\pm25.6$}} & $98.4$ & {\fontsize{5pt}{6pt}\selectfont\textcolor{gray}{$\pm\phantom{0}3.7$}} \\
\texttt{button-press-topdown-v3} & $99.8$ & {\fontsize{5pt}{6pt}\selectfont\textcolor{gray}{$\pm\phantom{0}0.6$}} & $98.6$ & {\fontsize{5pt}{6pt}\selectfont\textcolor{gray}{$\pm\phantom{0}3.3$}} & $97.0$ & {\fontsize{5pt}{6pt}\selectfont\textcolor{gray}{$\pm\phantom{0}5.9$}} & $97.6$ & {\fontsize{5pt}{6pt}\selectfont\textcolor{gray}{$\pm\phantom{0}5.4$}} & $100.0$ & {\fontsize{5pt}{6pt}\selectfont\textcolor{gray}{$\pm\phantom{0}0.0$}} & $99.4$ & {\fontsize{5pt}{6pt}\selectfont\textcolor{gray}{$\pm\phantom{0}1.3$}} \\
\texttt{window-open-v3} & $98.4$ & {\fontsize{5pt}{6pt}\selectfont\textcolor{gray}{$\pm\phantom{0}5.1$}} & $99.4$ & {\fontsize{5pt}{6pt}\selectfont\textcolor{gray}{$\pm\phantom{0}1.3$}} & $99.4$ & {\fontsize{5pt}{6pt}\selectfont\textcolor{gray}{$\pm\phantom{0}1.3$}} & $99.4$ & {\fontsize{5pt}{6pt}\selectfont\textcolor{gray}{$\pm\phantom{0}1.3$}} & $100.0$ & {\fontsize{5pt}{6pt}\selectfont\textcolor{gray}{$\pm\phantom{0}0.0$}} & $99.6$ & {\fontsize{5pt}{6pt}\selectfont\textcolor{gray}{$\pm\phantom{0}1.3$}} \\
\texttt{door-open-v3} & $89.6$ & {\fontsize{5pt}{6pt}\selectfont\textcolor{gray}{$\pm31.5$}} & $100.0$ & {\fontsize{5pt}{6pt}\selectfont\textcolor{gray}{$\pm\phantom{0}0.0$}} & $100.0$ & {\fontsize{5pt}{6pt}\selectfont\textcolor{gray}{$\pm\phantom{0}0.0$}} & $100.0$ & {\fontsize{5pt}{6pt}\selectfont\textcolor{gray}{$\pm\phantom{0}0.0$}} & $100.0$ & {\fontsize{5pt}{6pt}\selectfont\textcolor{gray}{$\pm\phantom{0}0.0$}} & $100.0$ & {\fontsize{5pt}{6pt}\selectfont\textcolor{gray}{$\pm\phantom{0}0.0$}} \\
\texttt{drawer-close-v3} & $100.0$ & {\fontsize{5pt}{6pt}\selectfont\textcolor{gray}{$\pm\phantom{0}0.0$}} & $100.0$ & {\fontsize{5pt}{6pt}\selectfont\textcolor{gray}{$\pm\phantom{0}0.0$}} & $100.0$ & {\fontsize{5pt}{6pt}\selectfont\textcolor{gray}{$\pm\phantom{0}0.0$}} & $100.0$ & {\fontsize{5pt}{6pt}\selectfont\textcolor{gray}{$\pm\phantom{0}0.0$}} & $100.0$ & {\fontsize{5pt}{6pt}\selectfont\textcolor{gray}{$\pm\phantom{0}0.0$}} & $100.0$ & {\fontsize{5pt}{6pt}\selectfont\textcolor{gray}{$\pm\phantom{0}0.0$}} \\
\texttt{drawer-open-v3} & $20.0$ & {\fontsize{5pt}{6pt}\selectfont\textcolor{gray}{$\pm42.2$}} & $99.0$ & {\fontsize{5pt}{6pt}\selectfont\textcolor{gray}{$\pm\phantom{0}3.2$}} & $98.6$ & {\fontsize{5pt}{6pt}\selectfont\textcolor{gray}{$\pm\phantom{0}4.4$}} & $100.0$ & {\fontsize{5pt}{6pt}\selectfont\textcolor{gray}{$\pm\phantom{0}0.0$}} & $100.0$ & {\fontsize{5pt}{6pt}\selectfont\textcolor{gray}{$\pm\phantom{0}0.0$}} & $100.0$ & {\fontsize{5pt}{6pt}\selectfont\textcolor{gray}{$\pm\phantom{0}0.0$}} \\
\texttt{reach-v3} & $99.8$ & {\fontsize{5pt}{6pt}\selectfont\textcolor{gray}{$\pm\phantom{0}0.6$}} & $100.0$ & {\fontsize{5pt}{6pt}\selectfont\textcolor{gray}{$\pm\phantom{0}0.0$}} & $100.0$ & {\fontsize{5pt}{6pt}\selectfont\textcolor{gray}{$\pm\phantom{0}0.0$}} & $100.0$ & {\fontsize{5pt}{6pt}\selectfont\textcolor{gray}{$\pm\phantom{0}0.0$}} & $100.0$ & {\fontsize{5pt}{6pt}\selectfont\textcolor{gray}{$\pm\phantom{0}0.0$}} & $100.0$ & {\fontsize{5pt}{6pt}\selectfont\textcolor{gray}{$\pm\phantom{0}0.0$}} \\
\texttt{window-close-v3} & $100.0$ & {\fontsize{5pt}{6pt}\selectfont\textcolor{gray}{$\pm\phantom{0}0.0$}} & $100.0$ & {\fontsize{5pt}{6pt}\selectfont\textcolor{gray}{$\pm\phantom{0}0.0$}} & $100.0$ & {\fontsize{5pt}{6pt}\selectfont\textcolor{gray}{$\pm\phantom{0}0.0$}} & $100.0$ & {\fontsize{5pt}{6pt}\selectfont\textcolor{gray}{$\pm\phantom{0}0.0$}} & $100.0$ & {\fontsize{5pt}{6pt}\selectfont\textcolor{gray}{$\pm\phantom{0}0.0$}} & $100.0$ & {\fontsize{5pt}{6pt}\selectfont\textcolor{gray}{$\pm\phantom{0}0.0$}} \\
\bottomrule
\end{tabular}}
\end{table}

\begin{table}[!htbp]
\centering
\caption{\textbf{Per-task success rate on MT50 V2 (final-checkpoint, 10-seed mean $\pm$ std).} Tasks sorted ascending by TOPPO mean (worst first). Columns trace the cumulative TOPPO build-up: \emph{Vanilla MT-PPO} (default actor and critic aggregators, no PopArt, no LN), $+$ \emph{LN-c} (add LayerNorm on the critic), $+$ \emph{LN-c + PopArt} (add PopArt return rescaling), $+$ \emph{LN-c + FG-c} (instead add $\alpha{=}1$ FairGrad on the critic), $+$ \emph{LN-c + PopArt + FG-c} (combine both interventions), and \emph{TOPPO} $=$ $+$ LN-c $+$ PopArt $+$ $\alpha{=}1$ FG-c $+$ PCGrad-a (full method). Entries with std $>$ mean (e.g.\ \texttt{push-back-v3} TOPPO ($10.0 \pm 31.6$), \texttt{disassemble-v3} TOPPO ($19.6 \pm 40.8$)) have a bimodal seed distribution (some seeds at $\sim 0\%$, some near $100\%$); the reported $\pm$std should be read as a frequency-of-success summary rather than a Gaussian dispersion in those rows.}
\label{tab:per-task-mt50}
\scriptsize
\setlength{\tabcolsep}{4pt}
\resizebox{\textwidth}{!}{
\begin{tabular}{l r@{\,}l r@{\,}l r@{\,}l r@{\,}l r@{\,}l r@{\,}l}
\toprule
Task & \multicolumn{2}{c}{Vanilla MT-PPO} & \multicolumn{2}{c}{+ LN-c} & \multicolumn{2}{c}{+ LN-c + PopArt} & \multicolumn{2}{c}{+ LN-c + FG-c} & \multicolumn{2}{c}{+ LN-c + PopArt + FG-c} & \multicolumn{2}{c}{\textbf{TOPPO}} \\
\midrule
\texttt{assembly-v3} & $0.0$ & {\fontsize{5pt}{6pt}\selectfont\textcolor{gray}{$\pm\phantom{0}0.0$}} & $0.0$ & {\fontsize{5pt}{6pt}\selectfont\textcolor{gray}{$\pm\phantom{0}0.0$}} & $0.2$ & {\fontsize{5pt}{6pt}\selectfont\textcolor{gray}{$\pm\phantom{0}0.6$}} & $0.0$ & {\fontsize{5pt}{6pt}\selectfont\textcolor{gray}{$\pm\phantom{0}0.0$}} & $0.0$ & {\fontsize{5pt}{6pt}\selectfont\textcolor{gray}{$\pm\phantom{0}0.0$}} & $0.0$ & {\fontsize{5pt}{6pt}\selectfont\textcolor{gray}{$\pm\phantom{0}0.0$}} \\
\texttt{push-back-v3} & $30.0$ & {\fontsize{5pt}{6pt}\selectfont\textcolor{gray}{$\pm48.3$}} & $50.0$ & {\fontsize{5pt}{6pt}\selectfont\textcolor{gray}{$\pm52.7$}} & $0.0$ & {\fontsize{5pt}{6pt}\selectfont\textcolor{gray}{$\pm\phantom{0}0.0$}} & $20.0$ & {\fontsize{5pt}{6pt}\selectfont\textcolor{gray}{$\pm42.2$}} & $9.2$ & {\fontsize{5pt}{6pt}\selectfont\textcolor{gray}{$\pm29.1$}} & $10.0$ & {\fontsize{5pt}{6pt}\selectfont\textcolor{gray}{$\pm31.6$}} \\
\texttt{disassemble-v3} & $0.2$ & {\fontsize{5pt}{6pt}\selectfont\textcolor{gray}{$\pm\phantom{0}0.6$}} & $8.0$ & {\fontsize{5pt}{6pt}\selectfont\textcolor{gray}{$\pm25.3$}} & $9.4$ & {\fontsize{5pt}{6pt}\selectfont\textcolor{gray}{$\pm29.0$}} & $19.0$ & {\fontsize{5pt}{6pt}\selectfont\textcolor{gray}{$\pm40.1$}} & $18.6$ & {\fontsize{5pt}{6pt}\selectfont\textcolor{gray}{$\pm39.3$}} & $19.6$ & {\fontsize{5pt}{6pt}\selectfont\textcolor{gray}{$\pm40.8$}} \\
\bottomrule
\end{tabular}}
\end{table}

\FloatBarrier
\begin{figure}[!htbp]
\centering
\includegraphics[width=\linewidth]{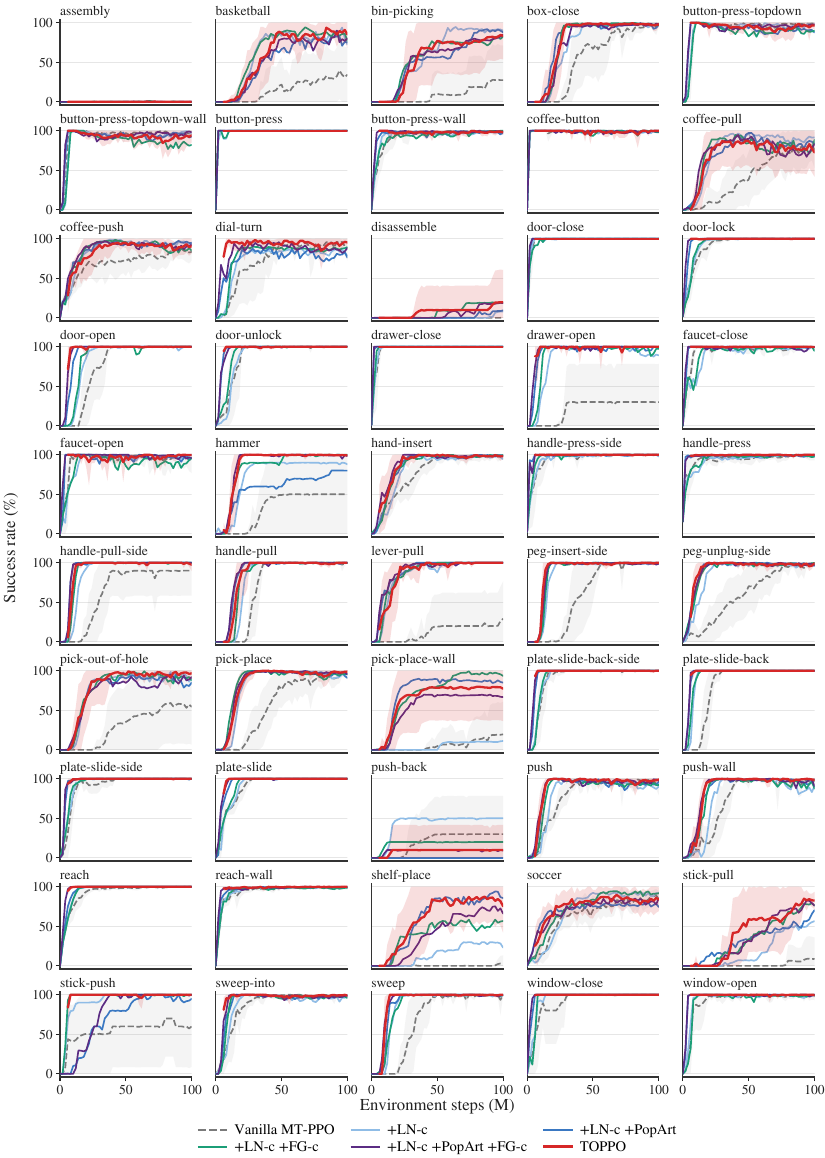}
\caption{Per-task MT50 V2 learning curves across all 50 tasks for the six intervention configurations. Lines are ten-seed means; shaded bands, drawn only for TOPPO and Vanilla MT-PPO, are sample standard deviations. Each configuration uses only evaluation checkpoints recorded for all ten seeds; the trajectories are unsmoothed.}
\label{fig:per-task-curves-mt50-v2}
\end{figure}

\begin{table}[!htbp]
\centering
\caption{\textbf{Per-task final-checkpoint success on MT50 V1 (10-seed mean $\pm$ std).} Rows are Meta-World MT50 tasks sorted worst-first by TOPPO (V1); the six columns are the V1 intervention rungs Vanilla MT-PPO $\to$ $+$LN-c $\to$ $+$LN-c$+$PopArt, the branch $+$LN-c$+$FG-c ($\alpha{=}1$; FG-c added to $+$LN-c without PopArt), then $+$LN-c$+$PopArt$+$FG-c $\to$ TOPPO (V1) ($+$PCGrad-a), matching the curves in \cref{fig:per-task-curves-mt50-v1}. Under V1 rewards, which have wider per-task reward scale~\citep{metaworldplus2025}, the dead-task floor is thicker (\cref{tab:worst-k-v1}) and FG-c without PopArt regresses (\cref{app:v1-sensitivity}). Rows with std $>$ mean (e.g.\ \texttt{plate-slide-back-v3} Vanilla, $30.0 \pm 48.3$) are bimodal across seeds and should be read as a frequency-of-success summary.}
\label{tab:per-task-mt50-v1}
\scriptsize
\setlength{\tabcolsep}{4pt}
\resizebox{\textwidth}{!}{
\begin{tabular}{l r@{\,}l r@{\,}l r@{\,}l r@{\,}l r@{\,}l r@{\,}l}
\toprule
Task & \multicolumn{2}{c}{Vanilla MT-PPO} & \multicolumn{2}{c}{+ LN-c} & \multicolumn{2}{c}{+ LN-c + PopArt} & \multicolumn{2}{c}{+ LN-c + FG-c} & \multicolumn{2}{c}{+ LN-c + PopArt + FG-c} & \multicolumn{2}{c}{\textbf{TOPPO (V1)}} \\
\midrule
\texttt{box-close-v3} & $0.0$ & {\fontsize{5pt}{6pt}\selectfont\textcolor{gray}{$\pm\phantom{0}0.0$}} & $7.4$ & {\fontsize{5pt}{6pt}\selectfont\textcolor{gray}{$\pm11.1$}} & $2.2$ & {\fontsize{5pt}{6pt}\selectfont\textcolor{gray}{$\pm\phantom{0}7.0$}} & $0.4$ & {\fontsize{5pt}{6pt}\selectfont\textcolor{gray}{$\pm\phantom{0}1.3$}} & $0.0$ & {\fontsize{5pt}{6pt}\selectfont\textcolor{gray}{$\pm\phantom{0}0.0$}} & $0.0$ & {\fontsize{5pt}{6pt}\selectfont\textcolor{gray}{$\pm\phantom{0}0.0$}} \\
\texttt{hand-insert-v3} & $0.0$ & {\fontsize{5pt}{6pt}\selectfont\textcolor{gray}{$\pm\phantom{0}0.0$}} & $0.0$ & {\fontsize{5pt}{6pt}\selectfont\textcolor{gray}{$\pm\phantom{0}0.0$}} & $0.0$ & {\fontsize{5pt}{6pt}\selectfont\textcolor{gray}{$\pm\phantom{0}0.0$}} & $0.0$ & {\fontsize{5pt}{6pt}\selectfont\textcolor{gray}{$\pm\phantom{0}0.0$}} & $0.0$ & {\fontsize{5pt}{6pt}\selectfont\textcolor{gray}{$\pm\phantom{0}0.0$}} & $0.0$ & {\fontsize{5pt}{6pt}\selectfont\textcolor{gray}{$\pm\phantom{0}0.0$}} \\
\texttt{handle-pull-v3} & $0.0$ & {\fontsize{5pt}{6pt}\selectfont\textcolor{gray}{$\pm\phantom{0}0.0$}} & $0.0$ & {\fontsize{5pt}{6pt}\selectfont\textcolor{gray}{$\pm\phantom{0}0.0$}} & $0.0$ & {\fontsize{5pt}{6pt}\selectfont\textcolor{gray}{$\pm\phantom{0}0.0$}} & $0.0$ & {\fontsize{5pt}{6pt}\selectfont\textcolor{gray}{$\pm\phantom{0}0.0$}} & $0.0$ & {\fontsize{5pt}{6pt}\selectfont\textcolor{gray}{$\pm\phantom{0}0.0$}} & $0.0$ & {\fontsize{5pt}{6pt}\selectfont\textcolor{gray}{$\pm\phantom{0}0.0$}} \\
\bottomrule
\end{tabular}}
\end{table}

\FloatBarrier
\begin{figure}[!htbp]
\centering
\includegraphics[width=\linewidth]{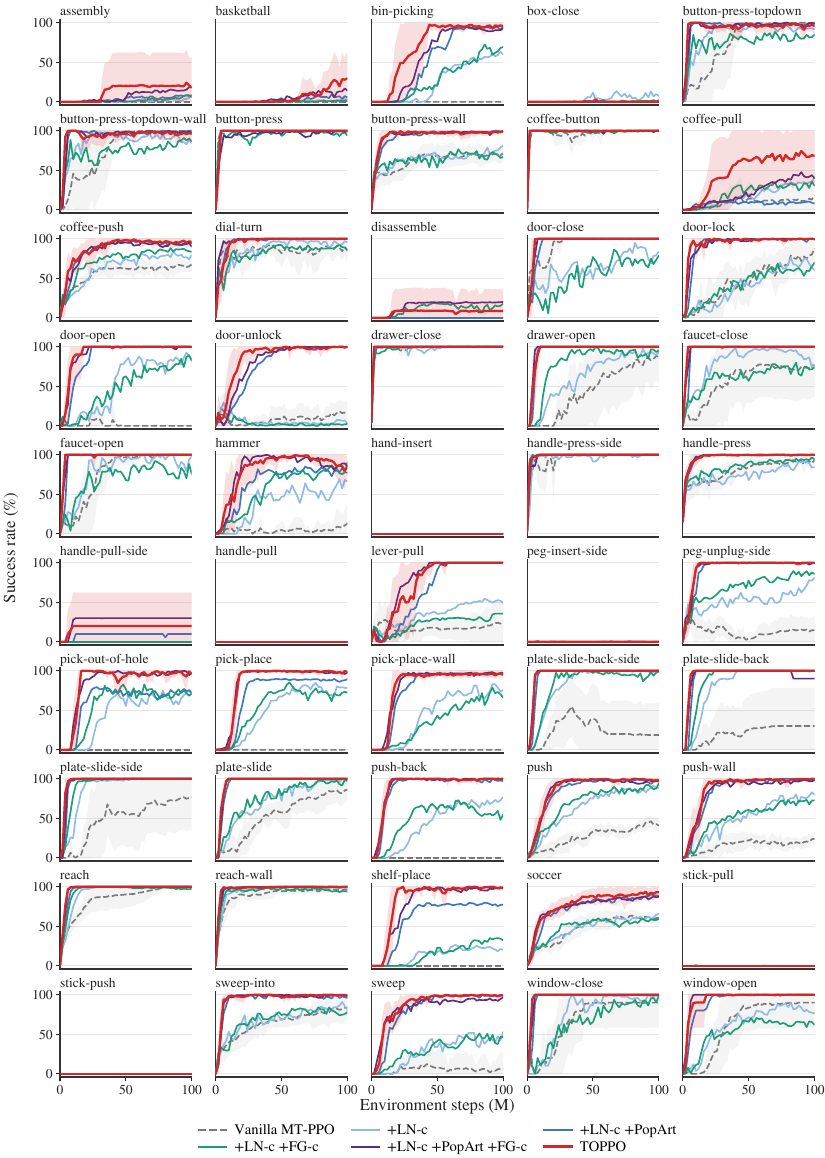}
\caption{Per-task MT50 V1 learning curves across all 50 tasks for the six intervention configurations under V1 rewards. Lines are ten-seed means; shaded bands, drawn only for TOPPO and Vanilla MT-PPO, are sample standard deviations. Each configuration uses only evaluation checkpoints recorded for all ten seeds; the trajectories are unsmoothed.}
\label{fig:per-task-curves-mt50-v1}
\end{figure}

\FloatBarrier  %

\subsection{Cross-paper reporting-protocol map}\label{app:reporting-protocol}

Reporting conventions across the MTRL literature are not standardized: papers differ on the aggregator (final-checkpoint vs.\ mean-curve peak vs.\ best-per-seed), the central statistic (mean vs.\ IQM), the dispersion measure (std vs.\ SEM), the seed count, and the V1/V2 reward semantics.
\Cref{tab:protocol} maps each cited baseline to its source paper's choices and \cref{tab:step-convention} documents step-budget alignment; \cref{tab:4metric,tab:bestckpt-steps} report the four-metric breakdown and per-seed best-checkpoint-step distribution for our headline configurations so readers can convert to any baseline's protocol.
Three further points matter for the headline comparisons:

\begin{itemize}\itemsep1pt\parskip0pt
\item \textbf{SEM versus std.} The mtrl-codebase family (CAGrad / CARE / FAMO / FairGrad) reports SEM, not std; SEM $\equiv$ std/$\sqrt{n}$ is roughly $3.16\times$ tighter at $n{=}10$ and is paired with mean-curve-peak aggregation, which is upward-biased.
Worked conversions: FairGrad MT10 $0.84 \pm 0.07$ SEM $\equiv \pm 0.22$ std; CAGrad MT10 $0.83 \pm 0.045$ SEM $\equiv \pm 0.142$ std; CAGrad MT50 $0.52 \pm 0.023$ SEM $\equiv \pm 0.073$ std.
\item \textbf{Step budgets match.} All cited baselines train at our 20M/100M budget except the original Soft Modularization paper (15M/50M); we instead use its 20M/100M re-runs from ARS (V2, \cref{tab:headline}) and Meta-World+ (V1, \cref{tab:v1-mwp}), so step budgets need no renormalization.
Meta-World+ does not list its MT50 configurations in the paper; we take them from its companion \texttt{metaworld-algorithms} repository (PaCo on MT50 uses $K{=}20$).
Evaluation protocols still differ (\cref{tab:protocol} caption, \cref{app:revision-baselines}).
\item \textbf{Reward versions must match.} Older gradient-method reports use V1 rewards, whereas our default setting uses V2.
A V1-baseline/V2-TOPPO contrast would favor TOPPO through the reward change and cannot isolate an algorithmic difference.
We therefore use the reward-matched V2 rows in \cref{tab:headline} and the separate V1 comparison in \cref{tab:v1-mwp}.
\end{itemize}

\begin{table}[!htbp]
\centering
\caption{\textbf{Step-budget alignment across cited baselines.} We convert every paper's reporting convention --- ``per-task'', ``per mtrl-iteration'', ``epochs $\times$ batch'' --- to total environment steps (every parallel-env step counts $+1$); all rows train at our $20\mathrm{M}$\,/\,$100\mathrm{M}$ budget except Soft Modularization (15M\,/\,50M paper-original; \cref{tab:headline} and \cref{tab:v1-mwp} use its 20M\,/\,100M re-runs from ARS and Meta-World+), so cross-paper comparisons need no step renormalization.}
\label{tab:step-convention}
\small
\resizebox{\textwidth}{!}{\begin{tabular}{l l c c l}
\toprule
Paper & Reported & Total MT10 & Total MT50 & Basis \\
\midrule
\textbf{Ours (PPO)} & 20M / 100M (total) & 20M & 100M & this paper \\
Original Meta-World (2020) & 500 epochs $\times$ batch & 20M & 100M & garage scripts \\
CAGrad (2021) & 2M (mtrl iter) & 20M & 100M & $\texttt{num\_envs}=\texttt{num\_tasks}$ \\
CARE (2021) & 2M per task & 20M & 100M & explicit \\
FAMO (2023) & follows CAGrad exactly & 20M & -- & MT10 only \\
FairGrad (2024) & 2M / batch 1280 & 20M & -- & MT10 only \\
PaCo (2022) & 20M total / 100M total & 20M & 100M & explicit \\
MOORE (2024) & 20 ep $\times$ 100K vec $\times$ envs & 20M & 100M & \texttt{VecCore} source \\
ARS (2025) & $2{\cdot}10^7$ / $10^8$ & 20M & 100M & matches ours exactly \\
Soft Modularization (2020) & 15M / 50M & \textbf{15M} & \textbf{50M} & outlier; JSON configs \\
Meta-World+ (2025) multi-task & not explicitly tabulated & 20M & 100M & companion-repo configs (\cref{app:reporting-protocol}) \\
\bottomrule
\end{tabular}}
\end{table}

\begin{table}[!htbp]
\centering
\caption{\textbf{Reporting protocol across cited baselines.} For each cited paper we tabulate seed count, aggregator (A: final-checkpoint average; B: best-per-seed then average; C: mean-curve peak; D: late-window mean; E: other), dispersion measure (std / SEM / IQM), eval episodes per task, and reward version, so readers can align cross-paper comparisons in \cref{tab:headline,tab:v1-mwp} with the matching column of \cref{tab:4metric}. V1$^*$ marks rows inferred from a pinned pre-V2 Meta-World commit; n/r: not reported. $^\dagger$~Meta-World+ Appendix~D tables do not specify whether $\pm$ is std, SEM, or a $95\%$~CI; we treat it as std (\cref{tab:v1-mwp,tab:v1-sensitivity}). Published ARS results also differ from ours in evaluation protocol: they use stochastic evaluation on $10$ sampled goals per task, sampled training goals, and a different MuJoCo binding, whereas our reruns use our deterministic protocol on the $50$ fixed benchmark goals (\cref{app:revision-baselines}).}
\label{tab:protocol}
\small
\setlength{\tabcolsep}{4pt}
\resizebox{\textwidth}{!}{\begin{tabular}{l c c l c c c}
\toprule
Paper & Seeds & Aggreg. & Description & Disp. & Eval eps. & Reward \\
\midrule
\textbf{Ours (PPO)} & 10 & all of A/B/C/D & best/peak/final/late, all reported & std & 50 & V2 \\
MT-SAC (orig 2020) & 10 & (E) & max-over-time per (seed,task) & none & 5 & V1 \\
Soft-Mod (2020) & 3 & (E) & final-policy 100-ep dedicated pass & (graphical only) & 100 & V1 \\
PCGrad (2020) & n/r & -- & learning curves only (no numeric table) & -- & n/r & V1 \\
CAGrad (2021) & 10 & (C) & mean-curve peak (App B.3) & \textbf{SEM} & 1 & V1$^*$ \\
CARE (2021) & 10 & (C) & mean-curve peak & SEM & 5 & V1 \\
PaCo (2022) & 10 & (A) & final 20M, explicit anti-peak & std & 10 & \textbf{V2} \\
FAMO (2023) & 10 & (C) & inherits CAGrad & SEM & 1 & V1$^*$ \\
MOORE (2024) & 10 & (A) & final ckpt (20M MT10 / 100M MT50) & std & 10 & \textbf{V2} \\
FairGrad (2024) & 10 & (C) & mean-curve peak (\S 5.3) & SEM & 1 & V1$^*$ \\
Meta-World+ (2025) & 10 & (A/B) & IQM (final-step or best ckpt; ambiguous) & IQM, $\pm$ unspec.$^\dagger$ & 50 & V1+V2 \\
\textbf{ARS (2025)} & \textbf{8} & (A) & final EvalFinalSR & std & 10 & \textbf{V2} \\
Our SAC-MT / ARS reruns & 10 & (A) & final ckpt, our evaluation protocol & std & 50 & V2 \\
\bottomrule
\end{tabular}}
\end{table}

\begin{table}[!htbp]
\centering
\caption{\textbf{Four-metric disclosure for our headline configurations.} Each row reports MT10 and MT50 success ($\%$, mean$\,\pm\,$std over 10 seeds at 20M / 100M total environment steps) under four aggregators that span the conventions used by cited baselines: \emph{best} (per-seed argmax then averaged), \emph{peak} (argmax of the mean curve), \emph{final} (terminal step), and \emph{late} (mean of the last 10\% of steps). Reporting all four lets readers convert our numbers to any baseline's protocol (\cref{tab:protocol}).}
\label{tab:4metric}
\scriptsize
\setlength{\tabcolsep}{3pt}
\resizebox{\textwidth}{!}{\begin{tabular}{l r@{\,}l r@{\,}l r@{\,}l r@{\,}l r@{\,}l r@{\,}l r@{\,}l r@{\,}l}
\toprule
& \multicolumn{8}{c}{MT10} & \multicolumn{8}{c}{MT50} \\
\cmidrule(lr){2-9} \cmidrule(lr){10-17}
Configuration & \multicolumn{2}{c}{best} & \multicolumn{2}{c}{peak} & \multicolumn{2}{c}{final} & \multicolumn{2}{c}{late} & \multicolumn{2}{c}{best} & \multicolumn{2}{c}{peak} & \multicolumn{2}{c}{final} & \multicolumn{2}{c}{late} \\
\midrule
Vanilla MT-PPO & $89.5$ & {\fontsize{5pt}{6pt}\selectfont\textcolor{gray}{$\pm\phantom{0}4.6$}} & $87.8$ & {\fontsize{5pt}{6pt}\selectfont\textcolor{gray}{$\pm\phantom{0}5.1$}} & $85.2$ & {\fontsize{5pt}{6pt}\selectfont\textcolor{gray}{$\pm\phantom{0}6.5$}} & $85.1$ & {\fontsize{5pt}{6pt}\selectfont\textcolor{gray}{$\pm\phantom{0}5.7$}} & $79.6$ & {\fontsize{5pt}{6pt}\selectfont\textcolor{gray}{$\pm\phantom{0}3.2$}} & $78.7$ & {\fontsize{5pt}{6pt}\selectfont\textcolor{gray}{$\pm\phantom{0}3.7$}} & $78.5$ & {\fontsize{5pt}{6pt}\selectfont\textcolor{gray}{$\pm\phantom{0}3.3$}} & $78.5$ & {\fontsize{5pt}{6pt}\selectfont\textcolor{gray}{$\pm\phantom{0}3.5$}} \\
+LN-c & $99.5$ & {\fontsize{5pt}{6pt}\selectfont\textcolor{gray}{$\pm\phantom{0}0.3$}} & $97.9$ & {\fontsize{5pt}{6pt}\selectfont\textcolor{gray}{$\pm\phantom{0}1.6$}} & $96.2$ & {\fontsize{5pt}{6pt}\selectfont\textcolor{gray}{$\pm\phantom{0}3.4$}} & $95.3$ & {\fontsize{5pt}{6pt}\selectfont\textcolor{gray}{$\pm\phantom{0}1.9$}} & $89.2$ & {\fontsize{5pt}{6pt}\selectfont\textcolor{gray}{$\pm\phantom{0}1.9$}} & $88.3$ & {\fontsize{5pt}{6pt}\selectfont\textcolor{gray}{$\pm\phantom{0}1.9$}} & $87.5$ & {\fontsize{5pt}{6pt}\selectfont\textcolor{gray}{$\pm\phantom{0}2.4$}} & $87.7$ & {\fontsize{5pt}{6pt}\selectfont\textcolor{gray}{$\pm\phantom{0}2.2$}} \\
+LN-c +PopArt & $99.9$ & {\fontsize{5pt}{6pt}\selectfont\textcolor{gray}{$\pm\phantom{0}0.2$}} & $98.5$ & {\fontsize{5pt}{6pt}\selectfont\textcolor{gray}{$\pm\phantom{0}1.2$}} & $96.8$ & {\fontsize{5pt}{6pt}\selectfont\textcolor{gray}{$\pm\phantom{0}1.5$}} & $96.4$ & {\fontsize{5pt}{6pt}\selectfont\textcolor{gray}{$\pm\phantom{0}2.3$}} & $90.7$ & {\fontsize{5pt}{6pt}\selectfont\textcolor{gray}{$\pm\phantom{0}1.8$}} & $89.2$ & {\fontsize{5pt}{6pt}\selectfont\textcolor{gray}{$\pm\phantom{0}2.2$}} & $88.9$ & {\fontsize{5pt}{6pt}\selectfont\textcolor{gray}{$\pm\phantom{0}2.3$}} & $89.0$ & {\fontsize{5pt}{6pt}\selectfont\textcolor{gray}{$\pm\phantom{0}1.8$}} \\
+LN-c +FG-c & $99.8$ & {\fontsize{5pt}{6pt}\selectfont\textcolor{gray}{$\pm\phantom{0}0.2$}} & $98.6$ & {\fontsize{5pt}{6pt}\selectfont\textcolor{gray}{$\pm\phantom{0}0.7$}} & $95.7$ & {\fontsize{5pt}{6pt}\selectfont\textcolor{gray}{$\pm\phantom{0}4.0$}} & $96.2$ & {\fontsize{5pt}{6pt}\selectfont\textcolor{gray}{$\pm\phantom{0}2.4$}} & $91.2$ & {\fontsize{5pt}{6pt}\selectfont\textcolor{gray}{$\pm\phantom{0}1.6$}} & $89.9$ & {\fontsize{5pt}{6pt}\selectfont\textcolor{gray}{$\pm\phantom{0}1.7$}} & $89.7$ & {\fontsize{5pt}{6pt}\selectfont\textcolor{gray}{$\pm\phantom{0}1.6$}} & $89.6$ & {\fontsize{5pt}{6pt}\selectfont\textcolor{gray}{$\pm\phantom{0}1.5$}} \\
+LN-c +PopArt +FG-c & $99.8$ & {\fontsize{5pt}{6pt}\selectfont\textcolor{gray}{$\pm\phantom{0}0.2$}} & $99.0$ & {\fontsize{5pt}{6pt}\selectfont\textcolor{gray}{$\pm\phantom{0}0.6$}} & $94.8$ & {\fontsize{5pt}{6pt}\selectfont\textcolor{gray}{$\pm\phantom{0}5.1$}} & $96.7$ & {\fontsize{5pt}{6pt}\selectfont\textcolor{gray}{$\pm\phantom{0}2.0$}} & $91.2$ & {\fontsize{5pt}{6pt}\selectfont\textcolor{gray}{$\pm\phantom{0}2.1$}} & $90.1$ & {\fontsize{5pt}{6pt}\selectfont\textcolor{gray}{$\pm\phantom{0}2.3$}} & $90.1$ & {\fontsize{5pt}{6pt}\selectfont\textcolor{gray}{$\pm\phantom{0}2.3$}} & $89.7$ & {\fontsize{5pt}{6pt}\selectfont\textcolor{gray}{$\pm\phantom{0}2.2$}} \\
\textbf{TOPPO} & $99.6$ & {\fontsize{5pt}{6pt}\selectfont\textcolor{gray}{$\pm\phantom{0}0.4$}} & $97.9$ & {\fontsize{5pt}{6pt}\selectfont\textcolor{gray}{$\pm\phantom{0}2.2$}} & $97.2$ & {\fontsize{5pt}{6pt}\selectfont\textcolor{gray}{$\pm\phantom{0}3.1$}} & $96.0$ & {\fontsize{5pt}{6pt}\selectfont\textcolor{gray}{$\pm\phantom{0}1.6$}} & $92.1$ & {\fontsize{5pt}{6pt}\selectfont\textcolor{gray}{$\pm\phantom{0}1.7$}} & $91.3$ & {\fontsize{5pt}{6pt}\selectfont\textcolor{gray}{$\pm\phantom{0}1.6$}} & $90.9$ & {\fontsize{5pt}{6pt}\selectfont\textcolor{gray}{$\pm\phantom{0}1.6$}} & $91.0$ & {\fontsize{5pt}{6pt}\selectfont\textcolor{gray}{$\pm\phantom{0}1.7$}} \\
\bottomrule
\end{tabular}}
\end{table}

\begin{table}[!htbp]
\centering
\caption{\textbf{Per-seed best-checkpoint step distribution} (median, IQR, min--max in env-step millions across 10 seeds, per configuration). These steps relate our runs to the peak-based reporting conventions in \cref{tab:protocol}. Headline and ablation comparisons use the final checkpoint; \cref{tab:4metric} and the MT50 best column of \cref{tab:per-side} additionally report best-checkpoint (and, in \cref{tab:4metric}, peak and late) values.}
\label{tab:bestckpt-steps}
\small
\setlength{\tabcolsep}{4pt}
\begin{tabular}{l c c c c c}
\toprule
Configuration & Bench. & Budget (M) & Median (M) & IQR (M) & min--max (M) \\
\midrule
Vanilla MT-PPO & MT10 & 20 & 15.0 & [13.4, 17.2] & 10.0--19.6 \\
+LN-c & MT10 & 20 & 11.8 & [10.9, 16.4] & 6.8--20.0 \\
+LN-c +PopArt & MT10 & 20 & 13.4 & [9.7, 15.9] & 6.4--17.6 \\
+LN-c +FG-c & MT10 & 20 & 8.4 & [6.7, 9.7] & 5.6--19.2 \\
+LN-c +PopArt +FG-c & MT10 & 20 & 8.4 & [7.0, 9.8] & 6.4--12.4 \\
\textbf{TOPPO} & MT10 & 20 & 12.6 & [9.4, 17.2] & 7.2--19.6 \\
\midrule
Vanilla MT-PPO & MT50 & 100 & 93.5 & [89.0, 97.2] & 83.0--99.0 \\
+LN-c & MT50 & 100 & 77.0 & [72.0, 86.5] & 48.0--96.0 \\
+LN-c +PopArt & MT50 & 100 & 88.0 & [67.0, 92.0] & 40.0--96.0 \\
+LN-c +FG-c & MT50 & 100 & 84.0 & [70.0, 97.0] & 58.0--100.0 \\
+LN-c +PopArt +FG-c & MT50 & 100 & 86.0 & [66.0, 97.5] & 40.0--100.0 \\
\textbf{TOPPO} & MT50 & 100 & 84.0 & [51.0, 93.5] & 46.0--96.0 \\
\bottomrule
\end{tabular}
\end{table}

\FloatBarrier
\section{Extended Experiments and Robustness Checks}\label{app:revision}

This appendix compares TOPPO with SAC-MT and the ARS family under a common evaluation protocol, evaluates alternative normalizers and a categorical critic, reports measured computation cost, applies the critic-side stack to an off-policy learner (MTSAC) and a second control setting (AMP humanoid), and examines actor-side aggregation choices.
The Meta-World headline configuration is unchanged throughout; alternative configurations are labeled explicitly where they appear.
Unless specified otherwise, success rates are final-checkpoint percentages at $100$M total environment steps over $10$ seeds, and every $\pm$ denotes the sample standard deviation across seeds.
For V1, IQM denotes the paper's $25\%$-trimmed mean of per-seed mean success rates, accompanied by the untrimmed sample standard deviation.
Paired tests use the untrimmed scores for shared training seeds; consequently, a reported paired difference need not equal the difference between two printed IQMs.

\subsection{SAC-family baselines under a common evaluation protocol}\label{app:revision-baselines}

We rerun the ARS-family implementations and MTSAC (SAC-MT in \cref{tab:combined-headline-tail}) using Meta-World~3.0.0, V2 rewards, a $500$-step horizon, and $20$M (MT10) or $100$M (MT50) total training interactions, including SAC's random exploration warmup.
All of these runs use TOPPO's evaluation protocol: deterministic policies, one episode for each of the $50$ benchmark goals per task, success if achieved at any point in the episode, and $10$ seeds.
Worst-10 is computed within each seed by averaging its ten lowest task success rates, then summarized across seeds.
For the published rows in \cref{tab:combined-headline-tail}, worst-$k$ is the bottom-$k$ score reported by \citet{cho2025ars}, whose aggregation is not fully specified; the worst-$5$ zeros of the published non-ARS methods are inferred from their zero worst-$10$ scores.

The reruns differ from the published ARS setup in the MuJoCo binding, fixed benchmark training goals, deterministic rather than stochastic evaluation, and $50$ rather than $10$ evaluation goals; published ARS results use $8$ seeds (unstated for ARS-LN-1024).
We therefore retain the published and rerun numbers separately in \cref{tab:combined-headline-tail,tab:revision-ars}.
ARS-LoRA's attach-versus-reset decision also reads this deterministic evaluation signal at its decision point ($85.8$M steps on MT50, $16$M on MT10), so the protocol affects that arm's training as well as its measurement.
On MT50, the most recent per-seed evaluation (at $85$M steps) ranges from $71.5\%$ to $82.8\%$, $6.5$--$17.8$ points above the $65\%$ attach threshold; on MT10, the $16$M evaluation ranges from $85.4\%$ to $99.2\%$ against the $80\%$ threshold, so all ten seeds attach the adapters on both benchmarks.
We retain the released baseline hyperparameter recipes without per-arm tuning.

\begin{table}[!htbp]
\centering
\small
\caption{MT50 V2 ARS-family comparisons. Published values are from Tables~2 and~6 of \citet{cho2025ars}, which report no ARS-LN-1024 worst-10 score. Published values use $8$ seeds, except ARS-LN-1024 (Table~6), whose seed count is not stated. Our runs match the corresponding rows of \cref{tab:combined-headline-tail}. All entries are percentages, mean $\pm$ sample std.}
\label{tab:revision-ars}
\begin{tabular}{lcccc}
\toprule
 & \multicolumn{2}{c}{Published} & \multicolumn{2}{c}{Our runs ($n{=}10$)} \\
\cmidrule(lr){2-3} \cmidrule(lr){4-5}
Method & Mean & Worst-10 & Mean & Worst-10 \\
\midrule
ARS & $65.9\pm2.9$ & $9.1\pm6.6$ & $66.84\pm4.81$ & $0.62\pm1.64$ \\
ARS-LN-400 & $78.3\pm3.7$ & $21.0\pm11.4$ & $76.64\pm5.67$ & $11.70\pm9.57$ \\
ARS-LN-1024 & $88.9\pm2.6$ & --- & $85.75\pm3.23$ & $43.26\pm9.77$ \\
ARS-LoRA & $83.2\pm1.8$ & $29.3\pm7.7$ & $82.28\pm3.09$ & $19.04\pm11.56$ \\
\midrule
TOPPO (headline) & --- & --- & $90.88\pm1.59$ & $56.50\pm7.76$ \\
\bottomrule
\end{tabular}
\end{table}

On MT50, against the strongest baseline, ARS-LN-1024, TOPPO improves the final mean by $5.14$ points and worst-10 by $13.24$ points (differences of unrounded means; \cref{tab:combined-headline-tail,tab:revision-ars,fig:revision-baselines}).
Against published values, TOPPO's worst-10 exceeds the best reported published worst-10 (ARS-LoRA, $29.3$) by $27.2$ points.

\begin{figure}[!htbp]
\centering
\includegraphics[width=\linewidth]{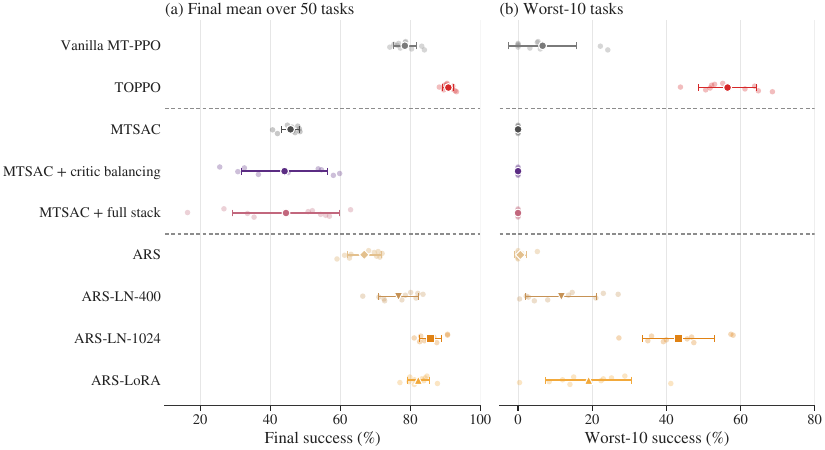}
\caption{Final MT50 V2 success of our PPO and baseline runs, ten seeds each at $100$M transitions: (a) final mean over 50 tasks and (b) worst-10 success. Pale points are the ten seeds, large markers the means, and whiskers sample std; dashed lines separate the PPO, MTSAC, and ARS families. Published baseline numbers are listed separately in \cref{tab:revision-ars}. The critic-balancing and full-stack MTSAC ports show much greater seed dispersion without a reliable mean improvement.}
\label{fig:revision-baselines}
\end{figure}

\paragraph{Porting the stack to MTSAC.}\label{app:revision-sac}
MTSAC uses the ARS backbone hyperparameters with adaptive reward scaling disabled.
Its critic-balancing variant adds LN-c, PopArt, and FairGrad-c at $\alpha{=}1$; the full-stack variant additionally uses actor-side PCGrad.
The port uses EMA PopArt with coefficient $3\times10^{-4}$, the upstream SAC LayerNorm parameterization, and entropy temperature in normalized value units.

\begin{table}[!htbp]
\centering
\small
\caption{MT50 V2 transfer to MTSAC at $100$M steps, $10$ seeds. The first row is SAC-MT from \cref{tab:combined-headline-tail}. Neither port improves the final mean; both increase seed dispersion.}
\label{tab:revision-sac-transfer}
\begin{tabular}{lcc}
\toprule
Method & Final mean & Final worst-10 \\
\midrule
SAC-MT (MTSAC), ARS hyperparameters & $45.82\pm2.61$ & $0.00\pm0.00$ \\
MTSAC + critic balancing & $44.11\pm12.27$ & $0.00\pm0.00$ \\
MTSAC + full stack & $44.51\pm15.37$ & $0.00\pm0.00$ \\
\bottomrule
\end{tabular}
\end{table}

The final mean changes by $-1.71$ points for critic balancing and $-1.31$ points for the full stack relative to MTSAC, with seed standard deviations increasing by factors of $4.7$ and $5.9$, respectively (\cref{tab:revision-sac-transfer,fig:revision-baselines}).
The full-stack port therefore provides no reliable final-budget gain beyond the critic-only port.
\Cref{fig:matched-sac-curves} shows the corresponding training trajectories.
Transferred to MTSAC without retuning, critic balancing does not improve final success.

\begin{figure}[!htbp]
\centering
\includegraphics[width=\linewidth]{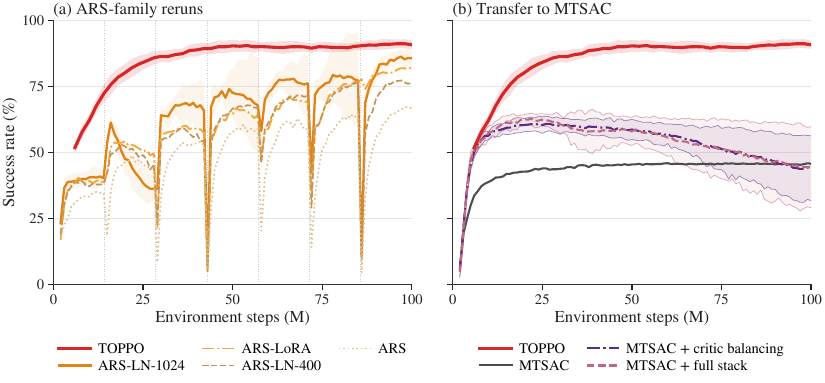}
\caption{Recorded MT50 V2 trajectories, ten seeds per configuration. (a) TOPPO and the ARS family. (b) TOPPO, MTSAC, and MTSAC with critic balancing or the full stack. Lines show means on each group's common recorded evaluation grid; bands show sample std for TOPPO and ARS-LN-1024 in (a) and for TOPPO and the two MTSAC ports in (b), where thin edges separate the two overlapping port bands. Curves are not smoothed. Dotted vertical lines in (a) mark the scheduled ARS learner resets every $14.3$M transitions. LoRA follows its configured reset/attach rule.}
\label{fig:matched-sac-curves}
\end{figure}

\subsection{Normalizer alternatives and value clipping}\label{app:revision-normalizers}

At the LN-c rung, with both gradient combiners set to their default mean, we compare PopArt with per-task reward normalization, return-based scaling~\citep{schaul2021returnscaling}, and the HL-Gauss categorical value loss~\citep{farebrother2024stopregressing}.
Reward normalization divides rewards by a running standard deviation of discounted returns and clips normalized rewards to $[-10,10]$.
Return-based scaling instead rescales the value-loss residual while retaining raw value units.
HL-Gauss replaces scalar regression with cross-entropy against Gaussian-smoothed targets on a fixed categorical support.
This campaign trains $120$ runs and reuses completed reference runs where configurations coincide.
Final scores are computed from each run's recorded evaluation log.

\begin{figure}[!htbp]
\centering
\includegraphics[width=\linewidth]{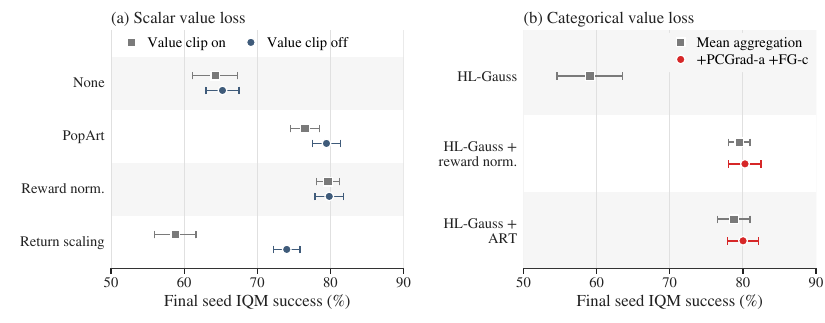}\par
\caption{MT50 V1 normalizer controls, ten seeds at $100$M steps (seed IQM $\pm$ untrimmed std). (a) Scalar value loss, value clipping on (gray) or off (slate). (b) HL-Gauss critics with mean aggregation (gray) or +PCGrad-a +FG-c (red); bins differ by row (\cref{tab:revision-hlgauss}); ART: PopArt-style target rescaling.}
\label{fig:revision-normalizers}
\end{figure}

\paragraph{Matched clipping controls.}
The headline PPO configuration clips value changes by $0.2$ in the units used by its value loss.
Those units differ between raw-return, normalized-return, and PopArt critics.
HL-Gauss structurally requires value clipping to be disabled.
We therefore compare all five V1 critic settings with value clipping disabled and report the clip-on scalar controls alongside them in \cref{tab:revision-normalizers,fig:revision-normalizers}.

\begin{table}[!htbp]
\centering
\small
\caption{MT50 V1 normalizer comparison at the LN-c rung, $10$ seeds and $100$M steps. IQM is computed across per-seed mean success rates; std uses all seeds. Worst-20 is the per-seed lower-tail mean, averaged across seeds. The matched comparison is the clip-off column.}
\label{tab:revision-normalizers}
\begin{tabular}{lccc}
\toprule
Normalizer & Clip off: IQM $\pm$ std & Clip on: IQM $\pm$ std & Clip-off worst-20 \\
\midrule
None & $65.25\pm2.27$ & $64.27\pm3.08$ & $19.49\pm3.72$ \\
PopArt & $79.47\pm1.94$ & $76.53\pm1.97$ & $48.95\pm4.77$ \\
Reward normalization & $79.85\pm1.97$ & $79.69\pm1.55$ & $49.97\pm4.91$ \\
Return-based scaling & $74.03\pm1.80$ & $58.79\pm2.83$ & $37.39\pm4.34$ \\
HL-Gauss & $59.07\pm4.50$ & --- & $5.30\pm5.27$ \\
\bottomrule
\end{tabular}
\end{table}

Under matched clip-off conditions, reward normalization and PopArt differ by only $+0.43$ points in paired untrimmed means ($p{=}0.64$).
The larger clip-on gap ($+3.61$ points) arises from value clipping rather than from the normalizer: removing clipping improves PopArt by $3.49$ points ($p{=}0.0027$) but changes reward normalization by less than half a point.
Removing clipping also improves return-based scaling by $15.78$ points ($p{=}8.3\times10^{-8}$); both significant gains survive Bonferroni correction over the four within-normalizer clipping comparisons.
The PopArt ablations of the intervention ladder therefore measure gains over an unnormalized critic.

On V2, mean success is $87.54\pm2.42$ without a normalizer, $88.90\pm2.26$ with PopArt, $89.49\pm1.98$ with reward normalization, $89.08\pm1.77$ with return-based scaling, and $87.74\pm2.12$ with HL-Gauss.
No alternative differs significantly from the no-normalizer control at the nominal $0.05$ level.
The only nominally significant pair among all ten comparisons, reward normalization versus HL-Gauss ($p{=}0.036$), does not survive Bonferroni correction ($p_{\mathrm{adj}}{=}0.36$).
These V2 rows mix clip-on scalar critics with clip-off HL-Gauss; loss form and clipping therefore vary together in this comparison.

\paragraph{Return scaling and time-limit truncation.}
In this benchmark, episodes end by time-limit truncation and bootstrap rather than entering a true terminal state.
The discount-variance term in return-based scaling therefore vanishes: measured $\operatorname{Var}[\gamma_t]\simeq3.9\times10^{-13}$, so its scale estimate reduces to reward variance.
A diagnostic variant that counts truncations as terminal for the scale estimator restores variance to $1.956\times10^{-3}$, close to $0.99^2/500$, but changes final success by only $+1.72$ points ($p{=}0.25$) in the clip-on V1 setting.
The degenerate scale estimate therefore does not explain the low clip-on score of return-based scaling; removing value clipping recovers $15.78$ points in paired untrimmed means ($15.24$ IQM points, \cref{tab:revision-normalizers}), which closes most of its gap to PopArt and reward normalization.

\paragraph{Alternative normalization in the full stack.}
With PCGrad-a and FG-c enabled, reward normalization reaches V1 IQM $82.22\pm2.86$, compared with $80.03\pm2.61$ for the PopArt-EMA reference (PCGrad row of \cref{tab:actor-sweep}).
The paired untrimmed mean difference is $+2.49$ points ($p{=}0.066$), which is not statistically significant.
This comparison uses EMA PopArt rather than the headline's Welford statistics.

\subsection{Composing a categorical critic with scale normalization}\label{app:revision-hlgauss}

On a raw-return support, standalone HL-Gauss trails even the unnormalized scalar critic on V1 (IQM $59.07$ vs.\ $65.25$; \cref{tab:revision-normalizers}).
We therefore compose HL-Gauss with either per-task reward normalization or PopArt-style adaptive target rescaling (ART).
The latter normalizes value targets before categorical projection but omits PopArt's output-preserving reparameterization, for which this categorical implementation has no exact analog.
Both compositions disable value clipping.
Their support and bin count are calibrated from preliminary return statistics, so they introduce additional choices beyond the headline scalar-critic configuration.

\begin{table}[!htbp]
\centering
\small
\caption{MT50 V1 categorical-critic compositions, $10$ seeds and $100$M steps. All rows use LN-c and disable value clipping. The final two rows add PCGrad-a and FairGrad-c at $\alpha{=}1$. IQM is accompanied by untrimmed sample std; worst-20 is mean $\pm$ sample std.}
\label{tab:revision-hlgauss}
\begin{tabular}{lccc}
\toprule
Configuration & Bins; support & IQM $\pm$ std & Worst-20 \\
\midrule
HL-Gauss alone & $101$; $[-8\!\times\!10^5,1.6\!\times\!10^6]$ & $59.07\pm4.50$ & $5.30\pm5.27$ \\
HL-Gauss + reward normalization & $1000$; $[-10,30]$ & $79.51\pm1.48$ & $49.42\pm3.69$ \\
HL-Gauss + ART & $2400$; $[-10,60]$ & $78.78\pm2.21$ & $47.02\pm5.46$ \\
\midrule
Full stack + HL-Gauss + reward norm. & $1000$; $[-10,30]$ & $80.29\pm2.20$ & $51.49\pm5.50$ \\
Full stack + HL-Gauss + ART & $2400$; $[-10,60]$ & $80.03\pm2.11$ & $50.62\pm5.27$ \\
\bottomrule
\end{tabular}
\end{table}

At the LN-c rung, HL-Gauss with reward normalization raises V1 IQM from $59.07$ to $79.51$ (\cref{tab:revision-hlgauss}); its paired untrimmed improvement is $20.93$ points ($p{=}5.8\times10^{-7}$).
It does not differ significantly from matched scalar reward normalization ($-0.26$ points, $p{=}0.76$) or matched scalar PopArt ($+0.17$ points, $p{=}0.71$).
The large improvement over standalone HL-Gauss survives Bonferroni correction over these three comparisons; the composition jointly changes target scaling, support, and bin count.
On V2, the corresponding $600$-bin composition on $[-4,32]$ reaches mean $89.95\pm1.92$ and worst-10 $54.24\pm7.85$; its $+2.21$-point paired mean difference from standalone HL-Gauss is not significant ($p{=}0.061$).

A smaller $101$-bin reward-normalized control reaches mean $79.89\pm1.24$ over three seeds, with a $+1.15$-point paired difference from the $1000$-bin composition on the same seeds ($p{=}0.45$).
This control uses support $[-6,25]$, and $25$ of $50$ tasks have target standard deviation below one bin.
Thus a count of tasks with sub-bin target variation is insufficient to explain the standalone critic's failure.

Both compositions remain viable in the full stack, but adding the two gradient combiners yields no statistically reliable extra gain: $+0.87$ points for reward normalization ($p{=}0.085$) and $+1.46$ points for ART ($p{=}0.151$) over their respective LN-c rungs.
Comparisons to scalar full-stack references additionally change value clipping.

\subsection{Measured computation cost}\label{app:revision-compute}

We time $10$M training interactions per configuration on an RTX~3090 with an AMD EPYC~7313 CPU, allocating $16$ CPU cores and $64$GB host memory.
\Cref{tab:revision-compute,fig:revision-compute} report the resulting linear extrapolation to $100$M interactions.
Evaluation is disabled, and each configuration is measured once.

\begin{table}[!htbp]
\centering
\small
\caption{Training time extrapolated from $10$M measured interactions on an RTX~3090, with evaluation disabled and one measurement per configuration. The three additional combiner rows apply the named method to both actor and critic, with LN-c and PopArt enabled.}
\label{tab:revision-compute}
\begin{tabular}{lc@{\qquad}lc}
\toprule
Configuration & Hours/$100$M & Configuration & Hours/$100$M \\
\midrule
Vanilla MT-PPO & $8.7$ & FairGrad on both sides & $16.5$ \\
$+$ LN-c & $8.8$ & CAGrad on both sides & $16.9$ \\
$+$ LN-c $+$ PopArt & $8.9$ & PCGrad on both sides & $23.9$ \\
$+$ LN-c $+$ PopArt $+$ FG-c & $15.6$ & MTSAC & $7.2$ \\
TOPPO (also PCGrad-a) & $19.8$ & ARS & $7.5$ \\
 & & ARS-LN-1024 & $17.9$ \\
\bottomrule
\end{tabular}
\end{table}

LN-c and PopArt add approximately $2\%$ to the vanilla training time in this measurement; extracting and combining per-task gradients accounts for the larger overhead.
Among the measured symmetric combiners, FairGrad is the least expensive, while PCGrad takes approximately $45\%$ longer.
TOPPO takes approximately $10\%$ longer per environment interaction than ARS-LN-1024 on this rig, so wall-clock comparisons scale the step counts in \cref{fig:train-curve} by this ratio.

The stratified MT50 sampler produces $31$ minibatches per repeat, giving $16\times31{=}496$ combiner calls per rollout collection and $99{,}200$ per $100$M-step seed.
In the production critic telemetry of the headline MT50 V2 runs, Gram formation plus the FairGrad solve averages $6.9$--$7.4$\,ms per call over training on the RTX~3090 hosts of seeds $3$--$9$ and $0.8$--$0.9$\,ms on the H200 NVL host of seeds $0$--$2$; this timer excludes gradient extraction and recombination.
At $99{,}200$ calls per $100$M-step seed, the RTX~3090 rate amounts to $11$--$12$ minutes, about $1.0\%$ of the $19.8$-hour estimate measured on the same GPU model.
The solve-only benchmark on $1{,}984$ real training Gram matrices is reported separately in \cref{app:solver}.

All configurations use the same CPU and GPU model; the $+$FG-c and $+$PopArt configurations ran on different physical nodes.
SAC and PPO run in different software frameworks, with offline logging for SAC and logging disabled for PPO.
The $10$M SAC measurement includes replay warmup, so its $100$M extrapolation reflects the warmup-inclusive rate rather than the full-run average throughput.
Evaluation is excluded from all timings, and all measurements use $K{=}50$.

\begin{figure}[!htbp]
\centering
\includegraphics[width=\linewidth]{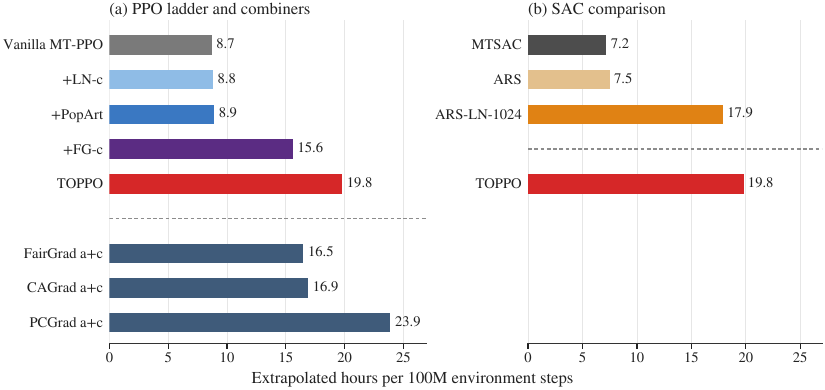}
\caption{Training-loop time on RTX 3090/EPYC 7313 hardware with evaluation disabled; each bar is one $10$M-step measurement extrapolated to $100$M steps, so there are no seed intervals. (a) The PPO intervention path and, below the dashed line, FairGrad, CAGrad, and PCGrad applied to both actor and critic (a+c). (b) MTSAC, ARS, and ARS-LN-1024, with TOPPO for reference. These are whole-loop estimates; the solve-only timings are in \cref{app:solver}.}
\label{fig:revision-compute}
\end{figure}

\subsection{AMP humanoid: a two-objective control stress test}\label{app:revision-amp}

We also study humanoid walk steering with an adversarial motion prior (AMP)~\citep{peng2021amp}, combining task-tracking and motion-style rewards through a shared critic trunk.
Our implementation extends MimicKit~\citep{peng2025mimickit}, which provides the humanoid character, steering environment, and AMP agent; the reference walking clip comes from the public DeepMimic motion data~\citep{peng2018deepmimic}, and all training and evaluation runs use the Newton physics simulator~\citep{newton2025}.
This is a two-objective control setting that tests the critic-side mechanism outside Meta-World.
The broader campaign comprises $25$ configurations with $10$ shared training seeds each; the configurations in \cref{tab:revision-amp} all train for $100$M samples.
Each policy is evaluated on $4{,}096$ episodes under a common environment configuration, using $8{,}192$ motion-style samples.

Our style metric is the squared Fr\'echet distance after a fixed $64$-dimensional random projection of AMP observation features standardized by demonstration statistics.
It uses the same feature definition for all arms, independent of each policy's learned discriminator.
We report task return alongside style distance because improving motion similarity can trade off against task tracking.
All uncertainty in \cref{tab:revision-amp} is sample standard deviation across training seeds, and inference pairs those seeds rather than individual evaluation episodes.

\begin{table}[!htbp]
\centering
\small
\caption{AMP humanoid walk steering, $100$M samples and $10$ seeds per arm. Style FID is a squared distance; lower is better. Values are mean $\pm$ sample std. The PPO FairGrad row also disables critic gradient clipping, and the full stack retains that change.}
\label{tab:revision-amp}
\begin{tabular}{llcc}
\toprule
Family & Configuration & Task return & Style FID \\
\midrule
PPO & Base & $533.56\pm20.74$ & $(5.477\pm1.327)\times10^4$ \\
PPO & Two-component critic $+$ LN $+$ PopArt & $534.99\pm16.89$ & $(3.855\pm0.871)\times10^4$ \\
PPO & Above $+$ FG-c, no critic grad clip & $550.63\pm9.01$ & $(3.635\pm3.943)\times10^4$ \\
PPO & Full stack (also PCGrad-a) & $521.11\pm14.64$ & $(1.315\pm0.088)\times10^4$ \\
\midrule
SAC & Base & $573.77\pm4.03$ & $(5.596\pm1.522)\times10^6$ \\
SAC & Full stack & $575.33\pm2.57$ & $(4.486\pm1.108)\times10^6$ \\
\bottomrule
\end{tabular}
\end{table}

The PPO full stack reduces mean style FID by $76.0\%$ relative to the base policy (\cref{tab:revision-amp,fig:revision-amp}), with improvement in all ten seeds (paired $p{=}4.7\times10^{-6}$; 95\% CI for the relative mean difference $[-93.8\%,-58.2\%]$; significant after Bonferroni correction across $25$ comparisons).
The intermediate rungs do not isolate module effects: the +FG-c rung also disables critic gradient clipping and has a seed standard deviation larger than its mean ($(3.635\pm3.943)\times10^4$), and adding PCGrad-a gives a further paired change of $-2.32\times10^4$ (95\% CI $[-5.16,0.52]\times10^4$, $p{=}0.098$).
Moreover, the full-stack policy walks more slowly: the paired change in projected speed is $-0.416$ (from $1.169$ to $0.753$; 95\% CI $[-0.529,-0.303]$), while the speed-error change is not significant (95\% CI $[-0.119,0.023]$).
The style gain thus comes at the cost of slower walking, a tradeoff between the motion-style and task objectives.

The SAC full stack reduces mean style FID by $19.8\%$, but the paired effect is not statistically reliable at $n{=}10$ ($p{=}0.078$; 95\% CI for the relative mean difference $[-42.4\%,2.8\%]$).
Cross-family module comparisons also differ in implementation: the SAC FairGrad path forms a preconditioned Gram in fp32, whereas the PPO path forms it in fp64 (\cref{app:solver}); in addition, no arm rescales the combined gradient to the norm of the default aggregation, so direction and magnitude effects are not separated.
In summary, the full PPO stack reliably improves style FID at the cost of slower walking, while the SAC improvement is smaller and not significant at $n{=}10$.

A separate short H200 timing run adds $4.2\%$ end-to-end time for the PPO full stack ($189.5$ to $197.5$s), compared with $49.5\%$ for SAC ($287.6$ to $430.1$s).
Each is a single measurement over approximately $8$M samples.
Evaluation spans multiple GPU types; for the checkpoints we re-evaluated on a second GPU type, style FID changes by at most $0.994\%$.

\begin{figure}[!htbp]
\centering
\includegraphics[width=\linewidth]{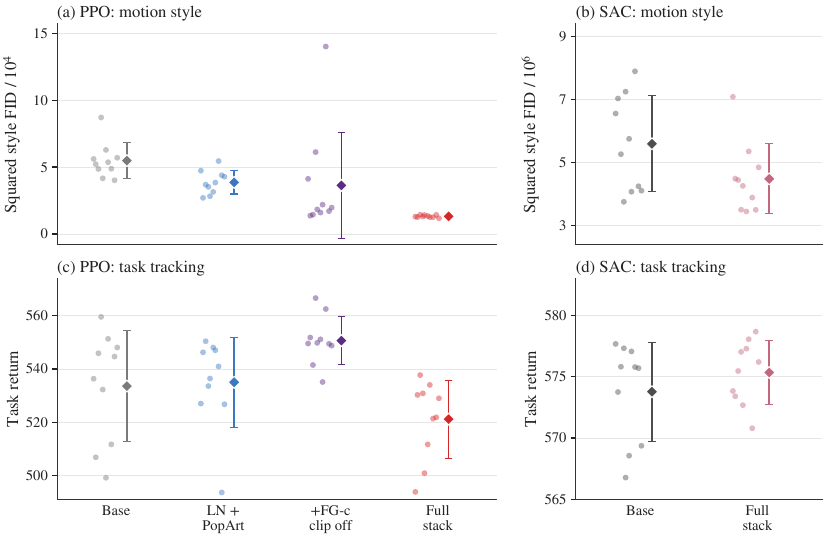}
\caption{AMP walk steering, ten seeds per arm at $100$M samples. Pale points are seeds; filled diamonds and whiskers show the mean and sample std. (a, b) Motion style as squared Fr\'echet distance in a common fixed feature space (lower is better; PPO in units of $10^4$, SAC in units of $10^6$). (c, d) Task return, showing the accompanying tradeoff. The PPO +FG-c arm also disables critic gradient clipping, a change retained by the full stack.}
\label{fig:revision-amp}
\end{figure}

\subsection{Actor aggregation and fairness sensitivity}\label{app:revision-actor}

We extend the actor-side comparison with eight choices at a fixed critic configuration: the default mean, PCGrad, CAGrad with $c{=}0.5$, and FairGrad with $\alpha_a\in\{0.5,1,2,5,10\}$.
The critic uses LayerNorm and FairGrad at $\alpha_c{=}1$ throughout.
This study uses EMA PopArt with update coefficient $3\times10^{-4}$, value clipping enabled, and the same $[400,400,400]$ actor and critic architectures.
Each of the $16$ actor-by-reward-regime cohorts contains $10$ completed seeds at $100$M environment steps.
We report this EMA sensitivity study separately from the Welford-PopArt headline configuration.

\begin{table}[!htbp]
\centering
\small
\setlength{\tabcolsep}{5pt}
\caption{Actor aggregation sensitivity on MT50, final checkpoint at $100$M steps, $10$ seeds per cell.
All uncertainties are untrimmed sample standard deviations.
V1 IQM uses the paper's seed-trimming convention: sort the ten per-seed mean success rates, discard the two lowest and two highest, and average the remaining six.
The Welford reference is the separate headline cohort and is excluded from the EMA actor comparisons.}
\label{tab:actor-sweep}
\begin{tabular}{lccc}
\toprule
Actor aggregation & V1 mean $\pm$ std & V1 IQM $\pm$ std & V2 mean $\pm$ std \\
\midrule
\multicolumn{4}{l}{\emph{EMA PopArt study; fixed LN-c and FG-c ($\alpha_c{=}1$)}}\\
Default mean & $78.32\pm1.95$ & $78.42\pm1.95$ & $90.38\pm1.90$ \\
PCGrad & $79.94\pm2.61$ & $80.03\pm2.61$ & $89.86\pm1.79$ \\
CAGrad ($c{=}0.5$) & $79.60\pm1.99$ & $79.71\pm1.99$ & $90.64\pm1.91$ \\
FairGrad ($\alpha_a{=}0.5$) & $81.03\pm2.34$ & $81.09\pm2.34$ & $89.58\pm2.04$ \\
FairGrad ($\alpha_a{=}1$) & $81.04\pm2.38$ & $80.82\pm2.38$ & $91.10\pm1.37$ \\
FairGrad ($\alpha_a{=}2$) & $80.00\pm1.74$ & $80.09\pm1.74$ & $90.54\pm2.08$ \\
FairGrad ($\alpha_a{=}5$) & $80.45\pm2.05$ & $80.47\pm2.05$ & $89.19\pm1.80$ \\
FairGrad ($\alpha_a{=}10$) & $80.29\pm1.84$ & $80.20\pm1.84$ & $89.70\pm2.17$ \\
\midrule
\multicolumn{4}{l}{\emph{Welford PopArt: headline reference}}\\
PCGrad (TOPPO) & $79.69\pm1.35$ & $79.35\pm1.35$ & $90.88\pm1.59$ \\
\bottomrule
\end{tabular}
\end{table}

Within the EMA cohort, PCGrad is not the highest-scoring actor choice in either reward regime, and on V2 it falls below the default mean ($89.86$ vs.\ $90.38$), reversing the Welford-cohort ordering of \cref{tab:per-side} (\cref{tab:actor-sweep,fig:revision-actor}).
Actor FairGrad at $\alpha_a{=}1$ exceeds PCGrad's untrimmed mean by $1.10$ points on V1 and $1.24$ points on V2, with paired two-sided Student-$t$ $p$-values of $0.126$ and $0.064$, respectively.
Across the seven comparisons against PCGrad in each regime, none reaches even nominal $p<0.05$; the smallest $p$-values are $0.057$ on V1 and $0.064$ on V2, and none survives Bonferroni correction within its seven-comparison family.
No tested alternative significantly outperforms PCGrad-a under FG-c; several actor aggregators perform comparably.

\begin{figure}[!htbp]
\centering
\includegraphics[width=\linewidth]{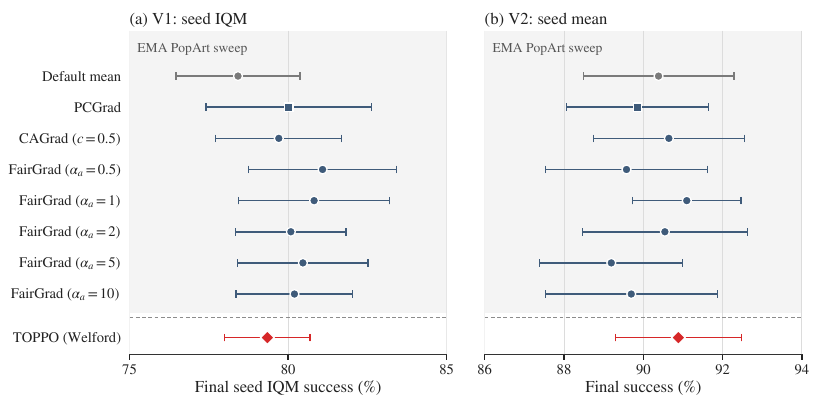}
\caption{Actor aggregation on MT50, ten seeds at $100$M transitions: (a) V1 seed IQM and (b) V2 seed mean of final success; whiskers are untrimmed sample std. Shaded rows use EMA PopArt, LN-c, and FG-c ($\alpha_c{=}1$) and vary only the actor aggregator: the default mean (gray) and seven alternatives (slate; PCGrad as a square). The red diamond is the separate Welford-PopArt TOPPO headline.}
\label{fig:revision-actor}
\end{figure}

\end{document}